\documentclass[11pt]{article}

\usepackage{wrapfig}
\usepackage{enumerate}
\usepackage[OT1]{fontenc}
\usepackage{amsmath,amssymb}
\usepackage[usenames]{color}
\usepackage{mathtools}
\usepackage{tikz}
\usetikzlibrary{matrix}
\usetikzlibrary{bayesnet}
\usetikzlibrary{arrows}
\usepackage{color}
\usepackage{graphicx}
\usepackage{caption}
\usetikzlibrary{backgrounds}

\usepackage{smile}
\usepackage{multirow}
\usepackage[colorlinks,
            linkcolor=red,
            anchorcolor=blue,
            citecolor=blue
            ]{hyperref}

\def\given{\,|\,}
\def\biggiven{\,\big{|}\,}
\def\tr{\mathop{\text{tr}}\kern.2ex}

\long\def\comment#1{}

\def\tr{\mathop{\text{Tr}}}

\def\cS{{\mathcal{S}}}

\newcommand{\bel}{\begin{eqnarray}\label}
\newcommand{\eel}{\end{eqnarray}}
\newcommand{\bes}{\begin{eqnarray*}}
\newcommand{\ees}{\end{eqnarray*}}

\def\algoname{ETC}
\def\algosp{{\rm ETC }}

\newcommand{\smallsec}[1]{\vskip4pt\noindent{\bf#1}}

\newcommand{\OO}{\mathbb{O}}
\newcommand{\UU}{\mathbb{U}}
\newcommand{\BB}{\mathbb{B}}
\newcommand{\FF}{\mathbb{F}}
\newcommand{\DD}{\mathbb{D}}

\def\diag{\DD}
\def\diagO{\tilde\OO}
\def\ObsO{\OO}
\def\ObsB{\BB}
\def\ObsF{\FF}

\def\vecb{{b}}
\def\pone{\mathbb{X}}
\def\ptwo{\mathbb{Y}}
\def\past{w}
\def\futr{z}

\def\lat{q}
\def\param{\theta}
\def\Param{\Theta}
\newcommand{\Tran}{\mathbb{T}}
\def\distP{\PP}
\def\traj{\tau}
\def\utraj{\underline\traj}
\def\est{\mathfrak{E}}
\def\poly{{\rm poly}}

\def\seq#1#2#3{{#1^{#3}_{#2}}}
\def\sp#1#2#3{{^{#1}_{#2}}#3}
\def\eid#1{\ind({#1})}

 \newcommand{\indp}{\perp\!\!\!\!\perp} 
 \newcommand{\notindp}{\not\perp\!\!\!\!\perp}

\let\emptyset\varnothing

\usepackage{mathrsfs}

\usepackage{fullpage}

\def\##1\#{\begin{align}#1\end{align}}
\def\$#1\${\begin{align*}#1\end{align*}}
\begin{document}

\title{Embed to Control Partially Observed Systems:  Representation Learning with Provable Sample Efficiency}
\author
{\normalsize
Lingxiao Wang\thanks{Northwestern University; 
\texttt{lingxiaowang2022@u.northwestern.edu}}
\qquad
Qi Cai\thanks{Northwestern University; 
\texttt{qicai2022@u.northwestern.edu}}
\qquad Zhuoran Yang\thanks{Yale University; 
\texttt{zhuoran.yang@yale.edu}}
\qquad Zhaoran Wang\thanks{Northwestern University; 
\texttt{zhaoranwang@gmail.com}}
}
\date{}
\maketitle
\begin{abstract}
 Reinforcement learning in partially observed Markov decision processes (POMDPs) faces two challenges. (i) It often takes the full history to predict the future, which induces a sample complexity that scales exponentially with the horizon. (ii) The observation and state spaces are often continuous, which induces a sample complexity that scales exponentially with the extrinsic dimension. Addressing such challenges requires learning a minimal but sufficient representation of the observation and state histories by exploiting the structure of the POMDP.

To this end, we propose a reinforcement learning algorithm named Embed to Control (ETC), which learns the representation at two levels while optimizing the policy.~(i) For each step, ETC learns to represent the state with a low-dimensional feature, which factorizes the transition kernel. (ii) Across multiple steps, ETC learns to represent the full history with a low-dimensional embedding, which assembles the per-step feature. We integrate (i) and (ii) in a unified framework that allows a variety of estimators (including maximum likelihood estimators and generative adversarial networks). For a class of POMDPs with a low-rank structure in the transition kernel, ETC attains an $O(1/\epsilon^2)$ sample complexity that scales polynomially with the horizon and the intrinsic dimension (that is, the rank). Here $\epsilon$ is the optimality gap. To our best knowledge, ETC is the first sample-efficient algorithm that bridges representation learning and policy optimization in POMDPs with infinite observation and state spaces.
\end{abstract}

\section{Introduction}

Deep reinforcement learning demonstrates significant empirical successes in Markov decision processes (MDPs) with large state spaces \citep{mnih2013playing,mnih2015human,silver2016mastering,silver2017mastering}. Such empirical successes are attributed to the integration of representation learning into reinforcement learning. In other words, mapping the state to a low-dimensional feature enables model/value learning and optimal control in a sample-efficient manner. Meanwhile, it becomes more theoretically understood that the low-dimensional feature is the key to sample efficiency in the linear setting \citep{cai2020provably, jin2020provably, ayoub2020model, agarwal2020flambe, modi2021model, uehara2021representation}. 

In contrast, partially observed Markov decision processes (POMDPs) with large observation and state spaces remain significantly more challenging. Due to a lack of the Markov property, the low-dimensional feature of the observation at each step is insufficient for the prediction and control of the future \citep{sondik1971optimal, papadimitriou1987complexity, coates2008learning,azizzadenesheli2016reinforcement, guo2016pac}. Instead, it is necessary to obtain a low-dimensional embedding of the history, which assembles the low-dimensional features across multiple steps \citep{hefny2015supervised, sun2016learning}. In practice, learning such features and embeddings requires various heuristics, e.g., recurrent neural network architectures and auxiliary tasks \citep{hausknecht2015deep,li2015recurrent,mirowski2016learning,girin2020dynamical}. In theory, the best results are restricted to the tabular setting \citep{azizzadenesheli2016reinforcement, guo2016pac, jin2020sample, liu2022partially}, which does not involve representation learning. 

To this end, we identify a class of POMDPs with a low-rank structure on the state transition kernel (but not on the observation emission kernel), which allows prediction and control in a sample-efficient manner. More specifically, the transition admits a low-rank factorization into two unknown features, whose dimension is the rank. On top of the low-rank transition, we define a Bellman operator, which performs a forward update for any finite-length trajectory. The Bellman operator allows us to further factorize the history across multiple steps to obtain its embedding, which assembles the per-step feature.

By integrating the two levels of representation learning, that is, (i) feature learning at each step and (ii) embedding learning across multiple steps, we propose a sample-efficient algorithm, namely Embed to Control (ETC), for POMDPs with infinite observation and state spaces. The key to ETC is balancing exploitation and exploration along the representation learning process. To this end, we construct a confidence set of embeddings upon identifying and estimating the Bellman operator, which further allows efficient exploration via optimistic planning. It is worth mentioning that such a unified framework allows a variety of estimators (including maximum likelihood estimators and generative adversarial networks).

We analyze the sample efficiency of ETC under the future and past sufficiency assumptions. In particular, such assumptions ensure that the future and past observations are sufficient for identifying the belief state, which captures the information-theoretic difficulty of POMDPs. We prove that ETC attains an $O(1/\epsilon^2)$ sample complexity that scales polynomially with the horizon and the dimension of the feature (that is, the rank of the transition). Here $\epsilon$ is the optimality gap. The polynomial dependency on the horizon is attributed to embedding learning across multiple steps, while polynomial dependency on the dimension is attributed to feature learning at each step, which is the key to bypassing the infinite sizes of the observation and state spaces. 

\smallsec{Contributions.} In summary, our contribution is threefold.
\begin{itemize}
\item We identify a class of POMDPs with the low-rank transition, which allows representation learning and reinforcement learning in a sample-efficient manner. 
\item We propose ETC, a principled approach integrating embedding and control in the low-rank POMDP.
\item We establish the sample efficiency of ETC in the low-rank POMDP with infinite observation and state spaces. 
\end{itemize}
\smallsec{Related Work.} Our work follows the previous studies of POMDPs. In general, solving a POMDP is intractable from both the computational and the statistical perspectives \citep{papadimitriou1987complexity, vlassis2012computational, azizzadenesheli2016reinforcement, guo2016pac, jin2020sample}. Given such computational and statistical barriers, previous works attempt to identify tractable POMDPs. In particular, \cite{azizzadenesheli2016reinforcement, guo2016pac, jin2020sample, liu2022partially} consider the tabular POMDPs with (left) invertible emission matrices. \cite{efroni2022provable} considers the POMDPs where the state is fully determined by the most recent observations of a fixed length. \cite{cayci2022learning} analyze POMDPs where a finite internal state can approximately determine the state. In contrast, we analyze POMDPs with the low-rank transition and allow the state and observation spaces to be arbitrarily large. Meanwhile, our analysis hinges on the future and past sufficiency assumptions, which only require that the density of the state is identified by that of the future and past observations, respectively. 
In recent work, \cite{cai2022sample} also utilizes the low-rank property in the transition. Nevertheless, \cite{cai2022sample} assumes that the feature representation of state-action pairs is known, thus relieving the agent from feature learning. In contrast, we aim to recover the efficient state-action representation for planning. 
In terms of the necessity of exploration, \cite{azizzadenesheli2016reinforcement, guo2016pac} analyze POMDPs where an arbitrary policy can conduct efficient exploration. Similarly, \cite{cayci2022learning} consider POMDPs with a finite concentrability coefficient \citep{munos2003error, chen2019information}, where the visitation density of an arbitrary policy is close to that of the optimal policy. In contrast, \cite{jin2020sample, efroni2022provable, cai2022sample} consider POMDPs where strategic exploration is necessary. In our work, we follow \cite{jin2020sample, efroni2022provable, cai2022sample} and design strategic exploration to attain sample efficiency in solving the POMDPs. 


To learn a sufficient embedding for control, we utilize the low-rank transition of POMDPs. Our idea is motivated by the previous analysis of low-rank MDPs \citep{cai2020provably, jin2020provably, ayoub2020model, agarwal2020flambe, modi2021model, uehara2021representation}. In particular, the state transition of a low-rank MDP aligns with that in our low-rank POMDP model. Nevertheless, we remark that such states are observable in a low-rank MDP but are unobservable in POMDPs with the low-rank transition. Such unobservability makes solving a low-rank POMDP much more challenging than solving a low-rank MDP.

\smallsec{Notation.} We denote by $\RR^d_{+}$ the space of $d$-dimensional vectors with nonnegative entries. We denote by $L^p(\cX)$ the $L^p$ space of functions defined on $\cX$. We denote by $\Delta(d)$ the space of $d$-dimensional probability density arrays. We denote by $[H] = \{1, \ldots, H\}$ the index set of size $H$. For a linear operator $M$ mapping from an $L^p$ space to an $L^q$ space, we denote by $\|M\|_{p\mapsto q}$ the operator norm of $M$. For a vector $x\in\RR^d$, we denote by $[x]_i$ the $i$-th entry of $x$.

\section{Partially Observable Markov Decision Process}
We define a partially observable Markov decision process (POMDP) by the following tuple, 
\$
\cM = (\cS, \cA, \cO, \{\PP_h\}_{h\in[H]},  \{\OO_h\}_{h\in[H]}, r, H, \mu_1),
\$
where $H$ is the length of an episode, $\mu_1$ is the initial distribution of state $s_1$, and $\cS$, $\cA$, $\cO$ are the state, action, and observation spaces, respectively. Here $\PP_h(\cdot\given \cdot, \cdot)$ is the transition kernel, $\OO_h(\cdot\given\cdot)$ is the emission kernel, and $r(\cdot)$ is the reward function. In each episode, the agent with the policy $\pi = \{\pi_h\}_{h\in[H]}$ interact with the environment as follows. The environment select an initial state $s_1$ drawn from the distribution $\mu_1$. In the $h$-th step, the agent receives the reward $r(o_h)$ and the observation $o_h$ drawn from the observation density $\OO_h(\cdot\given s_h)$, and makes the decision $a_h = \pi_h(\traj^{h}_1)$ according to the policy $\pi_h$, where $\traj^{h}_1 = \{o_1, a_1, \ldots, a_{h-1}, o_h\}$ is the interaction history. The environment then transits into the next state $s_{h+1}$ drawn from the transition distribution $\PP_h(\cdot\given s_h, a_h)$. The procedure ends until the environment transits into the state $s_{H+1}$. 

In the sequel, we assume that the action space $\cA$ is finite with capacity $|\cA| = A$. Meanwhile, we highlight that the observation and state spaces $\cO$ and $\cS$ are possibly infinite.

\smallsec{Value Functions and Learning Objective.} For a given policy $\pi = \{\pi_h\}_{h\in[H]}$, we define the following value function that captures the expected cumulative rewards from interactions,
\#
V^\pi = \EE_\pi\biggl[\sum^H_{h = 1} r(o_h) \biggr].
\#
Here we denote by $\EE_\pi$ the expectation taken with respect to the policy $\pi$. Our goal is to derive a policy that maximizes the cumulative rewards. In particular, we aim to derive the $\epsilon$-suboptimal policy $\pi$ such that
\$
 V^{\pi^*} - V^\pi\leq \epsilon,
\$
based on minimal interactions with the environment, where $\pi^* = \argmax_\pi V^\pi$ is the optimal policy.

\smallsec{Notations of POMDP.} In the sequel, we introduce notations of the POMDP to simplify the discussion. We define 
\$
\seq{a}{h}{h+k-1} = \{a_h, a_{h+1}, \ldots, a_{h+k-1}\}, \quad \seq{o}{h}{h+k} = \{o_h, o_{h+1}, \ldots, o_{h+k}\}
\$
as the sequences of actions and observations, respectively. Correspondingly, we write $r(\seq{o}{1}{H}) = \sum^H_{h = 1}r(o_h)$ as the cumulative rewards for the observation sequence $\seq{o}{1}{H}$. Meanwhile, we denote by $\traj^{h+k}_h$ the sequence of interactions from the $h$-th step to the $(h+k)$-th step, namely,
\$
\traj^{h+k}_h = \{o_h, a_h, \ldots, o_{h+k-1}, a_{h+k-1}, o_{h+k}\} = \{\seq{a}{h}{h+k-1}, \seq{o}{h}{h+k}\}.
\$
Similarly, we denote by $\utraj^{h+k}_h$ the sequence of interactions from the $h$-th step to the $(h+k)$-th step that includes the latest action $a_{h+k}$, namely,
\$
\utraj^{h+k}_h = \{o_h, a_h, \ldots, o_{h+k}, a_{h+k}\} = \{\seq{a}{h}{h+k}, \seq{o}{h}{h+k}\}.
\$
In addition, with a slight abuse of notation, we define
\$
\PP^\pi(\traj^{h+k}_h) &= \PP^\pi(o_h, \ldots, o_{h+k} \given a_{h}, \ldots, a_{h+k-1})  = \PP^\pi(\seq{o}{h}{h+k} \given  \seq{a}{h}{h+k-1}),\notag\\
\PP^\pi(\traj^{h+k}_h\given s_h) &= \PP^\pi(o_h, \ldots, o_{h+k} \given s_h, a_{h}, \ldots, a_{h+k-1})  = \PP^\pi(\seq{o}{h}{h+k} \given s_h,  \seq{a}{h}{h+k-1}).
\$
\smallsec{Extended POMDP.} To simplify the discussion and notations in our work, we introduce an extension of the POMDP, which allows us to access steps $h$ smaller than zero and larger than the length $H$ of an episode. 

In particular, the interaction of an agent with the extended POMDP starts with a dummy initial state $s_{1-\ell}$ for some $\ell>0$. During the interactions, all the dummy action and observation sequences $\utraj^0_{1-\ell} = \{o_{1-\ell}, a_{1-\ell}, \ldots, o_{0}, a_0\}$ leads to the same initial state distribution $\mu_1$ that defines the POMDP. Moreover, the agent is allowed to interact with the environment for $k$ steps after observing the final observation $o_H$ of an episode. Nevertheless, the agent only collects the reward $r(o_h)$ at steps $h\in[H]$, which leads to the same learning objective as the POMDP. In addition, we denote by $[H]^{+} = \{1-\ell, \ldots, H+k\}$ the set of steps in the extended POMDP. In the sequel, we do not distinguish between a POMDP and an extended POMDP for the simplicity of presentation. 





\section{A Sufficient Embedding for Prediction and Control}
\label{sec::embedding}

The key of solving a POMDP is the practice of inference, which recovers the density or linear functionals of density (e.g., the value functions) of future observation given the interaction history. To this end, previous approaches \citep{shani2013survey} 
typically maintain a belief, namely, a conditional density $\PP(s_h = \cdot\given \traj^{h}_1)$ of the current state given the interaction history. The typical inference procedure first conducts filtering, namely, calculating the belief at $(h+1)$-th step given the belief at $h$-th step. Upon collecting the belief, the density of future observation is obtained via prediction, which acquires the distribution of future observations based on the distribution of state $s_{h+1}$.

In the case that maintaining a belief or conducting the prediction is intractable, previous approaches establish predictive states \citep{hefny2015supervised, sun2016learning}, 
which is an embedding that is sufficient for inferring the density of future observations given the interaction history. Such approaches typically recover the filtering of predictive representations by solving moment equations. In particular, \cite{hefny2015supervised, sun2016learning} 
establishes such moment equations based on structural assumptions on the filtering of such predictive states. Similarly, \cite{anandkumar2012method, jin2020sample} 
establishes a sequence of observation operators and recovers the trajectory density via such observation operators.

Motivated by the previous work, we aim to construct a embedding that are both learn-able and sufficient for control. A sufficient embedding for control is the density of the trajectory, namely,
\#\label{eq::def_embedding}
\Phi(\traj^H_1) = \PP(\traj^{H}_1).
\#  
Such an embedding is sufficient as it allows us to estimate the cumulative rewards function $V^\pi$ of an arbitrary given policy $\pi$. Nevertheless, estimating such an embedding is challenging when the length $H$ of an episode and the observation space $\cO$ are large. To this end, we exploit the low-rank structure in the state transition of POMDPs.

\subsection{Low-Rank POMDP}
\begin{assumption}[Low-Rank POMDP]
\label{asu::low_rank_POMDP}
We assume that the transition kernel $\PP_h$ takes the following low-rank form for all $h\in[H]^{+}$,
\$
\PP_h(s_{h+1}\given s_h, a_h) = \psi^*_h(s_{h+1})^\top \phi^*_h(s_{h}, a_h),
\$
where
\$
\psi^*_h: \cS\mapsto \RR^d_{+}, \quad \phi^*_h: \cS\times\cA\mapsto \Delta(d)
\$
are unknown features. 
\end{assumption}
Here recall that we denote by $[H]^{+} = \{1-\ell, \ldots, H+k\}$ the set of steps in the extended POMDP. Note that our low-rank POMDP assumption does not specify the form of emission kernels.~In contrast, we only require the transition kernels of states to be linear in unknown features. 

\smallsec{Function Approximation.} We highlight that the features in Assumption \ref{asu::low_rank_POMDP} are unknown to us. Correspondingly, we assume that we have access to a parameter space $\Param$ that allows us to fit such features as follows.
\begin{definition}[Function Approximation]
\label{def::parameter}
We define the following function approximation space $\cF^\Param = \{\cF^\Param_h\}_{h\in[H]}$ corresponding to the parameter space $\Param$,
\$
\cF^\Param_h = \bigl\{(\psi^\param_h, \phi^\param_h, \OO^\param_h): \param \in\Param\bigr\}, \quad \forall h\in[H]^{+}.
\$
Here, $\OO^\param_h:\cS\times\cO\mapsto \RR_{+}$ is an emission kernel and $\psi^\param_h: \cS\mapsto \RR^d_{+}$, $\phi^\param_h: \cS\mapsto\Delta(d)$ are features for all $h\in[H]^{+}$ and $\param\in\Param$. In addition, it holds that $\psi^\param(\cdot)^\top\phi^\param(s_h, a_h)$ defines a probability over $s_{h+1}\in\cS$ for all $h\in[H]^{+}$ and $(s_h, a_h)\in\cS\times\cA$.
\end{definition}
Here we denote by $\psi^\param_h, \phi^\param_h, \OO^\param_h$ a parameterization of features and emission kernels. In practice, one typically utilizes linear or neural netowrk parameterization for the features and emission kernels. In the sequel, we write $\PP^{\param}$ and $\PP^{\param, \pi}$ as the probability densities corresponding to the transition dynamics defined by $\{\psi^\param_h, \phi^\param_h, \OO^\param_h\}_{h\in[H]}$ and policy $\pi$, respectively. 
We impose the following realizability assumption to ensure that the true model belongs to the parameterized function space $\cF^\Param$.
\begin{assumption}[Realizable Parameterization]
\label{asu::exact_paramete}
We assume that there exists a parameter $\param^*\in\Param$, such that $\psi^{\param^*}_h = \psi^*_h$, $\phi^{\param^*}_h = \phi^*_h$, and $\OO^{\param^*}_h = \OO_h$ for all $h\in[H]$.
\end{assumption}
We define the following forward emission operator as a generalization of the emission kernel.
\begin{definition}[Forward Emission Operator]
\label{def::forward_emi_linear}
We define the following forward emission operator $\UU^\param_{h}: L^1(\cS) \mapsto L^1(\cA^{k}\times\cO^{k+1})$ for all $h\in[H]$,
\#\label{eq::def_U_linear}
(\UU^\param_{h} f)(\traj_{h}^{h+k}) = \int_{\cS} \PP^\param(\traj^{h+k}_h \given s_h)\cdot f(s_h) \ud s_h, \quad \forall  f\in L^1(\cS),~\forall \traj_{h}^{h+k}\in\cA^k\times\cO^{k+1}.
\#
Here recall that we denote by $\traj_{h}^{h+k} = \{\seq{a}{h}{h+k-1},\seq{o}{h}{h+k}\}\in\cA^k\times\cO^{k+1}$ the trajectory of interactions. In addition, recall that we define $\PP^\param(\traj^k_h \given s_h) = \PP^\param(\seq{o}{h}{h+k} \given s_h, \seq{a}{h}{h+k-1})$ for notational simplicity.
\end{definition}
We remark that when applying to a belief or a density over state $s_h$, the forward emission operator returns the density of trajectory $\traj^{h+k}_h$ of $k$ steps ahead of the $h$-th step.

\smallsec{Bottleneck Factor Interpretation of Low-Rank Transition.}
Recall that in Assumption \ref{asu::low_rank_POMDP}, the feature $\phi^*_h$ maps from the state-action pair $(s_h, a_h)\in\cS\times\cA$ to a $d$-dimensional simplex in $\Delta(d)$. Equivalently, one can consider the low-rank transition as a latent variable model, where the next state $s_{h+1}$ is generated by first generating a bottleneck factor $q_h\sim \phi^*(s_h, a_h)$ and then generating the next state $s_{h+1}$ by $[\psi^*(\cdot)]_{q_h}$. In other words, the probability array $\phi^*(s_h, a_h)\in\Delta(d)$ induces a transition dynamics from the state-action pair $(s_h, a_h)$ to the bottleneck factor $q_h\in[d]$ as follows,
\$
\PP_h({\lat_h} \given s_{h}, a_{h}) =\bigl[\phi^*_h(s_h, a_h)\bigr]_{{\lat_h}},\quad \forall {\lat_h} \in [d].
\$
Correspondingly, we write $\PP_h(s_{h+1}\given {\lat_h}) = [\psi^*_h(s_{h+1})]_{{\lat_h}}$ the transition probability from the bottleneck factor ${\lat_h}\in[d]$ to the state $s_{h+1}\in\cS$. See Figure \ref{fig::POMDP_Dag} for an illustration of the data generating process with the bottleneck factors. 

\begin{figure}
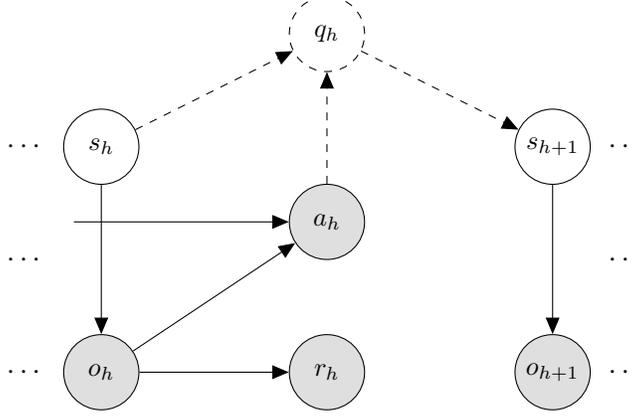

\begin{center}
 \tikz{
 \node at (0, 0) {$\cdots$};\node at (0, -1.5) {$\cdots$};
 \node at (0, -3) {$\cdots$};
 \node at (0.5, -1)(hist){};
 \node[latent, xshift=1cm, minimum size=10mm] (sh) {$s_h$};
 \node[latent, dashed, xshift=4cm, yshift=1.5cm,minimum size=10mm] (qh) {$\lat_h$};
 \node[latent, xshift=7cm, minimum size = 10mm](sh+1){$s_{h+1}$};
 \node[obs, xshift=4cm, yshift=-1cm, minimum size=10mm] (ah) {$a_h$};
 \node[obs, xshift=1cm, yshift=-3cm, minimum size=10mm] (oh) {$o_h$};
 \node[obs, xshift=4cm, yshift=-3cm, minimum size=10mm] (rh) {$r_h$};
 \node[obs, xshift=7cm, yshift=-3cm, minimum size=10mm] (oh+1) {$o_{h+1}$};
 \node at (8, 0) {$\cdots$};\node at (8, -1.5) {$\cdots$};\node at (8, -3) {$\cdots$};
 \edge{hist}{ah}
 \edge{sh}{oh};\edge{oh}{ah};\edge{oh}{rh};
 \edge{sh+1}{oh+1};
 \draw[dashed,->] (qh) -- (sh+1);
 \draw[dashed,->] (ah) -- (qh);
 \draw[dashed,->] (sh) -- (qh);
 }
\end{center}
\caption{The directed acyclic graph (DAG) of a POMDP with low-rank transition. Here $\{s_h, s_{h+1}\}$, $\{o_h, o_{h+1}\}$, $a_h$, $r_h$ are the states, observations, action, and reward, respectively. In addition, we denote by $\lat_h$ the bottleneck factor induced by the low-rank transition, which depends on the state and action pair $(s_h, a_h)$ and determines the density of next state $s_{h+1}$. 
In the DAG, we represent observable and unobservable variables by the shaded and unshaded nodes, respectively. In addition, we use the dashed node and arrows for the latent factor $\lat_h$ and its corresponding transitions, respectively, to differentiate such bottlenect factor from the state of the POMDP.
}\label{fig::POMDP_Dag}
\end{figure}


\smallsec{Understanding Bottleneck Factor.} Utilizing the low-rank structure of the state transition requires us to understand the bottleneck factors $\{\lat_h\}_{h\in[H]}$ defined by the low-rank transition. We highlight that the bottleneck factor $\lat_h$ is a compressed and sufficient factor for inference.
In particular, the bottleneck factor $\lat_h$ determines the distribution of next state $s_{h+1}$ through the feature $\psi^{*}_h(s_{h+1} = \cdot) = \PP(s_{h+1} = \cdot\given \lat_h = \cdot)$. Such a property motivate us to obtain our desired embedding via decomposing the density of trajectory based on the feature set $\{\psi^{*}_h\}_{h\in[H]^+}$. To achieve such a decomposition, we first introduce the following sufficiency condition for all the parameterized features $\psi^\param_h$ with $\param\in\Param$.

\begin{assumption}[Future Sufficiency]
\label{asu::lin_future_suff}
We define the mapping $g^\param_h: \cA^k\times\cO^{k+1}\mapsto \RR^d$ for all parameter $\param\in\Param$ and $h\in[H]$ as follows,
\$
g^\param_h = \Bigl[\UU^\param_h \bigl[\psi^\param_{h-1}\bigr]_1, \ldots, \UU^\param_h \bigl[\psi^\param_{h-1}\bigr]_d\Bigr]^\top,
\$
where we denote by $[\psi^\param_{h-1}\bigr]_i$ the $i$-th entry of the mapping $\psi^\param_{h-1}$ for all $i\in[d]$. 
We assume for some $k>0$ that the matrix
\$
M^\param_h = \int_{\cA^{k}\times\cO^{k+1}} g^\param_h(\traj^{h+k}_h)g^\param_h(\traj^{h+k}_h)^\top \ud \traj^{h+k}_h \in\RR^{d\times d}
\$
is invertible. We denote by $M^{\param, \dagger}_h$ the inverse of $M^\param_h$ for all parameter $\param\in\Param$ and $h\in[H]$.
\end{assumption}

Intuitively, the future sufficiency condition in Assumption \ref{asu::lin_future_suff} guarantees that the density of trajectory $\traj^{h+k}_h$ in the future captures the information of the bottleneck variable $\lat_{h-1}$, which further captures the belief at the $h$-th step. To see such a fact, we have the following lemma.

\begin{lemma}[Pseudo-Inverse of Forward Emission]
\label{lem::left_inv}
We define linear operator $\UU^{\param, \dagger}_{h}: L^1(\cA^{k}\times\cO^{k+1}) \mapsto L^1(\cS)$ for all $\param\in\Param$ and $h\in[H]$ as follows,
\#\label{eq::def_dagger_linear}
(\UU^{\param, \dagger}_{h} f)(s_{h}) =\int_{\cA^{k}\times\cO^{k+1}} \psi^\param_{h-1}(s_{h})^\top M^{\param, \dagger}_h g^\param_h(\traj^{h+k}_h)\cdot f(\traj_h^{h+k}) \ud \traj^{h+k}_h,
\#
where $f\in L^1(\cA^{k}\times\cO^{k+1})$ is the input of linear operator $\UU^{\param, \dagger}_{h}$ and $g^\param_h$ is the mapping defined in Assumption \ref{asu::lin_future_suff}. Under Assumptions \ref{asu::low_rank_POMDP} and \ref{asu::lin_future_suff}, it holds for all $h\in[H]$, $\param\in\Param$, and $\pi\in\Pi$ that
\$
\UU^{\param, \dagger}_{h}\UU^{\param}_{h}(\PP^{\param, \pi}_h) = \PP^{\param, \pi}_{h}.
\$
Here $\PP^{\param, \pi}_h \in L^1(\cS)$ maps from all state $s_h\in\cS$ to the probability $\PP^{\param, \pi}_h(s_h)$, which is the probability of visiting the state $s_h$ in the $h$-th step when following the policy $\pi$ and the model defined by parameter $\param$.
\end{lemma}
\begin{proof}
\vskip-5pt
See \S\ref{sec::pf_lem_left_inv} for a detailed proof.
\end{proof}
\vskip-5pt
By Lemma \ref{lem::left_inv}, the forward emission operator $\UU^\param_h$ defined in Definition \ref{def::forward_emi_linear} has a pseudo-inverse $\UU^{\param, \dagger}_{h}$ under the future sufficiency condition in Assumption \ref{asu::lin_future_suff}. Thus, one can identify the belief state by inverting the conditional density of the trajectory $\traj^{h+k}_h$ given the interaction history $\traj^h_{1}$.
More importantly, such invertibility further allows us to decompose the desired embedding $\Phi(\traj^H_1)$ in \eqref{eq::def_embedding} across steps, which we introduce in the sequel.

\subsection{Multi-Step Embedding Decomposition via Bellman Operator}

To accomplish the multi-step decomposition of embedding, we first define the Bellman operator as follows.

\begin{definition}[Bellman Operator]
\label{def::OO_linear}
We define the Bellman operators $\ObsB^\param_h(a_h, o_h): L^1(\cA^k\times\cO^{k+1}) \mapsto L^1(\cA^k\times\cO^{k+1})$ for all $(a_h, o_h)\in\cA\times\cO$ and $h\in[H]$ as follows,
\$
&\bigl(\ObsB^\param_h(a_h, o_h) f\bigr)(\traj^{h+k+1}_{h+1}) = \int_{\cS} \PP^\param(\traj^{h+k+1}_{h}\given s_h)\cdot (\UU^{\param, \dagger}_{h} f)(s_h) \ud s_h,\quad \forall\traj^{h+k+1}_{h+1}\in\cA^{k}\times\cO^{k+1}.
\$
Here recall that we denote by $\traj^{h+k+1}_{h} = \{\seq{o}{h}{h+k+1}, \seq{a}{h}{h+k}\}$ and $\PP^\param(\traj^{h+k+1}_{h}\given s_h) = \PP^\param(\seq{o}{h}{h+k+1}\given s_h, \seq{a}{h}{h+k+1})$ for notational simplicity.
\end{definition}
We call $\ObsB^\param_h(a_h, o_h)$ in Definition \ref{def::OO_linear} a Bellman operator as it performs a temporal transition from the density of trajectory $\traj^{h+k}_h$ to the density of trajectory $\traj^{h+k+1}_{h+1}$ and the observation $o_h$, given that one take action $a_h$ at the $h$-th step. More specifically, Assumption \ref{asu::lin_future_suff} guarantees that the density of trajectory $\traj^{h+k}_h$ identifies the density of $s_h$ in the $h$-th step. The Bellman operator then performs the transition from the density of $s_h$ to the density of the trajectory $\traj^{h+k+1}_{h+1}$ and observation $o_h$ given the action $a_h$. 
The following Lemma shows that our desired embedding $\Phi(\traj^H_1)$ can be decomposed into products of the Bellman operators defined in Definition \ref{def::OO_linear}.

\begin{lemma}[Embedding Decomposition]
\label{lem::traj_density_linear}
Under Assumptions \ref{asu::low_rank_POMDP} and \ref{asu::lin_future_suff}, it holds for all the parameter $\param\in\Param$ that
\$
 \PP^{\param}(\traj^H_1) = \frac{1}{A^k}\cdot\int_{\cA^k\times\cO^{k+1}}  \bigl[\ObsB^{\param}_{H}(o_{H}, a_{H})\ldots \ObsB^{\param}_1(o_1, a_1) \vecb^{\param}_1 \bigr](\traj^{H+k+1}_{H+1}) \ud \traj^{H+k+1}_{H+1}.
\$
Here recall that we denote by $\traj^{H+k+1}_{H+1} = \{\seq{a}{H+1}{H+k}, \seq{o}{H+1}{H+k+1}\}$ the dummy future trajectory. Meanwhile, we define the following initial trajectory density,
\$
\vecb^\param_1(\traj^k_1)= \UU^\param_{1} \mu_1 = \PP^{\param}(\traj^k_1), \quad \forall\traj^k_1 \in \cA^{k}\times\cO^{k+1}.
\$
\end{lemma}
\begin{proof}
\vskip-5pt
See \S\ref{sec::pf_lem_traj_density_linear} for a detailed proof.
\end{proof}
\vskip-5pt
By Lemma \ref{lem::traj_density_linear}, we can obtain the desired representation $\Phi(\traj^H_1) = \PP(\traj^H_1)$ based on the product of the Bellman operators. It now remains to estimate the Bellman operators across each step. In the sequel, we introduce an identity that allows us to recover the Bellman operators based on observations.

\smallsec{Estimating Bellman Operator.} In the sequel, we introduce the following notation to simplify our discussion,
\#
\futr_h &= \traj^{h+k}_h = \{o_h, a_h, \ldots, a_{h+k-1}, o_{h+k}\}\in\cA^k\times\cO^{k+1},\label{eq::def_futr_h}\\
\past_{h-1} &= \utraj^{h-1}_{h-\ell} = \{o_{h-\ell}, a_{h-\ell}, \ldots, o_{h-1}, a_{h-1}\} \in \cA^\ell\times\cO^{\ell}.\label{eq::def_past_h}
\#
We first define two density mappings that induce the identity of the Bellman Operator. We define the density mapping $\pone^{\param, \pi}_h: \cA^\ell\times\cO^\ell\mapsto L^1(\cA^k\times\cO^{k+1})$ as follows,
\#\label{eq::def_p1_linear}
\pone^{\param, \pi}_h(\past_{h-1}) = \PP^{\param, \pi}(\past_{h-1}, \futr_h = \cdot), \quad \forall \past_{h-1} \in\cA^{\ell}\times\cO^\ell.
\#
Intuitively, the density mapping $\pone^{\param, \pi}_h$ maps from an input trajectory 
$\past_{h-1}$ to the density of 
$\futr_h$, which represents the density of $k$-steps interactions following the input trajectory 
$\past_{h-1}$. 
Similarly, we define the density mapping $\ptwo^{\param, \pi}_h: \cA^{\ell+1}\times\cO^{\ell+1}\mapsto L^1(\cA^k\times\cO^{k+1})$ as follows,
\#\label{eq::def_p2_linear}
\ptwo^{\param, \pi}_h(\past_{h-1}, a_h, o_h) =
\PP^{\param, \pi}(\past_{h-1}, a_h, o_h, \futr_{h+1} = \cdot),
\quad \forall (\past_{h-1}, a_h, o_h) \in\cA^{\ell+1}\times\cO^{\ell+1}
\#
Based on the two density mappings defined in \eqref{eq::def_p1_linear} and \eqref{eq::def_p2_linear}, respectively, we have the following identity for all $h\in[H]$ and $\param\in\Param$,
\#\label{eq::linear_OO_id}
\ObsB^\param_h(a_h, o_h) \pone^{\param, \pi}_h(\past_{h-1}) = \ptwo^{\param, \pi}_h(\past_{h-1}, a_h, o_h), \quad \forall \past_{h-1}\in\cA^{\ell+1}\times\cO^{\ell+1}.
\#
See \S\ref{sec::pf_linear_OO_id} for the proof of \eqref{eq::linear_OO_id}. 
We highlight that the identity in \eqref{eq::linear_OO_id} allows us to estimate the Bellman operator $\ObsB^{\param^*}_h(a_h, o_h)$ under the true parameter $\param^*\in\Param$. In particular, both $\pone^{\param^*, \pi}_h$ and $\ptwo^{\param^*, \pi}_h$ are density mappings involving the observations and actions, and can be estimated based on observable variables from the POMDP. Upon fitting such density mappings, we can recover the Bellman operator $\ObsB^{\param^*}_h(a_h, o_h)$ by solving the identity in \eqref{eq::linear_OO_id}.

\smallsec{An Overview of Embedding Learning.} We now summarize the learning procedure of the embedding. First, we estimate the density mappings defined in \eqref{eq::def_p1_linear} and \eqref{eq::def_p2_linear} under the true parameter $\param^*$ based on interaction history. Second, we estimate the Bellman operators $\{\ObsB^{\param^*}_h(a_h, o_h)\}_{h\in[H]}$ based on the identity in \eqref{eq::linear_OO_id} and the estimated density mappings in the first step. Finally, we recover the embedding $\Phi(\traj^H_1)$ by assembling the Bellman operators according to Lemma \ref{lem::traj_density_linear}.

\section{Algorithm Description of ETC}
\label{sec::algo_decript_linear}
In the sequel, we decribe the procedure of \algoname. In summary, \algosp iteratively (i) interacts with the environment to collect observations, (ii) fits the density mappings defined in \eqref{eq::def_p1_linear} and \eqref{eq::def_p2_linear}, respectively, by observations, (iii) identifies a confidence set of parameters by fitting the Bellman equations according to \eqref{eq::linear_OO_id}, and (iv) conducts optimistic planning based on the fitted embeddings and the associated the confidence set.

To conduct \algoname, we first initialize a sequence of datasets indexed by the step $h\in[H]$ and the action sequences $\seq{a}{h-\ell}{h+k}\in\cA^{k+\ell+1}$,
\$
\cD^0_h(\seq{a}{h-\ell}{h+k}) = \emptyset.
\$
Meanwhile, we initialize a policy $\pi^0\in\Pi$, where $\Pi$ is the class of all deterministic policies. In the sequel, we introduce the update procedure of \algosp in the $t$-th iterate.

\subsection{Data Collection}
\label{sec::data_collection}
We first introduce the data collecting process of an agent with the policy $\pi^{t-1}$ in the $t$-th iterate. For each of the step $h\in[H]$ and the action sequence $\seq{a}{h-\ell}{h+k} \in\cA^{k+\ell+1}$, the agent first execute the policy $\pi^{t-1}$ until the $(h-\ell)$-th step, and collects a sequence of actions and observations as follows,
\$
\sp{t}{}{\seq{a}{1-\ell}{h-\ell-1}} = \bigl\{\sp{t}{}{a_{1-\ell}}, \ldots, \sp{t}{}{a_{h-\ell-1}}\bigr\},
\quad
\sp{t}{}{\seq{o}{1-\ell}{h-\ell}} = \bigl\{\sp{t}{}{o_{1-\ell}}, \ldots, \sp{t}{}{o_{h-\ell}}\bigr\}.
\$
Here we use the superscript $t$ to denote the observations and actions acquired in the $t$-th iterate. Correspondingly, we denote by $\sp{t}{}{\traj_{h-\ell}^{h-1}} = \{\sp{t}{}{\seq{a}{1-\ell}{h-\ell-1}}, \sp{t}{}{\seq{o}{1-\ell}{h-\ell}}\}$ the interaction history from the $(h-\ell)$-th step to the $(h-1)$-th step. Then, the agent execute $\seq{a}{h-\ell}{h+k}$ regardless of the observations and collect the following observation sequence,
\$
\sp{t}{}{\seq{o}{h-\ell+1}{h+k+1}} = \bigl\{\sp{t}{}{o_{h-\ell+1}}, \ldots, \sp{t}{}{o_{h+k+1}}\bigr\}.
\$
Finally, we store the observation sequence $\sp{t}{}{\seq{o}{h-\ell}{h+k+1}}$ generated by fixing the action sequence $\seq{a}{h-\ell}{h+k}$ into a dataset indexed by such action sequence, namely,
\$
\cD^t_h(\seq{a}{h-\ell}{h+k}) \leftarrow \cD^{t-1}_h(\seq{a}{h-\ell}{h+k}) \cup \bigl\{\sp{t}{}{\seq{o}{h-\ell}{h+k+1}}\bigr\}.
\$
\subsection{Density Estimation} 
Upon collecting the data, we follow the embedding learning procedure and fit the density mappings for the estimation of Bellman operator. In practice, various approaches are available in fitting the density by observations, including the maximum likelihood estimation (MLE), the generative adversial approaches, and the reproducing kernel Hilbert space (RKHS) density estimation. In what follows, we unify such density estimation approaches by a density estimation oracle.






\begin{assumption}[Density Estimation Oracle]
\label{asu::density_est}
We assume that we have access to a density estimation oracle $\est(\cdot)$. Moreover, for all $\delta>0$ and dataset $\cD$ drawn from the density $p$ of size $n$ following a martingale process, we assume that
\$
\|\est(\cD) - p\|_1 \leq C\cdot \sqrt{w_{\est}\cdot \log(1/\delta)/n}
\$
with probability at least $1 - \delta$. Here $C>0$ is an absolute constant and $w_{\est}$ is a parameter that depends on the density estimation oracle $\est(\cdot)$.
\end{assumption}
We highlight that such convergence property can be achieved by various density estimations. In particular, when the function approximation space $\cP$ of $\est(\cdot)$ is finite, Assumption \ref{asu::density_est} holds for the maximum likelihood estimation (MLE) and the generative adversial approach with $w_{\est} = \log|\cP|$ \citep{geer2000empirical, zhang2006e, agarwal2020flambe}. Meanwhile, $w_{\est}$ scales with the entropy integral of $\cP$ endowed with the Hellinger distance if $\cP$ is infinite \citep{geer2000empirical, zhang2006e}. In addition, Assumption \ref{asu::density_est} holds for the RKHS density estimation \citep{gretton2005measuring,smola2007hilbert,cai2022sample} with  $w_{\est} = \poly(d)$, where $d$ is rank of the low-rank transition \citep{cai2022sample}.


We now fit the density mappings based on the density estimation oracle. For each step $h\in[H]$ and action sequence $\seq{a}{h-\ell}{h+k} \in\cA^{k+\ell+1}$, we first fit the density of trajectory as follows,
\$
\hat\PP^t_{h}(\cdot\given \seq{a}{h-\ell}{h+k}) =  \est\bigl(\cD^t_h(\seq{a}{h-\ell}{h+k})\bigr),
\$
where the dataset $\cD^t_h$ is updated based on the data collection procedure described in \S\ref{sec::data_collection}. Meanwhile, we define the following density mappings for the estimation of Bellman operators,
\#
\bigl[\hat\pone^t_h({\utraj^{h-1}_{h-\ell}})\bigr](\traj^{h+k}_h) &= \hat\PP^t_h(\traj^{h+k}_{h-\ell}),\label{eq::def_density_est_p1_linear}\\
\bigl[\hat \ptwo^t_h(\utraj^h_{h-\ell})\bigr](\traj^{h+k+1}_{h+1}) &= \hat\PP^t_h(\traj^{h+k+1}_{h-\ell})\label{eq::def_density_est_p2_linear}.
\#
Here recall that we define the trajectories $\utraj^{h}_{h-\ell} = \{\seq{a}{h-\ell}{h}, \seq{o}{h-\ell}{h}\}$ and $\traj^{h+k+1}_{h-\ell} = \{\seq{a}{h-\ell}{h+k}, \seq{o}{h-\ell}{h+k+1}\}$. Meanwhile, we write $\hat\PP^t_h(\traj^{h+k+1}_{h-\ell}) = \hat\PP^t(\seq{o}{h-\ell}{h+k+1}\given \seq{a}{h-\ell}{h+k})$ for notational simplicity. We remark that the density mappings $\hat\pone^t_h$ and $\hat \ptwo^t_h$ are estimations of the density mappings defined in \eqref{eq::def_p1_linear} and \eqref{eq::def_p2_linear}, respectively, under the true parameter $\param^*$ and the mixing policy that collects the sample. We then estimate the Bellman operators by minimizing the following objective,
\#\label{eq::def_obj_Bellman}
L^t_h(\param) = \sup_{\seq{a}{h-\ell}{h}\in\cA^{\ell+1}}\int_{\cO^{\ell+1}}\|\ObsB^{\param}_h(a_h, o_h)\hat\pone^t_h({\utraj^{h-1}_{h-\ell}}) - \hat\ptwo^t_h({\utraj^{h}_{h-\ell}})\|_1 \ud \seq{o}{h-\ell}{h}.
\#
We remark that the objective defined in \eqref{eq::def_obj_Bellman} is motivated by the identity in \eqref{eq::linear_OO_id}. In what follows, we introduce an exploration procedure based on the objective defined in \eqref{eq::def_obj_Bellman}. In addition, we acquire the estimation of initial trajectory density 
$
\hat\vecb^t_1(\traj^k_1) = \hat\PP^t_{1}(\traj^k_1)
$ by marginalizing the dummy past trajectory $\traj^0_{1-\ell}$ of $\hat\PP^t_{1}$. 

\subsection{Optimistic Planning} We remark that the objective defined in \eqref{eq::def_obj_Bellman} encapsulates the uncertainty in the estimation of the corresponding Bellman operator $\ObsB^{\param}_h(a_h, o_h)$. In particular, a smaller objective $L^t_h(\param)$ yields a higher confidence that $\param$ is close to the true parameter $\param^*$. Thus, we define the following confidence set of parameters,
\#\label{eq::def_CI_linear}
\cC^t = \Bigl\{\param\in\Param: \max\bigl\{\|\vecb^\param_1 - \hat\vecb^t_1\|_1, L^t_h(\param)\bigr\} \leq \beta_t\cdot \sqrt{1/t},\quad\forall h\in[H] \Bigr\},
\#
where $\beta_t$ is the tuning parameter in the $t$-th iterate. Meanwhile, for each parameter $\param\in\Param$, we can estimate the embedding
\$
\Phi^\param(\traj^H_1) = \PP^\param(\traj^H_1)
\$
based on the Bellman operators $\{\BB^\param_h\}_{h\in[H]}$ and Lemma \ref{lem::traj_density_linear}. Such embedding further allows us to evaluate a policy as follows,
\$
V^\pi(\param) = \int_{\cO^H} r(\seq{o}{1}{H})\cdot\PP^\param\bigl(\seq{o}{1}{H} \given \seq{(a^\pi)}{1}{H}\bigr) \ud \seq{o}{1}{H} = \int_{\cO^H}  r(\seq{o}{1}{H})\cdot \Phi^\param\bigl(\seq{o}{1}{H}, \seq{(a^\pi)}{1}{H}\bigr)\ud \seq{o}{1}{H},
\$
where we define $V^\pi(\param)$ as the cumulative rewards of $\pi$ in the POMDP induced by the parameter $\param\in\Param$. Meanwhile, we define $\seq{(a^\pi)}{1}{H} = (a^\pi_1, \ldots, a^\pi_{H})$, where the actions $a^\pi_{h}$ are the action taken by the deterministic policy $\pi$ in the $h$-th step given the observations. 

To conduct optimistic planning, we seek for the policy that maximizes the return among all parameters $\param\in\cC^t$ and the corresponding features. The update of policy takes the following form,
\$
\pi^t \leftarrow \argmax_{\pi\in\Pi}\max_{\param\in\cC^t} V^\pi(\param),
\$
where we denote by $\Pi$ the set of all deterministic policies. We summarize \algosp in Algorithm \ref{alg::POMDP_linear}.

\begin{algorithm}[htpb]
\caption{Embed to Control}
\label{alg::POMDP_linear}
\begin{algorithmic}[1]
\REQUIRE Number of iterates $T$. A set of tuning parameters $\{\beta_t\}_{t\in[T]}$.
\STATE{\bf Initialization:} Set $\pi_0$ as a deterministic policy. Set the dataset $\cD^0_{h}(\seq{a}{h-\ell}{h+k})$ as an empty set for all $(h, \seq{a}{h-\ell}{h+k})\in[H]\times\cA^{k+\ell+1}$.
\FOR{$t\in[T]$}
\FOR{$(h, \seq{a}{h-\ell}{h+k})\in[H]\times\cA^{k+\ell+1}$}
\STATE Start a new episode from the $(1-\ell)$-th step.
\STATE Execute policy $\pi^{t-1}$ until the $(h-\ell)$-th step and receive the observations $\sp{t}{}{\seq{o}{1-\ell}{h-\ell}}$.
\STATE Execute the action sequence $\seq{a}{h-\ell}{h+k}$ regardless of the observations and receive the observations $\sp{t}{}{\seq{o}{h-\ell+1}{h+k+1}}$.
\STATE Update the dataset $
\cD^t_h(\seq{a}{h-\ell}{h+k}) \leftarrow \cD^{t-1}_h(\seq{a}{h-\ell}{h+k}) \cup \bigl\{\sp{t}{}{\seq{o}{h-\ell}{h+k+1}}\bigr\}.
$
\ENDFOR
\STATE Estimate the density of trajectory $\hat\PP^t_h(\cdot\given \seq{a}{h-\ell}{h+k}) \leftarrow \est\bigl(\cD^t(\seq{a}{h-\ell}{h+k})\bigr)$ for all $h\in[H]$.
\STATE Update the density mappings $\hat\pone^t_h$ and $\hat\ptwo^t_h$ as follows,
\vskip-10pt
\$
\hat\pone^t_h(\past_{h-1}) = \hat\PP^t_h(\past_{h-1}, \futr_{h} = \cdot),\qquad \hat\ptwo^t_h(\past_{h-1}, a_h, o_h) = \hat\PP^t_h(\past_{h-1}, a_h, o_h, \futr_{h+1} = \cdot).
\$
\STATE Update the initial density estimation $\hat\vecb^t_1(\traj^H_1) \leftarrow \hat\PP^t(\traj^H_1)$.
\STATE Update the confidence set $\cC^t$ by \eqref{eq::def_CI_linear}.
\STATE Update the policy $
\pi^t \leftarrow \argmax_{\pi\in\Pi}\max_{\param\in\cC^t} V^\pi(\param)
$. 
\ENDFOR
\STATE {\bf Output:} policy set $\{\pi^t\}_{t\in[T]}$.
\end{algorithmic}
\end{algorithm}

\section{Analysis}
\label{sec::analysis}
In what follows, we present the sample complexity analysis of \algosp presented in Algorithm \ref{alg::POMDP_linear}. Our analysis hinges on the following assumptions.

\begin{assumption}[Bounded Pseudo-Inverse]
\label{asu::bdd_left_inv_linear}
We assume that $\|\UU^{\param,\dagger}_h\|_{1\mapsto1} \leq \nu$ for all $\param\in\Param$ and $h\in[H]$, where $\nu > 0$ is an absolute constant.
\end{assumption}
We remark that the upper bound of the pseudo-inverse in Assumption \ref{asu::bdd_left_inv_linear} quantifies the fundamental difficulty of solving the POMDP. In particular, the pseudo-inverse of forward emission recovers the state density at the $h$-th step based on the trajectory $\traj^{h+k}_h$ from the $h$-th step to the $(h+k)$-th step. Thus, the upper bound $\nu$ on such pseudo-inverse operator characterizes how ill-conditioned the belief recovery task is based on the trajectory $\traj^{h+k}_h$. In what follows, we impose a similar past sufficiency assumption.

\begin{assumption}[Past Sufficiency]
\label{asu::inv_ROO_linear}
We define for all $h\in[H]$ the following reverse emission operator $\ObsF^{\param, \pi}_h: \RR^d \mapsto L^1(\cO^\ell\times\cA^\ell)$ for all $h\in[H]$, $\pi\in\Pi$, and $\param\in\Param$,
\$
(\ObsF^{\param, \pi}_h v)({\utraj^{h-1}_{h-\ell}}) = \sum_{\lat_{h-1} \in[d]} [v]_{\lat_{h-1}} \cdot \PP^{\param, \pi}(\seq{o}{h-\ell}{h-1}\given {\lat_{h-1}}, \seq{a}{h-\ell}{h-1}), \quad \forall v\in\RR^d,
\$
where $({\utraj^{h-1}_{h-\ell}})\in\cA^\ell\times\cO^\ell$. We assume for some $\ell>0$ that the operator $\ObsF^{\param, \pi}_h$ is left invertible for all $h\in[H]$, $\pi\in\Pi$, and $\param\in\Param$. We denote by $\ObsF^{\param, \pi, \dagger}_h$ the left inverse of $\ObsF^{\param, \pi}_h$. We assume further that $\|\ObsF^{\param, \pi, \dagger}_h\|_{1\mapsto1} \leq \gamma$ for all $h\in[H]$, $\pi\in\Pi$, and $\param\in\Param$, where $\gamma > 0$ is an absolute constant.
\end{assumption}
We remark that the left inverse $\ObsF^{\param, \pi, \dagger}_h$ of reverse emission operator $\ObsF^{\param, \pi}_h$ recovers the density of the bottleneck factor $q_{h-1}$ based on the density of trajectory $\utraj^{h-1}_{h-\ell}$ from the $(h-\ell)$-th step to the $(h-1)$-th step. Intuitively, the past sufficiency assumption in Assumption \ref{asu::inv_ROO_linear} guarantees that the density of trajectory $\utraj^{h-1}_{h-\ell}$ from the past captures sufficient information of the bottleneck factor $q_{h-1}$, which further determines the state distribution at the $h$-th step. Thus, similar to the upper bound $\nu$ in Assumption \ref{asu::bdd_left_inv_linear}, the upper bound $\gamma$ in Assumption \ref{asu::inv_ROO_linear} characterizes how ill-conditioned the belief recovery task is based on the trajectory $\utraj^{h-1}_{h-\ell}$ generated by the policy $\pi$. 

In what follows, we analyze the mixture policy $\overline\pi^T$ of the policy set $\{\pi^t\}_{t\in[T]}$ returned by \algosp in Algorithmn \ref{alg::POMDP_linear}. In particular, the mixture policy $\overline\pi^T$ is executed by first sampling a policy $\pi$ uniformly from the policy set $\{\pi^t\}_{t\in[T]}$ in the beginning of an episode, and then executing $\pi$ throughout the episode.


\begin{theorem}
\label{thm::sample_complexity_linear}
Let $\overline\pi^T$ be the mixture policy of the policy set $\{\pi^t\}_{t\in[T]}$ returned by Algorithm \ref{alg::POMDP_linear}. Let $\beta_t = (\nu+1)\cdot A^{2k}\cdot \sqrt{w_{\est}\cdot(k+\ell)\cdot\log(H\cdot A\cdot T)}$ for all $t\in[T]$ and
\$
T = \cO\bigl(\gamma^2\cdot \nu^4\cdot d^2\cdot w^2_{\est}\cdot H^2\cdot A^{2(2k+\ell)}\cdot(k+\ell)\cdot \log( H\cdot A/\epsilon)/\epsilon^2 \bigr).
\$
Under Assumptions \ref{asu::low_rank_POMDP}, \ref{asu::lin_future_suff}, \ref{asu::density_est}, \ref{asu::bdd_left_inv_linear}, and \ref{asu::inv_ROO_linear}, it holds with probability at least $1-\delta$ that $\overline\pi^T$ is $\epsilon$-suboptimal.
\end{theorem}
\begin{proof}
\vskip-5pt
See \S\ref{sec::pf_thm_sample_complexity_linear} for a detailed proof.
\end{proof}
\vskip-5pt
In Theorem \ref{thm::sample_complexity_linear}, we fix the lengths of future and past trajectories $k$ and $\ell$, respectively, such that Assumptions \ref{asu::lin_future_suff} and \ref{asu::inv_ROO_linear} holds.
Theorem \ref{thm::sample_complexity_linear} shows that the mixture policy $\overline\pi^T$ of the policy set $\{\pi^t\}_{t\in[T]}$ returned by \algosp is $\epsilon$-suboptimal if the number of iterations $T$ scales with $\cO(1/\epsilon^2)$. We remark that such a dependency regarding $\epsilon$ is information-therotically optimal for reinforcement learning in MDPs \citep{ayoub2020model, agarwal2020flambe, modi2021model, uehara2021representation}, 
which is a special case of POMDPs. In addition, the sample complexity $T$ depends polynomially on the length of horizon $H$, number of actions $A$, the dimension $d$ of the low-rank transition in Assumption \ref{asu::low_rank_POMDP}, and the upper bounds $\nu$ and $\gamma$ in Assumptions \ref{asu::bdd_left_inv_linear} and \ref{asu::inv_ROO_linear}, respectively. We highlight that the sample complexity depends on the observation and state spaces only through the dimension $d$ of the low-rank transition, extending the previous sample efficiency analysis of tabular POMDPs \citep{azizzadenesheli2016reinforcement, jin2020sample}. 
In addition, the sample complexity depends on the upper bounds of the operator norms $\nu$ and $\gamma$ in Assumptions \ref{asu::bdd_left_inv_linear} and \ref{asu::inv_ROO_linear}, respectively, which quantify the fundamental difficulty of solving the POMDP. See \S\ref{sec::tabular_POMDP} for the analysis under the tabular POMDP setting.

\section{Conclusion}
In this paper, we propose Embed to Control (ETC) as a unified framework for embedding and control in POMDPs. In particular, by exploiting the low-rank transition and the future sufficiency condition, we decompose the embedding learning into the learning of Bellman operators across multiple steps. By assembling the Bellman operators, we identify a sufficient embedding for the control in the POMDP. Moreover, we identify a confidence set of parameters fitting the Bellman operators, which further allows us to conduct exploration. Our analysis shows that ETC attains the $\cO(1/\epsilon^2)$ sample complexity to attain an $\epsilon$-suboptimal policy. To our best knowledge, we provide the first sample efficiency analysis for representation learning in POMDPs with infinite observation and state spaces. 

\section*{Acknowledgements}
Zhaoran Wang acknowledges National Science Foundation (Awards 2048075, 2008827, 2015568, 1934931), Simons Institute (Theory of Reinforcement Learning), Amazon, J.P. Morgan, and Two Sigma for their supports.








\bibliographystyle{ims}
\bibliography{POMDP.bib}

\newpage

\appendix
\renewcommand{\arraystretch}{1.5}
\section*{List of Notation}
In the sequel, we present a list of notations in the paper.
\begin{table}[htpb]
\begin{center}
\begin{tabular}{ | m{3.25cm} || m{12cm}|} 
 \hline
  Notation &  Explanation \\ 
  \hline
  \hline
  $\cS$, $\cA$, $\cO$  & The state, action, and observation spaces, respectively. \\ 
  \hline
  $A$, $H$ & The capacity of action space $|\cA|$ and the length of an episode, respectively.  \\ 
  \hline
  $\Phi$& The embedding of trajectory defined in \eqref{eq::def_embedding}.\\
  \hline
  $\PP_h(s_{h+1}\given s_h, a_h)$ & The transition probability from $(s_h, a_h)$ to $s_{h+1}$.  \\ 
  \hline
   $\OO_h(o_h\given s_h)$& The emission probability of observing $o_h$ given $s_h$.\\
  \hline
  $[H]^{+}$ & The set of steps $\{1-\ell, \ldots, H+k\}$ of the extended POMDP.\\
  \hline
  $\seq{a}{h}{h+k-1}$, $\seq{o}{h}{h+k}$ & The sequences of actions and observations $\{a_h, \ldots, a_{h+k-1}\}$ and $\{o_h, \ldots, o_{h+k}\}$, respectively.\\
  \hline
  $\traj^{h+k}_h$& The sequence of interactions $\{o_h, a_h, \ldots, o_{h+k-1}, a_{h+k-1}, o_{h+k}\}$ from the $h$-th step to the $(h+k)$-th step.\\
  \hline
  $\utraj^{h+k}_h$& The sequence of interactions $\{o_h, a_h, \ldots, o_{h+k-1}, a_{h+k-1}, o_{h+k}, a_{h+k}\}$ from the $h$-th step to the $(h+k)$-th step, including the $(h+k)$-th action.\\
  \hline
   $\PP(\traj^{h+k}_h)$, $\PP(\traj^{h+k}_h\given s_h)$& The conditional densities $\PP(\seq{o}{h}{h+k}\given \seq{a}{h}{h+k-1})$ and $\PP(\seq{o}{h}{h+k}\given s_h, \seq{a}{h}{h+k-1})$, respectively.\\
  \hline
  $z_h$, $w_{h-1}$ & The shorthand for the sequences of interactions $\traj^{h+k}_h$ and $\utraj^{h-1}_{h-\ell}$, respectively, on page $7$ of the paper.\\
  \hline

$\UU^\param_h$, $\UU^{\param, \dagger}_h$ & The forward emission operator and its pseudo-inverse defined in Definition \ref{def::forward_emi_linear} and Lemma \ref{lem::left_inv}, respectively.\\
  \hline
  $M^\param_h$, $M^{\param, \dagger}_h$ & The $d$-by-$d$ matrix and its inverse defined in Assumption \ref{asu::lin_future_suff}.\\
  \hline
  $\ObsB^\param_h$ & The Bellman operator defined in Definition \ref{def::OO_linear}.\\
  \hline
   $\phi^*$, $\psi^*$ & The unknown features of the low-rank POMDP in Assumption \ref{asu::low_rank_POMDP}.\\
  \hline
  $\phi^\param$, $\psi^\param$, $\OO^\param_h$ & The parameterized features and emission kernel in Definition \ref{def::parameter}.\\
  \hline

\end{tabular}
\end{center}
\end{table}

\begin{table}[htpb]
\begin{center}
\begin{tabular}{ | m{3.25cm}  || m{12cm}|} 
   
 \hline
  Notation &  Explanation \\ 
  \hline
 \hline
  $\PP^{\param}$, $\PP^{\param, \pi}$ & The probability densities corresponding to the transition dynamics defined by $\{\psi^\param, \phi^\param, \OO^\param_h\}$ and the policy $\pi$, respectively.\\

  \hline
 
  $\est(\cdot)$& The density estimation oracle defined in Assumption \ref{asu::density_est}.\\
  \hline

   $w_\est$, $\nu$, $\gamma$ & The parameters in Assumptions \ref{asu::density_est}, \ref{asu::bdd_left_inv_linear}, and \ref{asu::inv_ROO_linear}, respectively.\\
  \hline
  $\pone^{\param, \pi}_h$, $\ptwo^{\param, \pi}_h$ & The density mappings defined in \eqref{eq::def_p1_linear} and \eqref{eq::def_p2_linear}, respectively.\\
  \hline
  $\vecb^\param_1(\traj^H_1)$ & The density of initial trajectory $\PP^\param(\traj^H_1)$.\\
  \hline
  $\cD^t_h$, $\pi^t$, $\cC^t$ & The dataset, policy, and confidence set of parameters, respectively, in the $t$-th iteration of Algorithm \ref{alg::POMDP_linear}.\\
  \hline
  $\hat\pone^t_h$, $\hat\ptwo^t_h$, $\hat\vecb^t_1$& The estimated density mappings and initial trajectory density, respectively, in the $t$-th iteration of Algorithm \ref{alg::POMDP_linear}.\\
  \hline
  $L^t_h$& The objective function defined in \eqref{eq::def_obj_Bellman}.\\
  \hline

\end{tabular}
\end{center}
\end{table}

\newpage

\section{Proof of Preliminary Result}
\label{sec::pf_prel}
In the sequel, we present the proof of preliminary results in \S\ref{sec::embedding}.
\subsection{Proof of Lemma \ref{lem::left_inv}}
\label{sec::pf_lem_left_inv}
\begin{proof}
It holds for all time step $h\in[H]$, policy $\pi\in\Pi$, parameter $\param\in\Param$ that
\#\label{eq::pf_lem_left_inv_eq1}
\PP^{\param, \pi}_h(s_h) &= \int_{\cS\times\cA}\PP^\param_{h-1}(s_h\given s_{h-1}, a_{h-1})\cdot \PP^{\param, \pi}(s_{h-1}, a_{h-1})\ud s_{h-1}, a_{h-1}\notag\\
&=\psi^\param_{h-1}(s_h)^\top \int_{\cS\times\cA} \phi^\param_{h-1}(s_{h-1}, a_{h-1})\cdot \PP^{\param, \pi}(s_{h-1}, a_{h-1})\ud s_{h-1}, a_{h-1}\notag\\
&=\psi^\param_{h-1}(s_{h})^\top W_{h-1}(\param, \pi),
\#
where we define
\$
W_{h-1}(\param, \pi) = \int_{\cS\times\cA} \phi^\param_{h-1}(s_{h-1}, a_{h-1})\cdot \PP^{\param, \pi}(s_{h-1}, a_{h-1})\ud s_{h-1}, a_{h-1}.
\$
Meanwhile, recall that we define the following linear operator in Lemma \ref{lem::left_inv},
\$
(\UU^{\param, \dagger}_{h} f)(s_{h}) =\int_{\cA^{k}\times\cO^{k+1}} \psi^\param_{h-1}(s_{h})^\top z^\param_h(\traj_h^{h+k}) \cdot f(\traj_h^{h+k}) \ud \traj^{h+k}_h, \quad \forall f\in L^1(\cA^{k}\times\cO^{k+1}),~\forall s_h\in\cS,
\$
where we define
\$
z^\param_h(\traj^{h+k}_h) = M^{\param, \dagger}_h(\UU^\param_h \psi^\param_{h-1})(\traj^{h+k}_h),\quad \forall \traj^{h+k}_h\in\cA^{k}\times\cO^{k+1},
\$
It thus follows from \eqref{eq::pf_lem_left_inv_eq1} that
\#
&\int_{\cA^{k}\times\cO^{k+1}} z^\param_h(\traj^{h+k}_h)(\UU^\param_h \PP^{\param, \pi}_h)(\traj^{h+k}_{h}) \ud \traj^{h+k}_h \notag\\
&\qquad= M^{\param, \dagger}_h\int_{\cA^{k}\times\cO^{k+1}} (\UU^\param_h \psi^\param_{h-1})(\traj^{h+k}_h)(\UU^\param_h \psi^\param_{h-1})(\traj^{h+k}_h)^\top W_{h-1}(\param, \pi) \ud \traj^{h+k}_h\notag\\
&\qquad = M^{\param, \dagger}_h M^{\param}_h W_{h-1}(\param, \pi) = W_{h-1}(\param, \pi).
\#
Here recall that we define
\$
M^{\param}_h = \int_{\cA^{k}\times\cO^{k+1}} (\UU^\param_h \psi^\param_{h-1})(\traj^{h+k}_h)(\UU^\param_h \psi^\param_{h-1})(\traj^{h+k}_h)^\top\ud \traj^{h+k}_h \in\RR^{d\times d}
\$
and $M^{\param,\dagger}_h$ as the inverse of $M^{\param}_h$ in Assumption \ref{asu::lin_future_suff}. Thus, we have
\#
\UU^{\param, \dagger}_{h}\UU^{\param}_{h}\bigl(\PP^{\param, \pi}_h(\cdot)\bigr) 
&= \psi^\param_{h-1}(\cdot)^\top \int_{\cA^{k}\times\cO^{k+1}} z^\param_h(\traj^{h+k}_h)(\UU^\param_h \PP^{\param, \pi}_h)(\traj^{h+k}_{h}) \ud \traj^{h+k}_h\notag\\
&=\psi^\param_{h-1}(\cdot)^\top W_{h-1}(\param, \pi) = \PP^{\param, \pi}_h(\cdot),
\#
which completes the proof of Lemma \ref{lem::left_inv}.
\end{proof}

\subsection{Proof of Equation \ref{eq::linear_OO_id}}
\label{sec::pf_linear_OO_id}
\begin{proof}
By the definition of Bellman operators in Definition \ref{def::OO_linear}, we have
\#\label{eq::pf_eq_id_1}
&\bigl(\ObsB^\param_h(a_h, o_h) \pone_h({\utraj^{h-1}_{h-\ell}})\bigr)(\traj^{h+k+1}_{h+1}) = \int_{\cS} \PP^\param( \traj^{h+k+1}_{h}\given s_h)\cdot \bigl(\UU^{\param, \dagger}_{h} \pone^{\param}_h({\utraj^{h-1}_{h-\ell}}) \bigr)(s_h) \ud s_h.
\#
Meanwhile, by the definition of $\pone^{\param}_h$ and $\UU^\param_h$ in \eqref{eq::def_p1_linear} and \eqref{eq::def_U_linear}, respectively, we have
\$
\bigl[\pone^{\param}_h({\utraj^{h-1}_{h-\ell}})\bigr](\traj^{h+k}_h) &= \PP^\param(\traj^{h+k}_{h-\ell}) = \int_{\cS} \PP^\param(\seq{o}{h-\ell}{h-1}, s_h\given \seq{a}{h-\ell}{h-1})\cdot \PP^\param(\traj^{h+k}_{h}\given s_h) \ud s_h\notag\\
&=\bigl(\UU^\param_{h} \PP^\param(\seq{o}{h-\ell}{h-1}, s_h = \cdot \given \seq{a}{h-\ell}{h-1})\bigr)(\traj^{h+k}_h).
\$
Thus, by Lemma \ref{lem::left_inv}, it holds that
\#\label{eq::pf_eq_id_2}
\UU^{\param, \dagger}_{h} \pone^{\param}_h({\utraj^{h-1}_{h-\ell}}) &= \UU^{\param, \dagger}_{h}\UU^\param_{h} \PP^\param( \seq{o}{h-\ell}{h-1}, s_h = \cdot\given \seq{a}{h-\ell}{h-1} )=\PP^\param(\seq{o}{h-\ell}{h-1}, s_h = \cdot\given \seq{a}{h-\ell}{h-1} ).
\#
Plugging \eqref{eq::pf_eq_id_2} into \eqref{eq::pf_eq_id_1}, we conclude that
\$
&\bigl(\ObsB^\param_h(a_h, o_h) \pone_h({\utraj^{h-1}_{h-\ell}})\bigr)(\traj^{h+k+1}_{h+1})= \int_{\cS} \PP^\param( \traj^{h+k+1}_{h}\given s_h)\cdot \PP^\param(s_h, \seq{o}{h-\ell}{h-1}\given \seq{a}{h-\ell}{h-1} ) \ud s_h= \PP^\param(\traj^{h+k+1}_{h-\ell}),
\$
where the second equality follows from the fact that the past observations $\seq{o}{h-\ell}{h-1}$ is independent of the forward observations $\seq{o}{h+1}{h+k+1}$ given the current state $s_h$. Thus, by the definition of $\ptwo^{\param}_h$ in \eqref{eq::def_p1_linear}, we conclude the proof of equation \ref{eq::linear_OO_id}.
\end{proof}

\subsection{Proof of Lemma \ref{lem::traj_density_linear}}
\label{sec::pf_lem_traj_density_linear}
\begin{proof}
We first define the following density function of initial trajectory,
\#\label{eq::pf_lem_tdl_eq1}
\vecb^\param_{1}(\traj^{1+k}_{1}) = (\UU^\param_1 \mu_1)(\traj^{1+k}_{1}) = \PP^\param(\traj^{1+k}_{1}) \in L^1(\cA^k\times\cO^{k+1}).
\#
Thus, it holds from the definition of Bellman operators in Definition \ref{def::OO_linear} that
\#\label{eq::pf_lem_tdl_eq2}
\bigl[\ObsB^\param_1(a_1, o_1) \vecb^\param_{1}\bigr](\traj^{k+2}_{2}) &= \int_{\cS} \PP^\param(\traj^{k+2}_{1}\given s_1)\cdot (\UU^{\param, \dagger}_{1} \vecb^\param_{1})(s_1) \ud s_1\notag\\
&=\int_{\cS} \PP^\param(\traj^{k+2}_{1}\given s_1) \cdot \mu_1(s_1) \ud s_1 = \PP(\traj^{k+2}_1), 
\#
where $\mu_1$ is the initial state density of the POMDP. Here the second equality follows from the left invertibility of the forward emission operator $\UU^\param_1$ in Lemma \ref{lem::left_inv} and the definition of $\vecb^\param_{1}$ in \eqref{eq::pf_lem_tdl_eq1}. Thus, by the recursive computation following \eqref{eq::pf_lem_tdl_eq2}, we obtain that
\$
\bigl[\ObsB^{\param}_{H}(o_{H}, a_{H})\ldots \ObsB^{\param}_1(o_1, a_1) \vecb^\param_1\bigr](\traj^{H+k+1}_{H+1}) = \PP^\param(\traj^{H+k+1}_1).
\$
Finally, by marginalizing over the dummy future trajectories \$\traj^{H+k+1}_{H+1} = \{a_{H+1}, o_{H+1}, \ldots, a_{H+k}, o_{H+k+1}\} \in\cA^k\times\cO^{k+1},\$
we conclude that
\$
\PP^{\param}(\traj^H_1) &= \frac{1}{A^k}\cdot \int_{\cA^k\times\cO^{k+1}} \PP^{\param}(\traj^{H+k+1}_{1}) \ud \traj^{H+K+1}_{H+1}\notag\\
&= \frac{1}{A^k}\cdot \int_{\cA^k\times\cO^{k+1}} \bigl[\ObsB^{\param}_{H}(a_{H}, o_{H})\ldots \ObsB^{\param}_1(a_1, o_1) \vecb^\param_1\bigr](\traj^{H+k}_H) \ud \seq{o}{h}{h+k}.
\$
Thus, we complete the proof of Lemma \ref{lem::traj_density_linear}.
\end{proof}

\section{Proof of Main Result}
\label{sec::pf_main}
In the sequel, we present the proof of the main result in \S\ref{sec::analysis}.
\subsection{Computing the Performance Difference}
In the sequel, we present lemmas for the sample efficiency analysis of \algoname. Our analysis is motivated by previous work \citep{jin2020sample, cai2022sample}. We first define linear operators $\{\Tran^\param_h, \diagO^\param_h\}_{h\in[H]}$ as follows,
\#
\bigl(\Tran^\param_h(a_h) f\bigr)(s_{h+1}) &= \int_{\cS} \PP^\param_{h}(s_{h+1}\given s_h, a_h) \cdot f(s_h) \ud s_h, \quad \forall f\in L^1(\cS), ~a_h\in\cA,\label{eq::def_tranOP_linear}\\
\bigl(\diagO^\param_{h}(o_h) f\bigr)(s_h) &= \OO^\param_h(o_h\given s_h)\cdot f(s_h),\quad \forall f \in L^1(\cS), ~o_h\in\cO \label{eq::def_diagO_linear}.
\#
It thus holds that
\#\label{eq::claim_OO_linear}
\ObsB^\param_h(a_h, o_h) = \UU^\param_{h+1}\Tran^\param_h(a_h)\diagO^\param_{h}(o_h)\UU^{\param, \dagger}_{h}.
\#
To see such a fact, note that we have for all $f\in L^1(\cA^{k}\times\cO^{k+1})$ that
\#\label{eq::pf_claim_OO_linear_eq1}
&\bigl(\UU^\param_{h+1}\Tran^\param_h(a_h)\diagO_{h}(o_h)\UU^{\param, \dagger}_{h} f\bigr)(\traj^{h+k+1}_{h+1}) \notag\\
&\qquad = \biggl(\UU^\param_{h+1} \int_{\cS} \PP^\param_{h}(\cdot\given s_h, a_h) \cdot \OO_h(o_h\given s_h)(\UU^{\param, \dagger}_{h} f)(s_h) \ud s_h\biggr)(\traj^{h+k+1}_{h+1})\notag\\
&\qquad = \int_{\cS^2} \PP^\param( \traj^{h+k+1}_{h+1}\given s_{h+1})\cdot \PP^\param_{h}(s_{h+1}\given s_h, a_h) \cdot \OO_h(o_h\given s_h)(\UU^{\param, \dagger}_{h} f)(s_h) \ud s_h\ud s_{h+1}\notag\\
&\qquad=\int_{\cS^2}\PP^\param(\seq{o}{h}{h+k+1}, s_{h+1}\given s_{h},  \seq{a}{h}{h+k}) (\UU^{\param, \dagger}_{h} f)(s_h) \ud s_h\ud s_{h+1},
\#
where the first and second equality  follows from the definitions of $\Tran^\param_h$, $\diagO_{h}$, and $\UU^\param_{h+1}$ in \eqref{eq::def_tranOP_linear}, \eqref{eq::def_diagO_linear}, and \eqref{eq::def_U_linear}, respectively. Meanwhile, the third equality follows from the fact that the POMDP is Markov with respect to the state and action pairs $(s_{h+1}, a_{h+1})$. Marginalizing over the state $s_{h+1}$ on the right-hand side of \eqref{eq::pf_claim_OO_linear_eq1}, we obtain for all $(\traj^{h+k+1}_{h+1})\in\cA^k\times\cO^{k+1}$ that
\$
\bigl(\UU^\param_{h+1}\Tran^\param_h(a_h)\diagO^\param_{h}(o_h)\UU^{\param, \dagger}_{h} f\bigr)(\traj^{h+k+1}_{h+1}) &= \int_{\cS} \PP^\param(\traj^{h+k+1}_h\given s_h)\cdot (\UU^{\param, \dagger}_{h} f)(s_h) \ud s_{h}\notag\\
& = \bigl(\ObsB^\param_{h}(a_h, o_h) f\bigr)(\traj^{h+k+1}_{h+1}), 
\$
where the second equality follows from the definition of Bellman operator $\ObsB^\param_{h}$ in Definition \ref{def::OO_linear}. Thus, we complete the proof of \eqref{eq::claim_OO_linear}.

\begin{lemma}[Performance Difference]
\label{lem::perf_diff_linear}
It holds for all policy $\pi\in\Pi$ and parameters $\param, \param'\in\Param$ that
\$
|V^\pi(\param) - V^\pi(\param')| &\leq H\cdot\nu\cdot \sum_{h=1}^{H-1}\sum_{\seq{a}{h-\ell}{h}\in\cA^{\ell+1}} \int_{\cO} \sum_{\lat_{h-1} \in[d]} \|u_{h, \lat_{h-1}}\|_1\ud o_h + H\cdot\nu\cdot\|\vecb^{\param}_1 - \vecb^{\param'}_1\|_1,
\$
where we define
\$
u_{h, \lat_{h-1}} = \bigl(\ObsB^\param_h(a_h, o_h) - \ObsB^{\param'}_h(a_h, o_h)\bigr)\UU^{\param'}_{h}\PP^{\param'}_{h}(s_{h} = \cdot\given {\lat_{h-1}})\cdot \PP^{\pi}({\lat_{h-1}} \given \seq{a}{h-\ell}{h-1}).
\$
\end{lemma}
\begin{proof}
See \S\ref{sec::pf_lem_perf_diff_linear} for a detailed proof.
\end{proof}

\subsection{Confidence Set Analysis}

We first present the following norm bound on Bellman operators.
\begin{lemma}[Norm Bound of Bellman Operator]
\label{lem::norm_bound_linear}
Under Assumptions \ref{asu::low_rank_POMDP}, \ref{asu::lin_future_suff}, and \ref{asu::bdd_left_inv_linear}, it holds for all $h\in[H]$, $\param\in\Param$, and $(a_h, o_h)\in\cA\times\cO$ that $\|\ObsB^{\param}_{h}(a_h, o_h)\|_{1\mapsto1} \leq \nu\cdot A^k$.
\end{lemma}
\begin{proof}
It holds for all $f\in L^1(\cA^{k}\times\cO^{k+1})$ that
\#\label{eq::pf_lem_norm_bound_linear_eq1}
\|\ObsB^{\param}_{h}(a_h, o_h) f\|_{1} &\le \int_{\cA^{k}\times\cO^{k+1}}\int_{\cS} \PP^\param(\traj^{h+k+1}_h\given s_h)\cdot|\UU^{\param, \dagger}_h f (s_h)| \ud s_h \ud \traj^{h+k+1}_{h+1}\notag\\
&\leq A^k\cdot\int_{\cS}|\UU^{\param, \dagger}_h f (s_h)| \ud s_h.
\#
Meanwhile, by the definition of $\UU^{\param, \dagger}_h$ in \eqref{eq::def_dagger_linear} and Assumption \ref{asu::bdd_left_inv_linear}, it holds that 
\#\label{eq::pf_lem_norm_bound_linear_eq2}
\int_{\cS}|\UU^{\param, \dagger}_h f (s_h)| \ud s_h  = \|\UU^{\param, \dagger}_h f (s_h)\|_1 \leq \nu\cdot \|f\|_1.
\#
Combining \eqref{eq::pf_lem_norm_bound_linear_eq1} and \eqref{eq::pf_lem_norm_bound_linear_eq2}, we conclude that
\$
\|\ObsB^{\param}_{h}(a_h, o_h) f \|_{1}\leq \nu\cdot A^k \cdot \|f\|_1,
\$
which completes the proof of Lemma \ref{lem::norm_bound_linear}.
\end{proof}

In what follows, we recall the definition of the reverse emission operator.
\begin{definition}[Reverse Emission]
We define for all $h\in[H]$ the following linear operator $\ObsF^{\param, \pi}_h: \RR^d \mapsto L^1(\cO^\ell\times\cA^\ell)$ for all $h\in[H]$, $\pi\in\Pi$, and $\param\in\Param$,
\$
(\ObsF^{\param, \pi}_h v)({\utraj^{h-1}_{h-\ell}}) = \sum_{\lat_{h-1} \in[d]} [v]_{\lat_{h-1}} \cdot \PP^{\param, \pi}(\seq{o}{h-\ell}{h-1}\given {\lat_{h-1}}, \seq{a}{h-\ell}{h-1}), \quad \forall v\in\RR^d,
\$
where $({\utraj^{h-1}_{h-\ell}})\in\cA^\ell\times\cO^\ell$.
\end{definition}
In addition, we define the following visitation measure of mix policy in the $t$-th iteration,
\$
\PP^t = \frac{1}{t}\cdot\sum^{t-1}_{\omega = 0}\PP^{\pi^\omega},
\$
where $\{\pi^\omega\}_{\omega \in [t]}$ is the set of policy returned by Algorithm \ref{alg::POMDP_linear}. We remark that the data collected by our data collection process in Algorithm \ref{alg::POMDP_linear} follow the trajectory density induced by $\PP^t$ in the $t$-th iterate. Hence, the estimated density $\hat\PP^t$ returned by our density estimator $\est(\cD^t)$ in the $t$-th iterate aligns closely to $\PP^t$. Meanwhile, recall that we define the following estimators in \eqref{eq::def_density_est_p1_linear} and \eqref{eq::def_density_est_p2_linear}, respectively,
\$
\bigl[\hat\pone^t_h({\utraj^{h-1}_{h-\ell}})\bigr](\traj^{h+k}_h) = \hat\PP^t_h(\traj^{h+k}_{h-\ell}), \qquad \bigl[\hat \ptwo^t_h(\utraj^h_{h-\ell})\bigr](\traj^{h+k+1}_{h+1}) = \hat\PP^t_h(\traj^{h+k+1}_{h-\ell}).
\$
Recall that we define the confidence set as follows,
\$
{\cC^t} = \biggl\{\param\in\Param:& \int_{\cO^{\ell+1}}\|\ObsB^{\param}_h(a_h, o_h)\hat\pone^t_h(\seq{a}{h-\ell}{h-1}) - \hat\ptwo^t_h({\utraj^{h}_{h-\ell}})\|_1 \ud \seq{o}{h-\ell}{h} \leq \beta_t,\quad \forall\seq{a}{h-\ell}{h}\in\cA^{\ell+1} \biggr\},
\$
where we select
\$
\beta_t = (\nu+1)\cdot A^{2k}\cdot \sqrt{w_{\est}\cdot(k+\ell)\cdot\log(H\cdot A\cdot T)/t}.
\$
In the sequel, we denote by $\param^t$ the parameter selected in optimistic planning. The following lemma guarantees that the true parameter $\param^*$ is included by our confidence set $\cC^t$ with high probability. Moreover, we show that initial density and the Bellman operators $\{\ObsB^{\param^t}\}$ corresponding to the parameter $\param^t$ aligns closely to that corresponding to the true parameter $\param^*$.
\begin{lemma}[Good Event Probability]
\label{lem::good_event_linear}
Under Assumptions \ref{asu::low_rank_POMDP}, \ref{asu::lin_future_suff}, and \ref{asu::inv_ROO_linear}, it holds with probability at least $1 - \delta$ that $\param^*\in{\cC^t}$. In addition, it holds for all $h\in[H]$ and $t\in[T]$ with probability at least $1 - \delta$ that
\#
&\|\vecb^{\param^t}_1 - \vecb^{\param^*}_1\|_1 = \cO(\nu\cdot A^{2k}\cdot\sqrt{w_{\est}\cdot (k+\ell)\cdot\log(H\cdot A\cdot T)/t}), \label{eq::good_event_linear_I}\\
&\sum_{\seq{a}{h-\ell}{h} \in \cA^{\ell+1}}\int_{\cO} \sum_{\lat_{h-1} \in[d]}\bigl\|\bigl(\ObsB^{\param^t}_h(a_h, o_h) - \ObsB^{\param^*}_h(a_h, o_h)\bigr)\UU^{\param^*}_{h}\PP^{\param^*}_h(s_h = \cdot\given {\lat_{h-1}})\bigr\|_1 \notag\\
&\qquad\cdot \PP^{\param^*, t}({\lat_{h-1}}\given \seq{a}{h-\ell}{h-1}) \ud o_h =\cO\bigl(\gamma\cdot\nu\cdot A^{2k+\ell}\cdot \sqrt{w_{\est}\cdot(k+\ell)\cdot\log(H\cdot A\cdot T)/t}\bigr), \label{eq::good_event_linear_II}
\#
where we define
\$
\PP^{\param^*, t}({\lat_{h-1}}\given \seq{a}{h-\ell}{h-1}) = \frac{1}{t}\cdot\sum^{t-1}_{\omega = 0} \PP^{\param^*, \pi^\omega}({\lat_{h-1}} \given \seq{a}{h-\ell}{h-1}).
\$
\end{lemma}
\begin{proof}
See \S\ref{pf::lem_good_event_linear} for a detailed proof.
\end{proof}

\subsection{Proof of Theorem \ref{thm::sample_complexity_linear}}
\label{sec::pf_thm_sample_complexity_linear}
We are now ready to present the sample complexity analysis of Algorithm \ref{alg::POMDP_linear}. 
\begin{proof}
It holds that
\#\label{eq::pf_thm_linear_eq0}
V^{*}(\param^*) - V^{\overline\pi^T}(\param^*) = \frac{1}{T}\cdot\sum^{T}_{t = 1} V^*(\param^*) - V^{\pi^t}(\param^*).
\#
It suffices to upper bound the performance difference
\$
V^*(\param^*) - V^{\pi^t}(\param^*)
\$
for all $t\in[T]$. By Lemma \ref{lem::good_event_linear}, it holds with probability at least $1-\delta$ that $\param^*\in{\cC^t}$ for all $t\in[T]$. Thus, by the update of $\pi^t$ in Algorithm \ref{alg::POMDP_linear}, it holds with probability at least $1-\delta$ that
\#\label{eq::pf_thm_linear_eq1}
V^*(\param^*) - V^{\pi^t}(\param^*) \leq V^{\pi^t}(\param^t) - V^{\pi^t}(\param^*).
\#
It now suffices to upper bound the performance difference on the right-hand side of \eqref{eq::pf_thm_linear_eq1}. By Lemma \ref{lem::perf_diff_linear}, it holds that
\#\label{eq::pf_thm_linear_eq1.5}
|V^{\pi^t}(\param^t) - V^{\pi^t}(\param^*)| &\leq \underbrace{H\cdot\nu\cdot \sum_{h=1}^{H-1}\sum_{\seq{a}{h-\ell}{h}\in\cA^{\ell+1}} \int_{\cO} \sum_{\lat_{h-1} \in[d]} \|u^t_{h, \lat_{h-1}}\|_1\cdot \PP^{\theta^*, \pi^t}({\lat_{h-1}} \given \seq{a}{h-\ell}{h-1})\ud o_h}_{\displaystyle \rm (i)} \notag\\
&\qquad + \underbrace{H\cdot\nu\cdot\|\vecb^{\param^t}_1 - \vecb^{\param^*}_1\|_1}_{\displaystyle \rm (ii)},
\#
where we write
\#\label{eq::pf_thm_linear_def_u}
u^t_{h, \lat_{h-1}} = \bigl(\ObsB^{\param^t}_h(a_h, o_h) - \ObsB^{\param^*}_h(a_h, o_h)\bigr)\UU^{\param^*}_{h}\PP^{\param^*}_{h}(s_{h} = \cdot\given {\lat_{h-1}} )
\#
for notational simplicity. We remark that the summation in term (i) is different from that in \eqref{eq::good_event_linear_II} of Lemma \ref{lem::good_event_linear}. In particular, by Lemma \ref{lem::good_event_linear}, it holds for all $h\in[H]$ and $t\in[T]$ that
\#\label{eq::pf_thm_linear_eq2}
&\sum_{\seq{a}{h-\ell}{h}\in\cA^{\ell+1}} \int_{\cO} \sum_{\lat_{h-1} \in[d]} \|w^t_{h, \lat_{h-1}}\|_1 \cdot \PP^{\param^*, t}({\lat_{h-1}}\given \seq{a}{h-\ell}{h-1}) \ud o_h \notag\\
&\qquad=\cO(\gamma\cdot \nu\cdot d\cdot A^{2k+\ell}\cdot \sqrt{w_{\est}\cdot(k+\ell)\cdot\log(H\cdot A\cdot T)/t})
\#
with probability at least $1 - \delta$, where we define
\$
\PP^{\param^*, t}({\lat_{h-1}}\given \seq{a}{h-\ell}{h-1}) = \frac{1}{t}\cdot\sum^{t-1}_{\omega = 0} \PP^{\param^*, \pi^\omega}({\lat_{h-1}}\given \seq{a}{h-\ell}{h-1}).
\$
The only difference between the left-hand side of \eqref{eq::pf_thm_linear_eq2} and the term (i) in \eqref{eq::pf_thm_linear_eq1.5} is the conditional density of the bottleneck factor $\lat_{h-1}$, which follows the visitation of $\pi^t$, namely, $\PP^{\param^*, \pi^t}$, in \eqref{eq::pf_thm_linear_eq1.5} but the mixture of visitation $\PP^{\param^*, t}$ in \eqref{eq::pf_thm_linear_eq2}. To upper bound \eqref{eq::pf_thm_linear_eq2} by \eqref{eq::pf_thm_linear_eq1.5}, we utilize the calculation trick proposed by \cite{jin2020sample}. In particular, we utilize the following lemma.
\begin{lemma}[Lemma 16 of \cite{jin2020sample}]
\label{lem::sum_trick}
Let $0\leq z_t \leq C_z$ and $0\leq w_t\leq C_w$ for all $t\in[T]$. We define $S_{t} = (1/t)\cdot \sum^{t}_{i = 1}w_i$ and $S_0 = 0$. Given 
\$
z_t\cdot S_{t-1} \leq C_z\cdot C_w \cdot C \cdot \sqrt{1/t}
\$
for all $t\in[T]$, it holds that
\$
\frac{1}{T}\cdot \sum^T_{t = 1} z_t\cdot w_t \leq 2C_z\cdot C_w \cdot (C+1)\cdot \sqrt{1/K}\cdot\log T.
\$
Here $C >0$ is an absolute constant.
\end{lemma}
\begin{proof}
See \cite{jin2020sample} for a detailed proof.
\end{proof}
It thus follows from \eqref{eq::pf_thm_linear_eq2} and Lemma \ref{lem::sum_trick} that
\#\label{eq::pf_thm_linear_eq3}
{\rm (i)}=\cO\bigl(\gamma\cdot\nu^2\cdot d\cdot H\cdot A^{2k+\ell}\cdot \sqrt{w_{\est}\cdot(k+\ell)\cdot\log T\cdot\log( H\cdot A\cdot T)/t}\bigr)
\#
with probability at least $1 - \delta$. Meanwhile, by \eqref{eq::good_event_linear_I} of Lemma \ref{lem::good_event_linear}, it holds that
\#\label{eq::pf_thm_linear_eq4}
{\rm(ii)} = H\cdot\nu\cdot\|\vecb^{\param^t}_1 - \vecb^{\param^*}_1\|_1 = \cO\bigl(H\cdot\nu^2\cdot A^{2k}\sqrt{w_{\est}\cdot (k+\ell)\cdot\log(H\cdot A\cdot T)/t}\bigr)
\#
with probability at least $1-\delta$. Finally, by plugging \eqref{eq::pf_thm_linear_eq3} and \eqref{eq::pf_thm_linear_eq4} into \eqref{eq::pf_thm_linear_eq2}, it holds for all $t\in[T]$ that
\#\label{eq::pf_thm_linear_eq5}
&|V^{\pi^t}(\param^t) - V^{\pi^t}(\param^*)| \notag\\
&\qquad= \cO\bigl(\gamma\cdot \nu^2\cdot d\cdot H \cdot A^{2k+\ell}\cdot\log T\cdot \sqrt{w_{\est}\cdot(k+\ell)\cdot\log( H\cdot A\cdot T)/t}\bigr)
\#
with probability at least $1-\delta$. Combining \eqref{eq::pf_thm_linear_eq1} and \eqref{eq::pf_thm_linear_eq5}, it holds with probability at least $1-\delta$ that
\$
&V^{*}(\param^*) - V^{\overline\pi^T}(\param^*) \notag\\
&\qquad= \cO\bigl(\gamma\cdot \nu^2\cdot d\cdot H \cdot A^{2k+\ell}\cdot\log T\cdot \sqrt{w_{\est}\cdot(k+\ell)\cdot\log( H\cdot A\cdot T)/T}\bigr).
\$
Thus, by setting 
\$
T = \cO\bigl(\gamma^2\cdot \nu^4\cdot d^2\cdot H^2\cdot A^{2(2k+\ell)}\cdot(k+\ell)\cdot \log( H\cdot A/\epsilon)/\epsilon^2 \bigr),
\$
it holds with probability at least $1-\delta$ that $V^{*}(\param^*) - V^{\overline\pi^T}(\param^*) \leq \epsilon$. Thus, we complete the proof of Theorem \ref{thm::sample_complexity_linear}.

\end{proof}

\section{Proof of Auxiliary Result}
\label{sec::pf_aux_res}
In the sequel, we present the proof of the auxiliary results in \S\ref{sec::pf_main}.

\subsection{Proof of Lemma \ref{lem::perf_diff_linear}}
\label{sec::pf_lem_perf_diff_linear}
\begin{proof}
By Lemma \ref{lem::traj_density_lemma}, it holds for all policy $\pi\in\Pi$ and parameter $\param\in\Param$ that
\$
V^\pi(\param) &= \int_{\cO^{H}} r(\seq{o}{1}{H-1})\cdot\PP^{\pi, \param}(\seq{o}{1}{H-1}) \ud \seq{o}{1}{H-1}\notag\\
&= \int_{\cO^{H}} r(\seq{o}{1}{H-1})\cdot\PP^{\param}\bigl(\seq{o}{1}{H-1}\biggiven \seq{(a^\pi)}{1}{H}\bigr) \ud \seq{o}{1}{H-1},
\$
where $\seq{(a^\pi)}{1}{H} = (a^\pi_1, \ldots, a^\pi_H)$ and the actions $a^\pi_{h} = \pi(\seq{o}{1}{h}, \seq{(a^\pi)}{1}{h-1})$ are taken by the policy $\pi$ for all $h\in[H]$. Following from the fact that $0\leq r(\seq{o}{1}{H-1}) \leq H$ for all observation array $\seq{o}{1}{H-1}\in\cO^{H}$, we obtain for all policy $\pi\in\Pi$ and parameters $\param, \param'\in\Param$ that
\#\label{eq::perf_diff_linear_eq-1}
|V^\pi(\param) - V^\pi(\param')| \leq H\cdot \int_{\cO^{H}}\bigl|\PP^{\param}\bigl(\seq{o}{1}{H-1}\biggiven \seq{(a^\pi)}{1}{H}\bigr) - \PP^{\param'}\bigl(\seq{o}{1}{H-1}\biggiven \seq{(a^\pi)}{1}{H}\bigr)\bigr| \ud\seq{o}{1}{H-1}, 
\#
where the actions $a^\pi_{h}$ are taken by the policy $\pi$ for all $h\in[H]$. In the sequel, we utilize a slight modification of Lemma \ref{lem::traj_density_linear}. In particular, following the same calculation as the proof of Lemma \ref{lem::traj_density_linear} in \S\ref{sec::pf_lem_traj_density_linear}, we have
\$
\PP^\param(\tau^{H+k}_1) = \bigl[\ObsB^{\param}_{H-1}(a_{H-1}, o_{H-1})\ldots \ObsB^{\param}_1(a_1, o_1)\vecb^\param_1\bigr](\tau^{H+k}_{H}).
\$
Thus, by marginalizing over the dummy future observations $\seq{o}{H+1}{H+k}$ and fixing the final observation $o_H$, we obtain for all dummy future actions $\seq{a}{H}{H+k-1}$ that
\#\label{eq::perf_diff_linear_eq0}
\PP^\param(\tau^{H}_1) = \int_{\cO^k} \ind_{o_H, \seq{a}{H}{H+k-1}}(\traj^{H+k}_1)\cdot\PP^\param(\traj^{H+k}_1) \ud \seq{o}{H+1}{H+K},
\#
where we define $\ind_{o_H, \seq{a}{H}{H+k-1}}$ the indicator that takes value one at the final observation $o_H$ and the fixed dummy future actions $\seq{a}{H}{H+k-1}$. By plugging \eqref{eq::perf_diff_linear_eq0} into \eqref{eq::perf_diff_linear_eq-1}, we have
\#\label{eq::perf_diff_linear_eq0.1}
|V^\pi(\param) - V^\pi(\param')| \leq H\cdot \int_{\cO^{H+k}} |f^{\param} - f^{\param'}|\bigl(\seq{o}{h}{h+k}, \seq{(a^\pi)}{H}{H+k-1}\bigr) \ud \seq{o}{1}{H+k},
\#
where $\seq{(a^\pi)}{H}{H+k-1} = (a^\pi_H, \ldots, a^\pi_{H+k-1})$ and the actions $a^\pi_{h}$ are taken by the policy $\pi$. Here we define
\$
f^{\param} = \ObsB^{\param}_{H-1}(a^\pi_{H-1}, o_{H-1})\ldots \ObsB^{\param}_1(a^\pi_1, o_1) \vecb^\param_1,
\$
where $\vecb^{\param}_1$ is the initial trajectory distribution for the first $k$ steps defined in Lemma \ref{lem::traj_density_linear}. Meanwhile, by the linearity of Bellman operators, we have
\#\label{eq::perf_diff_linear_eq0.2}
f^{\param} - f^{\param'} = \sum^{H-1}_{h = 0} \ObsB^{\param}_{H-1}( a^\pi_{H-1}, o_{H-1})\ldots \ObsB^\param_{h+1}(a^\pi_{h+1}, o_{h+1})v_h,
\#
where we define $v_0 = \vecb^{\param}_1 - \vecb^{\param'}_1$ and
\#\label{eq::def_vh_linear}
v_h &= \bigl(\ObsB^\param_h(a^\pi_h, o_h) - \ObsB^{\param'}_h(a^\pi_h, o_h)\bigr)\ObsB^{\param'}_{h-1}(a^\pi_{h-1}, o_{h-1})\ldots \ObsB^{\param'}_{1}(a^\pi_{1}, o_{1})\vecb^{\param'}_1, \quad h\in[H-1].
\#
By combining \eqref{eq::perf_diff_linear_eq0.1} and \eqref{eq::perf_diff_linear_eq0.2}, we have
\#\label{eq::perf_diff_linear_eq1}
|V^\pi(\param) - V^\pi(\param')| &\leq H\cdot \sum_{h=1}^{H-1}\int_{\cO^{H+k}} |\ObsB^{\param}_{H-1}( a^\pi_{H-1}, o_{H-1})\ldots \ObsB^\param_{h+1}(a^\pi_{h+1}, o_{h+1})v_h| \ud \seq{o}{1}{H+k}\notag\\
&\qquad + \int_{\cO^{H+k}}|\ObsB^{\param}_{H-1}(a^\pi_{H-1}, o_{H-1})\ldots \ObsB^\param_{1}(a^\pi_{1}, o_{1})(\vecb^{\param}_1 - \vecb^{\param'}_1)|\ud \seq{o}{1}{H+k}.
\#

The following lemma upper bounds the right-hand side of \eqref{eq::perf_diff_linear_eq1}.
\begin{lemma}
\label{lem::bound_RHS_linear}
Under Assumption \ref{asu::lin_future_suff}, it holds for all $h\in[H]$, $\pi\in\Pi$, and $v_h\in L^1(\cA^{k}\times\cO^{k+1})$ that
\$
\int_{\cO^{H+k-h}} |\ObsB^{\param}_{H-1}(o_{H-1}, a_{H-1})\ldots \ObsB^\param_{h+1}(a_{h+1}, o_{h+1})v_h| \ud \seq{o}{h+1}{H+k} \leq \nu\cdot \|v_h\|_1.
\$
\end{lemma}
\begin{proof}
See \S\ref{sec::pf_lem_bound_RHS_linear} for a detailed proof.
\end{proof}

By Lemma \ref{lem::bound_RHS_linear}, it follows from \eqref{eq::perf_diff_linear_eq1} that
\#\label{eq::perf_diff_linear_eq2}
|V^\pi(\param) - V^\pi(\param')| &\leq H\cdot\nu\cdot \sum_{h=0}^{H-1}\int_{\cO^{h}} \|v_h\|_1 \ud \seq{o}{1}{h} + H\cdot\nu\cdot \|\vecb^{\param}_1 - \vecb^{\param'}_1\|_1,
\#
where we define $v_0 = \vecb^{\param}_1 - \vecb^{\param'}_1$ and
\$
v_h &= \bigl(\ObsB^\param_h(a^\pi_h, o_h) - \ObsB^{\param'}_h(a^\pi_h, o_h)\bigr)\ObsB^{\param'}_{h-1}(a^\pi_{h-1}, o_{h-1})\ldots \ObsB^{\param'}_{1}(a^\pi_{1}, o_{1})\vecb^{\param'}_1, \quad h\in[H-1].
\$
Meanwhile, the following lemma upper bound the $L^1$-norm of $v_h$ for $h = 2,\ldots,H$.
\begin{lemma}
\label{lem::RHS_II_linear}
It holds for all $\pi\in\Pi$ and $h\in[H-1]$ that
\$
\int_{\cO^{h}} \|v_h\|_1 \ud \seq{o}{1}{h} &\leq \sum_{\seq{a}{h-\ell}{h}\in\cA^{\ell+1}} \int_{\cO} \sum_{\lat_{h-1}\in[d]} u_h \ud o_h,
\$
where we define
\$
v_h &= \bigl(\ObsB^\param_h(a^\pi_h, o_h) - \ObsB^{\param'}_h(a^\pi_h, o_h)\bigr)\ObsB^{\param'}_{h-1}(a^\pi_{h-1}, o_{h-1})\ldots \ObsB^{\param'}_{1}(a^\pi_{1}, o_{1})\vecb^{\param'}_1,  \notag\\
u_{h} &= \bigl\|\bigl(\ObsB^\param_h(a_h, o_h) - \ObsB^{\param'}_h(a_h, o_h)\bigr)\UU^{\param'}_{h}\PP^{\param'}_{h}(s_{h} = \cdot\given {\lat_{h-1}})\cdot \PP^{\pi}({\lat_{h-1}}\given \seq{a}{h-\ell}{h-1})\bigr\|_1,
\$
for all $h\in[H-1]$.
\end{lemma}
\begin{proof}
See \S\ref{sec::pf_lem_RHS_II_linear} for a detailed proof.
\end{proof}
Combining \eqref{eq::perf_diff_linear_eq2} and Lemma \ref{lem::RHS_II_linear}, we conclude that
\$
|V^\pi(\param) - V^\pi(\param')| &\leq H\cdot\nu\cdot \sum_{h=1}^{H-1}\sum_{\seq{a}{h-\ell}{h}\in\cA^{\ell+1}} \sum_{\lat_{h-1}\in[d]}\int_{\cO}\|u_{h, \lat_{h-1}}\|_1 \ud o_h + H\cdot\nu\cdot\|\vecb^{\param}_1 - \vecb^{\param'}_1\|_1,
\$
where we define
\$
u_{h, \lat_{h-1}} = \bigl(\ObsB^\param_h(a_h, o_h) - \ObsB^{\param'}_h(a_h, o_h)\bigr)\UU^{\param'}_{h}\PP^{\param'}_{h}(s_{h} = \cdot\given {\lat_{h-1}})\cdot \PP^{\param', \pi}({\lat_{h-1}}\given \seq{a}{h-\ell}{h-1}).
\$
Thus, we complete the proof of Lemma \ref{lem::perf_diff_linear}.
\end{proof}

\subsection{Proof of Lemma \ref{lem::good_event_linear}}
\label{pf::lem_good_event_linear}
\begin{proof}
We first show that $\param^*\in{\cC^t}$ with probability at least $1-\delta$. By Assumption \ref{asu::density_est}, it holds for all $t\in[T]$ that
\#\label{eq::pf_GE_linear_eq1}
\|\hat\vecb^{t}_1 - \vecb^{\param^*}_1\|_1 \leq \sqrt{w_{\est}\cdot (k+\ell)\cdot\log(H\cdot A\cdot T)/t}
\#
with probability at least $1-\delta$. Meanwhile, it holds that
\#\label{eq::pf_GE_linear_eq2}
&\int_{\cO^{\ell+1}}\|\ObsB^{\param^*}_h(a_h, o_h)\hat\pone^t_h(\utraj^{h-1}_{h-\ell}) - \hat\ptwo^t_h({\utraj^{h}_{h-\ell}})\|_1 \ud \seq{o}{h-\ell}{h}\notag\\
&\qquad\leq\int_{\cO^{\ell+1}}\|\ObsB^{\param^*}_h(a_h, o_h)\pone^{\param^*, \overline\pi^t}_h(\utraj^{h-1}_{h-\ell}) - \ptwo^{\param^*, \overline\pi^t}_h({\utraj^{h}_{h-\ell}})\|_1\ud \seq{o}{h-\ell}{h}\notag\\
&\qquad\qquad + \int_{\cO}\bigl\|\ObsB^{\param^*}_h(a_h, o_h)\bigl(\pone^{\param^*, \overline\pi^t}_h(\utraj^{h-1}_{h-\ell}) - \hat\pone^t_h(\utraj^{h-1}_{h-\ell})\bigr)\bigr\|_1\ud o_{h}\notag\\
&\qquad\qquad + \int_{\cO^{\ell+1}}\bigl\|\bigl(\ptwo^{\param^*, \overline\pi^t}_h - \hat\ptwo^t_h\bigr)({\utraj^{h}_{h-\ell}})\bigr\|_1 \ud \seq{o}{h-\ell}{h}.
\#
We now upper bound the right-hand side of \eqref{eq::pf_GE_linear_eq2}. According to the identity in \eqref{eq::linear_OO_id}, we have
\#\label{eq::pf_GE_linear_eq3}
\ObsB^{\param^*}_h(a_h, o_h)\pone^{\param^*, \overline\pi^t}_h(\seq{a}{h-\ell}{h-1}) - \ptwo^{\param^*, \overline\pi^t}_h({\utraj^{h}_{h-\ell}}) = 0
\#
for all $h\in[H]$ and $(\seq{a}{h-\ell}{h}, \seq{o}{h-\ell}{h})\in\cA^{\ell+1}\times\cO^{\ell+1}$. Meanwhile, by Assumption \ref{asu::density_est} and the update of density estimators $\hat\ptwo^t_h$ in \eqref{eq::def_density_est_p2_linear}, it holds for all $h\in[H]$, $t\in[T]$, and $\utraj^{h}_{h-\ell}\in\cA^\ell\times\cO^\ell$ that
\#\label{eq::pf_GE_linear_eq4}
\int_{\cO^{\ell+1}}\bigl\|\bigl(\ptwo^{\param^*, \overline\pi^t}_h - \hat\ptwo^t_h\bigr)({\utraj^{h}_{h-\ell}})\bigr\|_1 \ud \seq{o}{h-\ell}{h}
&=\sum_{\seq{a}{h+1}{h+k}\in\cA^{k}}\int_{\cO^{k+\ell+1}}  |(\hat \distP^{t}_h - \PP^{\param^*, \overline\pi^t}_h)(\traj^{h+k+1}_{h-\ell})|  \ud \seq{o}{h-\ell}{h+k+1}. \notag\\
& = \sum_{\seq{a}{h+1}{h+k}\in\cA^{k}} \|(\hat \distP^{t}_h - \PP^{\param^*, \overline\pi^t}_h)(\cdot\given \seq{a}{h-\ell}{h+k})\|_1 \notag\\
&\leq A^{k}\cdot \sqrt{w_{\est}\cdot(k+\ell)\cdot\log(H\cdot A\cdot T)/t}
\#
with probability at least $1 - \delta$. Similarly, by Assumption \ref{asu::density_est}, Lemma \ref{lem::norm_bound_linear}, and the update of density estimators $\hat\pone^t_h$ in \eqref{eq::def_density_est_p1_linear}, we further obtain for all $h\in[H]$ that
\#\label{eq::pf_GE_linear_eq5}
&\int_{\cO}\bigl\|\ObsB^{\param^*}_h(a_h, o_h)\bigl(\pone^{\param^*, \overline\pi^t}_h(\utraj^{h-1}_{h-\ell}) - \hat\pone^t_h(\utraj^{h-1}_{h-\ell})\bigr)\bigr\|_1\ud o_{h}\notag\\
&\qquad\leq \nu\cdot A^{2k}\cdot \sqrt{w_{\est}\cdot(k+\ell)\cdot\log(H\cdot A\cdot T)/t}
\#
with probability at least $1 - \delta$. Plugging \eqref{eq::pf_GE_linear_eq3}, \eqref{eq::pf_GE_linear_eq4}, and \eqref{eq::pf_GE_linear_eq5} into \eqref{eq::pf_GE_linear_eq2}, we obtain for all $h\in[H]$ that
\#\label{eq::pf_GE_linear_eq6}
\int_{\cO^{\ell+1}}\|\ObsB^{\param^*}_h(a_h, o_h)\hat\pone^t_h(\utraj^{h-1}_{h-\ell}) - \hat\ptwo^t_h({\utraj^{h}_{h-\ell}})\|_1 \ud \seq{o}{h-\ell}{h} \leq \beta_t\cdot\sqrt{1/t}
\#
with probability at least $1-\delta$. Thus, combining \eqref{eq::pf_GE_linear_eq2} and \eqref{eq::pf_GE_linear_eq6}, it holds that $\param^* \in{\cC^t}$ with probability at least $1-\delta$. In what follows, we prove \eqref{eq::good_event_linear_I} and \eqref{eq::good_event_linear_II}, respectively.

\smallsec{Part I: Proof of Upper Bound in \eqref{eq::good_event_linear_I}.} By the definition of confidence set ${\cC^t}$, it holds for all $t\in[T]$, $(\seq{a}{h}{h+k-1}, \seq{a}{h-\ell}{h})\in\cA^{k+\ell}$, and $h\in[H]$ that
\$
\|\vecb^{\param^t}_1 - \hat\vecb^t_1\|_1 \leq (1+\nu)\cdot A^{2k}\sqrt{w_{\est}\cdot (k+\ell)\cdot\log(H\cdot A\cdot T)/t}
\$
with probability at least $1- \delta$. Thus, by \eqref{eq::pf_GE_linear_eq1} and triangle inequality, it holds for all $t\in[T]$, $(\seq{a}{h}{h+k-1}, \seq{a}{h-\ell}{h})\in\cA^{k+\ell}$, and $h\in[H]$ that
\$
\|\vecb^{\param^t}_1 - \vecb^{\param^*}_1\|_1 &\leq \|\vecb^{\param^t}_1 - \hat\vecb^{t}_1\|_1 + \|\hat\vecb^{t}_1 - \vecb^{\param^*}_1\|_1\notag\\
& = \cO(\nu\cdot A^{2k}\cdot\sqrt{w_{\est}\cdot (k+\ell)\cdot\log(H\cdot A\cdot T)/t})
\$
with probability at least $1 - \delta$. Thus, we complete the proof of the upper bound in \eqref{eq::good_event_linear_I}.


\smallsec{Part II: Proof of Upper Bound in \eqref{eq::good_event_linear_II}.} It suffices to upper bound the following term for all $h\in[H]$ and $t\in[T]$,
\#\label{eq::pf_def_G}
G^t_h = \sum_{\seq{a}{h-\ell}{h}\in\cA^{\ell+1}} \biggl(\sum_{\lat_{h-1} \in[d]} \int_{\cO}\|u^t_{h, \lat_{h-1}}\|_1\ud o_{h}\biggr),
\#
where we write
\#\label{eq::pf_lem_linear_def_u}
&u^t_{h, \lat_{h-1}}(\cdot)= \bigl(\ObsB^{\param^t}_h(a_h, o_h) - \ObsB^{\param^*}_h(a_h, o_h)\bigr)\UU^{\param^*}_{h}\PP^{\param^*}_{h}(s_{h} = \cdot\given {\lat_{h-1}})\cdot \PP^{\theta^*, \overline\pi^t}({\lat_{h-1}} \given \seq{a}{h-\ell}{h-1})
\#
for notational simplicity. We remark that $u^t_{h, \lat_{h-1}}\in L^1(\cA^{k}\times\cO^{k+1})$ is a function in the space $L^1(\cA^{k}\times\cO^{k+1})$ by the definition of Bellman operators $\ObsB^{\param^t}_h$ and $\ObsB^{\param^*}_h$. In the sequel, we define the vector-valued function
\$
u^t_h = [u^t_{h, 1}, \ldots, u^t_{h, d}] \in \RR^d.
\$
It thus holds that 
\#\label{eq::pf_lem_linear_eq2}
G^t_h = \sum_{\seq{a}{h-\ell}{h}\in\cA^{\ell+1}}\int_{\cO^{k+2}}  \|u^t_h(\traj^{h+k+1}_{h+1})\|_1\ud \seq{o}{h}{h+k+1}.
\#
Here the integration and summation are taken with respect to the domain $\seq{o}{h+1}{h+k+1}\in\cO^{k+1}$ of $u^t_h$, the action sequence $\seq{a}{h-\ell}{h-1}\in\cA^\ell$ in \eqref{eq::pf_lem_linear_def_u}, and the action and observation pair $(a_h, o_h)\in\cA\times\cO$ in the Bellman operators that defines $u^t_{h, i}$ in \eqref{eq::pf_lem_linear_def_u}. We remark that in \eqref{eq::pf_lem_linear_eq2}, we abuse the notation slightly and write
\$
\|u^t_h(\traj^{h+k+1}_{h+1})\|_1 = \sum_{\lat_{h-1} \in[d]} |u^t_{h, \lat_{h-1}}(\traj^{h+k+1}_{h+1})|,
\$
where $u^t_{h, \lat_{h-1}}(\traj^{h+k+1}_{h+1})$ is defined in \eqref{eq::pf_lem_linear_def_u}. In the sequel, we upper bound the right-hand side of \eqref{eq::pf_lem_linear_eq2}. By Assumption \ref{asu::inv_ROO_linear}, it holds that
\#\label{eq::pf_lem_linear_(i)_eq1}
\|u^t_h\|_1 = \|\ObsF^{\param^*, \overline\pi_t, \dagger}_h\ObsF^{\param^*, \overline\pi_t}_h u^t_h\|_1 \leq \gamma\cdot \|\ObsF^{\param^*, \pi_t} u^t_h\|_1
\#
Meanwhile, by the definition of $\ObsF^{\param^*, \overline\pi^t}_h$, we have
\#\label{eq::pf_lem_linear_(i)_eq2}
&\ObsF^{\param^*, \overline\pi^t}_h u^t_h  = \sum_{\lat_{h-1} \in[d]} u^t_{h, \lat_{h-1}}\cdot \PP^{\param^*, \overline\pi^t}(\seq{o}{h-\ell}{h-1} \given {\lat_{h-1}}, \seq{a}{h-\ell}{h-1}).
\#
By the definition of $u^t_{h, \lat_{h-1}}$ in \eqref{eq::pf_lem_linear_def_u}, we further obtain that
\$
&u^t_{h, \lat_{h-1}}\cdot \PP^{\param^*, \overline\pi^t}(\seq{o}{h-\ell}{h-1} \given {\lat_{h-1}}, \seq{a}{h-\ell}{h-1}) = u^t_{h, \lat_{h-1}}\cdot \frac{\PP^{\param^*, \overline\pi^t}(\seq{o}{h-\ell}{h-1},{\lat_{h-1}} \given  \seq{a}{h-\ell}{h-1})}{\PP^{\param^*, \overline\pi^t}({\lat_{h-1}}  \given  \seq{a}{h-\ell}{h-1})}\notag\\
&\qquad = \bigl(\ObsB^{\param^t}_h(a_h, o_h) - \ObsB^{\param^*}_h(a_h, o_h)\bigr)\UU^{\param^*}_{h}\PP^{\param^*}_{h}(s_{h} = \cdot\given {\lat_{h-1}} )\cdot \PP^{\theta^*, \overline\pi^t}(\seq{o}{h-\ell}{h-1}, {\lat_{h-1}}\given \seq{a}{h-\ell}{h-1})\notag\\
&\qquad = \bigl(\ObsB^{\param^t}_h(a_h, o_h) - \ObsB^{\param^*}_h(a_h, o_h)\bigr)\UU^{\param^*}_{h}\PP^{\param^*, \overline\pi^t}(\seq{o}{h-\ell}{h-1}, {\lat_{h-1}} , s_{h} = \cdot\given \seq{a}{h-\ell}{h-1}).
\$
Thus, it follows from \eqref{eq::pf_lem_linear_(i)_eq2} and the linearity of Bellman operators $\ObsB^{\param}_h$ and $\UU^{\param}_{h}$ that 
\#\label{eq::pf_lem_linear_(i)_eq3}
(\ObsF^{\param^*, \pi_t} u^t_{h})({\utraj^{h-1}_{h-\ell}})&=\sum_{\lat_{h-1} \in[d]}\bigl(\ObsB^{\param^t}_h(a_h, o_h) - \ObsB^{\param^*}_h(a_h, o_h)\bigr)\UU^{\param^*}_{h}\PP^{\param^*, \overline\pi^t}(\seq{o}{h-\ell}{h-1}, {\lat_{h-1}}, s_{h} = \cdot\given \seq{a}{h-\ell}{h-1})\notag\\
&=\bigl(\ObsB^{\param^t}_h(a_h, o_h) - \ObsB^{\param^*}_h(a_h, o_h)\bigr)\UU^{\param^*}_{h}\PP^{\param^*, \overline\pi^t}(\seq{o}{h-\ell}{h-1}, s_{h} = \cdot\given \seq{a}{h-\ell}{h-1})\notag\\
&=\bigl(\ObsB^{\param^t}_h(a_h, o_h) - \ObsB^{\param^*}_h(a_h, o_h)\bigr)\pone^{\theta^*,\overline\pi^t }_h({\utraj^{h-1}_{h-\ell}}),
\#
where we marginalize the bottleneck factor $\lat_{h-1}$ in the second equality. Here recall that $\pone^{\theta^*,\overline\pi^t }_h$ is the density mapping defined in \eqref{eq::def_p1_linear} and $\overline\pi^t$ is the mixed policy in the $t$-th iteration. Plugging \eqref{eq::pf_lem_linear_(i)_eq3} into \eqref{eq::pf_lem_linear_(i)_eq1}, we obtain that
\#\label{eq::pf_lem_linear_(i)_bound}
\|u^t_h\|_1 \leq \gamma\cdot\sum_{\seq{a}{h-\ell}{h-1}\in\cA^{\ell}}\int_{\cO^{\ell}}|\bigl(\ObsB^{\param^t}_h(a_h, o_h) - \ObsB^{\param^*}_h(a_h, o_h)\bigr)\pone^{\theta^*,\overline\pi^t }_h({\utraj^{h-1}_{h-\ell}})| \ud \seq{o}{h-\ell}{h-1}.
\#
Plugging \eqref{eq::pf_lem_linear_(i)_bound} into \eqref{eq::pf_lem_linear_eq2}, we obtain that
\#\label{eq::pf_lem_linear_join}
&G^t_h\leq \gamma\cdot \sum_{\seq{a}{h-\ell}{h}\in\cA^{\ell+1}}\int_{\cO^{\ell+1}} \bigl\|\bigl(\ObsB^{\param^t}_h(a_h, o_h) - \ObsB^{\param^*}_h(a_h, o_h)\bigr)\pone^{\param^*, \overline\pi^t}_h({\utraj^{h-1}_{h-\ell}})\bigr\|_1 \ud \seq{o}{h-\ell}{h}.
\#
It remains to upper bound the right-hand side of \eqref{eq::pf_lem_linear_join}. By triangle inequality, we have
\#\label{eq::pf_GE_linear_II_eq1}
&\int_{\cO^{\ell+1}} \bigl\|\bigl(\ObsB^{\param^t}_h(a_h, o_h) - \ObsB^{\param^*}_h(a_h, o_h)\bigr)\pone^{\param^*, \overline\pi^t}_h({\utraj^{h-1}_{h-\ell}})\bigr\|_1 \ud \seq{o}{h-\ell}{h}\\
&\quad\leq\int_{\cO^{\ell+1}} \|\ObsB^{\param^t}_h(a_h, o_h) \pone^{\param^*, \overline\pi^t}_h({\utraj^{h-1}_{h-\ell}}) - \ptwo^{ \param^*, \overline\pi^t}_h({\utraj^{h}_{h-\ell}})\|_1 \ud \seq{o}{h-\ell}{h}\notag\\
&\quad\quad + \int_{\cO^{\ell+1}} \|\ObsB^{\param^*}_h(a_h, o_h) \pone^{\param^*, \overline\pi^t}_h({\utraj^{h-1}_{h-\ell}}) - \ptwo^{ \param^*, \overline\pi^t}_h({\utraj^{h}_{h-\ell}})\|_1 \ud \seq{o}{h-\ell}{h}.\notag
\#
By the identity in \eqref{eq::linear_OO_id}, it holds that
\#\label{eq::pf_GE_linear_II_eq0}
\ObsB^{\param^*}_h(a_h, o_h) \pone^{\param^*, \overline\pi^t}_h({\utraj^{h-1}_{h-\ell}}) - \ptwo^{ \param^*, \overline\pi^t}_h({\utraj^{h}_{h-\ell}}) = 0
\#
In the sequel, we upper bound the term
\$
\int_{\cO^{\ell+1}} \|\ObsB^{\param^t}_h(a_h, o_h) \pone^{\param^*, \overline\pi^t}_h({\utraj^{h-1}_{h-\ell}}) - \ptwo^{ \param^*, \overline\pi^t}_h({\utraj^{h}_{h-\ell}})\|_1 \ud \seq{o}{h-\ell}{h}
\$
on the right-hand side of \eqref{eq::pf_GE_linear_II_eq1}. The calculation is similar to that of the derivation of \eqref{eq::pf_GE_linear_eq6}. It holds for all $h\in[H]$ and $t\in[T]$ that
\#\label{eq::pf_GE_linear_II_eq2}
&\int_{\cO^{\ell+1}}\|\ObsB^{\param^t}_h(a_h, o_h)\pone^{\param^*, \overline\pi^t}_h(\utraj^{h-1}_{h-\ell}) - \ptwo^{\param^*, \overline\pi^t}_h({\utraj^{h}_{h-\ell}})\|_1 \ud \seq{o}{h-\ell}{h}\notag\\
&\qquad\leq\int_{\cO^{\ell+1}}\|\ObsB^{\param^t}_h(a_h, o_h)\hat\pone^t_h(\utraj^{h-1}_{h-\ell}) - \hat\ptwo^t_h({\utraj^{h}_{h-\ell}})\|_1\ud \seq{o}{h-\ell}{h} \notag\\
&\qquad\qquad+ \int_{\cO^{\ell+1}}\bigl\|\bigl(\ptwo^{\param^*, \overline\pi^t}_h - \hat\ptwo^t_h\bigr)({\utraj^{h}_{h-\ell}})\bigr\|_1 \ud \seq{o}{h-\ell}{h}\notag\\
&\qquad\qquad + \int_{\cO^{\ell+1}}\bigl\|\ObsB^{\param^t}_h(a_h, o_h)\bigl(\pone^{\param^*, \overline\pi^t}_h(\utraj^{h-1}_{h-\ell}) - \hat\pone^t_h(\utraj^{h-1}_{h-\ell})\bigr)\bigr\|_1\ud o_{h-\ell}^{h}.
\#
We now upper bound the right-hand side of \eqref{eq::pf_GE_linear_II_eq2}. By the definition of confidence set ${\cC^t}$, it holds for all $h\in[H]$ and $t\in[T]$ that
\#\label{eq::pf_GE_linear_II_eq3}
&\int_{\cO^{\ell+1}}\|\ObsB^{\param^t}_h(a_h, o_h)\hat\pone^t_h(\utraj^{h-1}_{h-\ell}) - \hat\ptwo^t_h({\utraj^{h}_{h-\ell}})\|_1\ud \seq{o}{h-\ell}{h}\leq \beta_t\cdot \sqrt{1/t}.
\#
Meanwhile, by Assumption \ref{asu::density_est} and the update of density estimators $\hat\ptwo^t_h$ in \eqref{eq::def_density_est_p2_linear}, it holds for all $h\in[H]$ and $t\in[T]$ that
\#\label{eq::pf_GE_linear_II_eq4}
\int_{\cO^{\ell+1}}\bigl\|\bigl(\ptwo^{\param^*, \overline\pi^t}_h - \hat\ptwo^t_h\bigr)({\utraj^{h}_{h-\ell}})\bigr\|_1 \ud \seq{o}{h-\ell}{h}&=\sum_{\seq{a}{h+1}{h+k} \in \cA^{k}}\int_{\cO^{k+\ell+1}}  |(\hat \distP^{t}_h - \PP^{\param^*, \overline\pi^t}_h)(\traj^{h+k+1}_{h-\ell})|  \ud \seq{o}{h-\ell}{h+k+1} \notag\\
&= \sum_{\seq{a}{h+1}{h+k}\in\cA^{k}} \|(\hat \distP^{t}_h - \PP^{\param^*, \overline\pi^t}_h)(\cdot\given \seq{a}{h-\ell}{h+k})\|_1 \notag\\
&\leq A^k\cdot \sqrt{w_{\est}\cdot(k+\ell)\cdot\log(H\cdot A\cdot T)/t}
\#
with probability at least $1 - \delta$. Similarly, by Assumption \ref{asu::density_est}, Lemma \ref{lem::norm_bound_linear}, and the update of density estimators $\hat\pone^t_h$ in \eqref{eq::def_density_est_p1_linear}, we further obtain for all $h\in[H]$ and $t\in[T]$ that
\#\label{eq::pf_GE_linear_II_eq5}
&\int_{\cO^{\ell+1}}\bigl\|\ObsB^{\param^t}_h(a_h, o_h)\bigl(\pone^{\param^*, \overline\pi^t}_h(\utraj^{h-1}_{h-\ell}) - \hat\pone^t_h(\utraj^{h-1}_{h-\ell})\bigr)\bigr\|_1\ud o_{h-\ell}^h\notag\\
&\qquad\leq \nu\cdot A^{2k}\cdot \sqrt{w_{\est}\cdot(k+\ell)\cdot\log(H\cdot A\cdot T)/t}
\#
with probability at least $1 - \delta$. Plugging \eqref{eq::pf_GE_linear_II_eq3}--\eqref{eq::pf_GE_linear_II_eq5} into \eqref{eq::pf_GE_linear_II_eq2}, we obtain for all $h\in[H]$ and $t\in[T]$ that
\#\label{eq::pf_GE_linear_II_eq6}
&\int_{\cO^{\ell+1}}\|\ObsB^{\param^*}_h(a_h, o_h)\hat\pone^t_h(\utraj^{h-1}_{h-\ell}) - \hat\ptwo^t_h({\utraj^{h}_{h-\ell}})\|_1 \ud \seq{o}{h-\ell}{h} \notag\\
&\qquad=\cO(\nu\cdot A^{2k}\cdot \sqrt{w_{\est}\cdot(k+\ell)\cdot\log(H\cdot A\cdot T)/t})
\#
with probability at least $1-\delta$. By plugging \eqref{eq::pf_GE_linear_II_eq0} and \eqref{eq::pf_GE_linear_II_eq6} into \eqref{eq::pf_GE_linear_II_eq1}, we obtain for all $h\in[H]$, $t\in[T]$, and $(\seq{a}{h}{h+k-1}, \seq{a}{h-\ell}{h})\in\cA^{k+\ell}$ that
\#\label{eq::pf_GE_linear_II_eq7}
&\int_{\cO^{\ell+1}} \bigl\|\bigl(\ObsB^{\param^t}_h(a_h, o_h) - \ObsB^{\param^*}_h(a_h, o_h)\bigr)\pone^{\param^*, \overline\pi^t}_h({\utraj^{h-1}_{h-\ell}})\bigr\|_1 \ud \seq{o}{h-\ell}{h}\notag\\
&\qquad = \cO(\nu\cdot A^{2k}\cdot \sqrt{w_{\est}\cdot(k+\ell)\cdot\log(H\cdot A\cdot T)/t})
\#
with probability at least $1-\delta$. Plugging \eqref{eq::pf_GE_linear_II_eq7} into \eqref{eq::pf_lem_linear_join}, we obtain for all $h\in[H]$ and $t\in[T]$ that
\$
G^t_h=\cO\bigl(\gamma\cdot\nu\cdot A^{2k+\ell}\cdot \sqrt{w_{\est}\cdot(k+\ell)\cdot\log(H\cdot A\cdot T)/t}\bigr)
\$
with probability at least $1-\delta$. Here $G^t_h$ is defined in \eqref{eq::pf_def_G}. Thus, we complete the proof of the upper bound in \eqref{eq::good_event_linear_II}.
\end{proof}

\subsection{Proof of Lemma \ref{lem::bound_RHS_linear}}
\label{sec::pf_lem_bound_RHS_linear}
\begin{proof}
Recall that we define linear operators $\{\Tran^\param_h, \diagO^\param_h\}_{h\in[H]}$ as follows,
\$
\bigl(\Tran^\param_h(a_h) f\bigr)(s_{h+1}) &= \int_{\cS} \PP^\param_{h}(s_{h+1}\given s_h, a_h) \cdot f(s_h) \ud s_h, \quad \forall f\in L^1(\cS), ~a_h\in\cA,\\
\bigl(\diagO^\param_{h}(o_h) f\bigr)(s_h) &= \OO^\param_h(o_h\given s_h)\cdot f(s_h),\quad \forall f \in L^1(\cS), ~o_h\in\cO .
\$
Recall that we have
\$
\ObsB^\param_h(a_h, o_h) = \UU^\param_{h+1}\Tran^\param_h(a_h)\diagO^\param_{h}(o_h)\UU^{\param, \dagger}_{h}.
\$
Thus, following from Lemma \ref{lem::left_inv} and the fact that $\Tran^\param_{h}(a_h) f$ is linear in $\psi^\param_h$, it further holds for all $h\in[H]$ and $v_h\in L^1(\cA^k\times\cO^{k+1})$ that
\#\label{eq::pf_lem_RHS_linear_eq0}
&B^{\param}_{H-1}(o_{H-1}, a_{H-1})\ldots \ObsB^\param_{h+1}(a_{h+1}, o_{h+1})v_h \notag\\
&\qquad= \UU^\param_{h} \Tran^\param_{H-1}(a_{H-1})\diagO^\param_{H-1}(o_{H-1})\ldots \Tran^\param_{h+1}(a_{h+1})\diagO^\param_{h+1}(o_{h+1})\UU^{\param, \dagger}_{h+1}v_h.
\#
We now prove Lemma \ref{lem::bound_RHS_linear} in the sequel. To begin with, it holds for all $h\in[H]$, $a_{h+1}\in\cA$, and $f\in L^1(\cS)$ that
\$
&\int_{\cO} \|\Tran^\param_{h+1}(a_{h+1})\diagO^\param_{h+1}(o_{h+1}) f\|_1 \ud o_{h+1} \notag\\
&\qquad\leq\int_{\cS^2\times\cO}\PP^\param(s_{h+2}\given s_{h+1}, a_{h+1})\cdot \OO^\param_{h+1}(o_{h+1}\given s_{h+1})\cdot |f(s_{h+1})| \ud o_{h+1}\ud s_{h+1}\ud s_{h+2} \notag\\
&\qquad = \int_{\cS} |f(s_{h+1})| \ud s_{h+1} = \|f\|_1.
\$
Inductively, it holds for all $h\in[H]$, $\seq{a}{h+1}{H-1}\in\cA^{H-h-1}$, and $f\in L^1(\cS)$ that
\#\label{eq::pf_lem_RHS_linear_eq1}
\int_{\cO^{H-h-1}}\|\Tran^\param_{H-1}(a_{H-1})\diagO^\param_{H-1}(o_{H-1})\ldots \Tran^\param_{h+1}(a_{h+1})\diagO^\param_{h+1}(o_{h+1})f\|_1\ud \seq{o}{h+1}{H-1}\leq \|f\|_1.
\#
Meanwhile, by the definition of $\UU_{H}^\param$ in \eqref{eq::def_U_linear} of Definition \ref{def::forward_emi_linear}, it holds for all $f\in L^1(\cS)$ and $\seq{a}{h}{h+k-1}\in \cA^{k}$ that
\#\label{eq::pf_lem_RHS_linear_eq2}
\int_{\cO^{k+1}}|(\UU_{H}^\param f)(\seq{o}{H}{H+k}, \seq{a}{H}{H+k-1})| \ud \seq{o}{H}{H+k} &\leq \int_{\cS\times\cO^{k+1}} \PP^\param(\seq{o}{H}{H+k}\given s_{H}, \seq{a}{H}{H+k-1})\cdot |f(s_H)| \ud \seq{o}{H}{H+k}\ud s_{H}\notag\\
& = \int_{\cS}|f(s_H)| \ud s_{H} = \|f\|_1.
\#
Combining \eqref{eq::pf_lem_RHS_linear_eq0}, \eqref{eq::pf_lem_RHS_linear_eq1}, and \eqref{eq::pf_lem_RHS_linear_eq2} with $h = H$, we obtain that
\#
&\int_{\cO^{H+k-h}} |\ObsB^{\param}_{H-1}(o_{H-1}, a_{H-1})\ldots \ObsB^\param_{h+1}(a_{h+1}, o_{h+1})v_h| \ud \seq{o}{h+1}{H+k} \notag\\
&\qquad =\int_{\cO^{H+k-h}}  | \UU^\param_{H} \Tran^\param_{H-1}(a_{H-1})\diagO^\param_{H-1}(o_{H-1})\ldots \Tran^\param_{h+1}(a_{h+1})\diagO^\param_{h+1}(o_{h+1})\UU^{\param, \dagger}_{h+1}v_h| \ud \seq{o}{h+1}{H+k}\notag\\
&\qquad \leq \int_{\cO^{H-h-1}}\|\Tran^\param_{H-1}(a_{H-1})\diagO^\param_{H-1}(o_{H-1})\ldots \Tran^\param_{h+1}(a_{h+1})\diagO^\param_{h+1}(o_{h+1})\UU^{\param, \dagger}_{h+1}v_h\|_1\ud \seq{o}{h+1}{H-1}\notag\\
&\qquad \leq \|\UU^{\param, \dagger}_{h+1}v_h\|_1.
\#
Finally, by Assumption \ref{asu::lin_future_suff}, it holds that
\$
\|\UU^{\param, \dagger}_{h+1}v_h\|_1 \leq \nu\cdot\|v_h\|_1.
\$
Thus, we completes the proof of Lemma \ref{lem::bound_RHS_linear}.
\end{proof}

\subsection{Proof of Lemma \ref{lem::RHS_II_linear}}
\label{sec::pf_lem_RHS_II_linear}
\begin{proof}
Recall that we define
\$
v_h &= \bigl(\ObsB^\param_h(a^\pi_h, o_h) - \ObsB^{\param'}_h(a^\pi_h, o_h)\bigr)\ObsB^{\param'}_{h-1}(a^\pi_{h-1}, o_{h-1})\ldots \ObsB^{\param'}_{1}(a^\pi_{1}, o_{1})\vecb^{\param'}_1, \quad \forall h\in[H],
\$
where the actions $a^\pi_{h}$ are taken by the policy $\pi$ for all $h\in[H]$. To accomplish the proof, we first handle the dependency of the actions $a^\pi_j$ on policy $\pi$ for $h-\ell\leq j\leq h$. To this end, we utilize the following upper bound,
\#\label{eq::pf_lem_RHS_II_linear_eq1}
\int_{\cO^{h}} \|v_h\|_1 \ud \seq{o}{1}{h} \leq \sum_{\seq{a}{h-\ell}{h}\in\cA^{\ell+1}} \int_{\cO^{H}}& \bigl\|\bigl(\ObsB^\param_h(a^\pi_h, o_h) - \ObsB^{\param'}_h(a_h, o_h)\bigr)\ObsB^{\param'}_{h-1}(a_{h-1}, o_{h-1})\ldots\\
&\ldots \ObsB^{\param'}_{h-\ell}(a_{h-\ell}, o_{h-\ell})\ObsB^{\param'}_{h-\ell-1}(a^\pi_{h-\ell-1}, o_{h-\ell-1}) \ldots\vecb^{\param'}_1\bigr\|_1 \ud \seq{o}{1}{h}.\notag
\#
Here we abuse the notation of index slightly for simplicity. We remark that the sequence of product of Bellman operators $B^{\param'}_{j}(a_j, o_j)$ ends at the index $j = 1$. Recall that we have
\$
\ObsB^\param_h(a_h, o_h) = \UU^\param_{h+1}\Tran^\param_h(a_h)\diagO^\param_{h}(o_h)\UU^{\param, \dagger}_{h},
\$
where the linear operators $\{\Tran^\param_h, \diagO^\param_h\}_{h\in[H]}$ are defined in \eqref{eq::def_tranOP_linear} and \eqref{eq::def_diagO_linear}, respectively. Thus, by Lemma \ref{lem::left_inv}, it holds for the right-hand side of \eqref{eq::pf_lem_RHS_II_linear_eq1} that
\#\label{eq::pf_lem_RHS_II_linear_eq2}
&\ObsB^{\param'}_{h-1}(a_{h-1}, o_{h-1})\ldots\ObsB^{\param'}_{h-\ell}(a_{h-\ell}, o_{h-\ell})\ObsB^{\param'}_{h-\ell-1}(a^\pi_{h-\ell-1}, o_{h-\ell-1}) \ldots\vecb^{\param'}_1\notag\\
&\qquad = \UU^{\param'}_{h} \Tran^{\param'}_{h-1}(a_{h-1})\diagO^\param_{h-1}(o_{h-1})\ldots \diagO^{\param'}_{1}(o_{1}) \mu_1,
\#
where $\mu_1 \in L^1(\cS)$ is the initial state probability density function. By the definition of linear operators $\{\Tran^\param_h, \diagO^\param_h\}_{h\in[H]}$ in \eqref{eq::def_tranOP_linear} and \eqref{eq::def_diagO_linear}, respectively, it further holds that
\#\label{eq::pf_lem_RHS_II_linear_eq3}
&\Tran^{\param'}_{h-1}(a_{h-1})\diagO^{\param'}_{h-1}(o_{h-1})\Tran^{\param'}_{h-2}(a_{h-2})\ldots \diagO^{\param'}_{1}(o_{1}) \mu_1 = \PP^{\param', \pi}(\seq{o}{1}{h-1}, s_{h} = \cdot \given \seq{a}{h-\ell}{h-1}) \in L^1(\cS).
\#
By plugging \eqref{eq::pf_lem_RHS_II_linear_eq2} and \eqref{eq::pf_lem_RHS_II_linear_eq3} into \eqref{eq::pf_lem_RHS_II_linear_eq1}, we obtain that
\#\label{eq::pf_lem_RHS_II_linear_eq4}
&\int_{\cO^{h}} \|v_h\|_1 \ud \seq{o}{1}{h}\leq \sum_{\seq{a}{h-\ell}{h}\in\cA^{\ell+1}} \int_{\cO^{h}} \bigl\|\bigl(\ObsB^\param_h(a_h, o_h) - \ObsB^{\param'}_h(a_h, o_h)\bigr)\UU^{\param'}_{h} \PP^{\param', \pi}(\seq{o}{1}{h-1}, s_{h} = \cdot \given \seq{a}{h-\ell}{h-1})\bigr\|_1 \ud \seq{o}{1}{h}.
\#
Meanwhile, it holds for all $s_{h}\in\cS$ that
\$
&\PP^{\param', \pi}(\seq{o}{1}{h-1}, s_{h} \given \seq{a}{h-\ell}{h-1}) = \sum_{\lat_{h-1}\in[d]} \PP^{\param'}_{h-1}(s_{h}\given {\lat_{h-1}})\cdot \PP^{\param', \pi}_{h-1}({\lat_{h-1}}, \seq{o}{h-\ell}{h-1}  \given \seq{a}{h-\ell}{h-1}).
\$
Thus, it follows from Jensen's inequality that
\#\label{eq::pf_lem_RHS_II_linear_eq5}
&\int_{\cO^{H}} \bigl\|\bigl(\ObsB^\param_h(a^\pi_h, o_h) - \ObsB^{\param'}_h(a_h, o_h)\bigr)\UU^{\param'}_{h}\PP^{\param', \pi}(\seq{o}{1}{h-1}, s_{h} = \cdot \given \seq{a}{h-\ell}{h-1})\bigr\|_1 \ud \seq{o}{1}{H-1}\notag\\
&\qquad \leq \int_{\cO^{H}} \sum_{\lat_{h-1}\in[d]} w_{h}\cdot \PP^{\param', \pi}(\seq{o}{1}{h-1}, {\lat_{h-1}} \given \seq{a}{h-\ell}{h-1}) \ud \seq{o}{1}{H-1}\notag\\
&\qquad = \int_{\cO} \sum_{\lat_{h-1}\in[d]} w_{h}\cdot \PP^{\param', \pi}({\lat_{h-1}} \given \seq{a}{h-\ell}{h-1}) \ud o_{h},
\#
where we write
\$
w_{h} = \bigl\|\bigl(\ObsB^\param_h(a^\pi_h, o_h) - \ObsB^{\param'}_h(a_h, o_h)\bigr)\UU^{\param'}_{h} \PP^{\param'}_{h-1}(s_h = \cdot\given {\lat_{h-1}})\bigr\|_1
\$
for notational simplicity. By plugging \eqref{eq::pf_lem_RHS_II_linear_eq5} into \eqref{eq::pf_lem_RHS_II_linear_eq4}, we complete the proof of Lemma \ref{lem::RHS_II_linear}.
\end{proof}

\section{Analysis for the Tabular POMDPs}
\label{sec::tabular_POMDP}
In the sequel, we present an analysis for the tabular POMDPs. We remark that our analysis extends the previous analysis of undercomplete POMDPs \citep{azizzadenesheli2016reinforcement, guo2016pac, jin2020sample}, where the emission matrices are left invertible. In particular, our analysis handles the overcomplete POMDPs with $O<S$, where $O$ and $S$ are the size of observation and state spaces $\cO$ amd $\cS$, respectively.

\subsection{Bellman Operator}
\label{sec::bellman_tabular}
We first introduce notations for matrices to simplify the discussions of POMDPs.
\smallsec{Notation.} We denote by $M = [f(i, j)]_{i, j} \in\RR^{n\times m}$ the $n$-by-$m$ matrix, where $f(i, j) \in\RR$ is the element in the $i$-th row and $j$-th column of $M$. In addition, for a matrix $M$, we denote by $M_{i,j}$ the $(i, j)$-th element of $M$.

In addition, recall that we denote by $\traj^{h+k}_h = \{o_h, a_h, \ldots, a_{h+k-1}, o_{h+k}\}$ the trajectory from the $h$-th observation $o_h$ to the $(h+k)$-th observation $o_{h+k}$. Similarly, we denote by $\utraj^k_h = (o_h,  a_h, \ldots, o_{h+k}, a_{h+k})$ the trajectory from the $h$-th observation $o_h$ to the $(h+k)$-th action $a_{h+k}$. We denote by $\seq{a}{h}{h+k-1} = (a_h, \ldots, a_{h+k-1})$ and $\seq{o}{h}{h+k} = (o_h, \ldots, o_{h+k})$ the action and observation sequences, respectively. Meanwhile, recall that we write
\$
\PP^\pi(\traj^{h+k}_h) &= \PP^\pi(o_h, \ldots, o_{h+k} \given a_{h}, \ldots, a_{h+k-1})  = \PP^\pi(\seq{o}{h}{h+k} \given  \seq{a}{h}{h+k-1}),\notag\\
\PP^\pi(\traj^{h+k}_h\given s_h) &= \PP^\pi(o_h, \ldots, o_{h+k} \given s_h, a_{h}, \ldots, a_{h+k-1})  = \PP^\pi(\seq{o}{h}{h+k} \given s_h,  \seq{a}{h}{h+k-1}).
\$
for notational simplicity.
\smallsec{Forward Emission Operator.} In the sequel, we define several matrices that describes the transition and emission in POMDPs. We define
\$
\diagO_h(o_h) &= \diag\bigl(\OO_h(o_h\given \cdot)\bigr) = \diag\Bigl(\bigl[\OO(o_h\given s_h)\bigr]_{s_h}\Bigr) \in \RR^{S\times S},\\
\Tran_h(a_h) &= \PP_h( \cdot\given  \cdot, a_h) = \bigl[\PP_h(s_{h+1}\given s_h, a_h)\bigr]_{s_{h} , s_{h+1}} \in \RR^{S\times S},\\
\ObsO_h &= \OO_h( \cdot \given \cdot) = \bigl[\OO_h(o_{h}\given s_h)\bigr]_{ o_{h}, s_{h}} \in \RR^{O\times S},
\$
where we denote by $\diag(v)\in\RR^{S\times S}$ the diagonal matrix where the diagonal entries aligns with the vector $v\in\RR^S$.

\begin{definition}[Forward Emission Operator]
For all $h\in[H]$ and $k>0$, we define the following forward emission operator,
\$
\UU_{h} &= \ObsO_{h+k}\Tran_{h+k-1}(\cdot)\diagO_{h+k-1}(\cdot)\cdots\Tran_{h}(\cdot)\diagO_{h}(\cdot) \\
&=\bigl[\ind(o_{h+k})^\top\ObsO_{h+k}\Tran_{h+k-1}(a_{h+k-1})\diagO_{h+k-1}(o_{h+k-1})\notag\\
&\qquad\cdots\Tran_{h}(a_{h})\diagO_{h}(o_{h}) \ind(s_h) \bigr]_{\traj^{h+k}_h, s_{h}}\in \RR^{(O^{k+1}\cdot A^{k})\times S},
\$
where we index the column of $\UU_{h}$ by the state $s_{h}\in\cS$, and the row of $\UU_{h}$ by the observation and action arrays
\$
\seq{a}{h}{h+k-1} = (a_{h}, \ldots, a_{h+k-1})\in\cA^{k}, \quad \seq{o}{h}{h+k} = (o_{h}, \ldots, o_{h+k})\in\cO^{k+1}.
\$
\end{definition}

\begin{lemma}[Forward Emission Operator]
\label{lem::FOO}
It holds for all $h\in[H]$ that
\$
\UU_{h} = \bigl[\PP(\traj^{h+k}_h \given s_h) \bigr]_{\traj^{h+k}_h, s_{h}} \in \RR^{(O^{k+1}\cdot A^{k})\times S}.
\$
\end{lemma}
\begin{proof}
See \S\ref{pf::lem_FOO} for a detailed proof.
\end{proof}

Lemma \ref{lem::FOO} characterizes the semantic meaning of the forward emission operator $\UU_{h+1}$. More specifically, Lemma \ref{lem::FOO} allows us to write $\UU_{h+1}$ in the following operator form,
\$
\UU_{h} = \PP(\underbrace{\seq{o}{h}{h+k} = \cdot}_{\displaystyle \text{\rm (a)}}  \given \underbrace{s_{h} = \cdot}_{\displaystyle\text{\rm (b)}}, \underbrace{\seq{a}{h}{h+k-1} = \cdot}_{\displaystyle\text{\rm (c)}} )\in \RR^{(O^{k+1}\cdot A^{k})\times S},
\$
where (a) and (c) correspond to the row indices, and (b) corresponds to the column indices of the forward emission matrix $\UU_{h}$. For a state distribution vector $\mu_{h} \in\RR^S$ of state $s_h$, it holds that 
\$
\UU_{h}\mu_{h} = \PP(\seq{o}{h}{h+k} = \cdot  \given \seq{a}{h}{h+k-1} = \cdot) \in \RR^{O^{k+1}\cdot A^{k}},
\$
which corresponds to the forward emission probability of $\seq{o}{h}{h+k} = (o_{h}, \ldots, o_{h+k})$ given an action sequence $\seq{a}{h}{h+k-1} = (a_{h}, \ldots, a_{h+k-1})$.

\begin{assumption}[Future Sufficiency]
\label{asu::inv_U}
We assume for some $k>0$ that the forward emission operator $\UU_{h}$ has full column rank for all $h\in[H]$. We denote by $\UU_{h}^{\dagger}$ the left inverse of $\UU_{h}$ for all $h\in[H]$. We assume further that $\|\UU_{h}^{\dagger}\|_{1\mapsto1} \leq \nu$ for all $h\in[H]$.
\end{assumption}

\smallsec{Planning with Bellman Operator.} In the sequel, we introduce the Bellman operators $\{\ObsB_h\}_{h\in[H]}$, which plays a central role in solving the overcomplete POMDP. We define the Bellman operators as follows.
\begin{definition}[Bellman Operator]
\label{def::obs_op}
For all $h\in[H]$, we define the Bellman operator $\ObsB_h$ as follows,
\$
\ObsB_h(a_h, o_h) = \UU_{h+1}\Tran_h(a_h)\diagO_h(o_h) \UU_{h}^{\dagger}, \quad \forall (a_h, o_h)\in\cA\times\cO.
\$
\end{definition}
Given the Bellman operators $\{\ObsB_h\}_{h\in[H]}$, we are able to estimate the probability of any given observation sequence $(o_1, a_1, \ldots, o_H)$. In particular, the following lemma holds.
\begin{lemma}
\label{lem::trans_est}
It holds for the Bellman operator defined in Definition \ref{def::obs_op} that
\$
\PP(\traj_1^{H-1}) = \eid{o_{H}}^\top \ObsB_{H-1}(a_{H-1}, o_{H-1}) \ldots \ObsB_1(a_1, o_1) \vecb_1, \quad \vecb_1 = \UU_{1} \mu_1.
\$
Here $\eid{o_{H}}$ is an indicator vector that takes value one at the indices $(\seq{o}{H}{H+k}, \seq{a}{H}{H+k-1})$ for any dummy observation sequence $\seq{o}{H+1}{H+k}$ and a randomly fixed action sequence $\seq{a}{H}{H+k-1}$. In addition, $\mu_1\in\RR^S$ is the probability array of initial state distribution, and $\vecb_1 \in\RR^{O^{k+1}\cdot A^k}$ is the probability distribution of the first $k$ steps, namely, 
\$
\vecb_1 = \UU_{1} \mu_1 = \bigl[\PP(\traj_1^k)\bigr]_{\traj^k_h} \in \RR^{O^{k+1}\cdot A^k}.
\$
\end{lemma}
\begin{proof}
See \S\ref{pf::lem_trans_est} for a detailed proof.
\end{proof}
Lemma \ref{lem::trans_est} allows us to estimate the probability of any given trajectory. In addition, for a deterministic policy $\pi$, it further holds that
\#\label{eq::lem_policy_traj}
\PP^\pi(\seq{o}{1}{H}) = \PP\bigl(\seq{o}{1}{H}\given \seq{(a^\pi)}{1}{H-1}\bigr),
\#
where $\seq{(a^\pi)}{1}{H-1} = (a^\pi_1, \ldots, a^\pi_{H-1})$ is the action sequence determined the observation sequence $\seq{o}{1}{H-1}$ and the deterministic policy $\pi$. Thus, for a given deterministic policy $\pi$, one can evaluate the policy $\pi$ based on the Bellman operators as follows,
\$
V^\pi &= \sum_{\seq{o}{1}{H-1} \in\cO^H}\PP^\pi(\seq{o}{1}{H})\cdot \sum_{h=1}^H r(o_h)\notag\\
&=\sum_{\seq{o}{1}{H-1} \in\cO^H}\sum_{h=1}^H  r(o_h) \cdot \eid{o_{H}}^\top \ObsB_{H-1}(a_{H-1}, o_{H-1}) \ldots \ObsB_1(a_1, o_1) \vecb_1.
\$

\smallsec{Estimating the Bellman Operator.}~To estimate the Bellman operators based on interactions, we utilize the following identity of Bellman operators,
\#\label{eq::id_OO}
\ObsB_h(a_h, o_h) \pone_h(\seq{o}{h}{h+k}) = \ptwo_h(\seq{a}{h-\ell}{h+k}, o_h).
\#
Here we define the probability tensors $\pone_h$ and $\ptwo_h$ as follows,
\$
\pone_h(\seq{a}{h-\ell}{h-1}) &= \UU_{h}\Tran_{h-1}(a_{h-1})\diagO_{h-1}(\cdot)\ldots \Tran_{h-\ell}(a_{h-\ell})\diagO_{h-\ell}(\cdot)\mu_{h-\ell},\notag\\
\ptwo_h(\seq{a}{h-\ell}{h}, o_h) &= \UU_{h+1}\Tran_h(a_h)\diagO_h(o_h)\Tran_{h-1}(a_{h-1})\diagO_{h-1}(\cdot)\ldots \Tran_{h-\ell}(a_{h-1})\diagO_{h-\ell}(\cdot)\mu_{h-\ell},
\$
where $\mu_{h-\ell} \in\RR^S$ is a probability density array for the state $s_{h-\ell}$. The following lemma characterizes the semantic meaning of the probability tensors $\pone_h$ and $\ptwo_h$ for all $h\in[H]$.
\begin{lemma}
\label{lem::P1_P2}
Let $\distP_{h-\ell}(s_{h-\ell} = \cdot) = \mu_{h-1}\in\RR^S$ be a probability density array for the state $s_{h-\ell}$. It holds for all $h\in[H]$ that
\$
\pone_h(\seq{a}{h-\ell}{h-1}) &= \bigl[\PP(\traj^{h+k}_{h-\ell}) \bigr]_{\traj^{h+k}_h, \seq{o}{h-\ell}{h-1}} \in\RR^{(A^{k+1}\cdot O^k)\times O^\ell},\\
\ptwo_h(\seq{a}{h-\ell}{h}, o_h) &= \bigl[\PP(\traj^{h+k+1}_{h-\ell})\bigr]_{\traj^{h+k+1}_{h+1}, \seq{o}{h-\ell}{h-1}} \in \RR^{(A^{k+1}\cdot O^k)\times O^\ell}.
\$
\end{lemma}
\begin{proof}
See \S\ref{pf::lem_p1_p2} for a detailed proof.
\end{proof}

\subsection{Algorithm} We now introduce \algosp under the tabular POMDPs. In particular, \algosp iteratively (i) collects data and fit the density of visitation trajectory, (ii) fits the Bellman operators and construct confidence sets, and (iii) conducts optimistic planning. See Algorithm \ref{alg::POMDP} for the summary. 

We remark that the data collection process is identical to that for the low-rank POMDPs. Meanwhile, in the tabular POMDPs, we estimate the density of visitation trajectory by count-based estimators as follows.
\#
\hat \vecb^t_1 &= \frac{1}{t}\cdot\sum_{\seq{a}{1}{k}\in\cA^k} \biggl(\sum_{\traj^{k+1}_1 \in\cD^t(\seq{a}{1}{k})} \eid{\traj^{k+1}_1}\biggr),\label{eq::tabular_hatb}\\
\hat \pone^t_h(\seq{a}{h-\ell}{h-1}) &=  \frac{1}{t}\cdot\sum_{\seq{a}{h}{h+k-1}\in\cA^k}\biggl(\sum_{\traj^{h+k}_{h-\ell}\in\cD^t(\seq{a}{h-\ell}{h+k-1})} \eid{\traj^{h+k}_h}\eid{\seq{o}{h-\ell}{h-1}}^\top\biggr),\label{eq::tabular_hatX}\\
\hat\ptwo^t_h(\seq{a}{h-\ell}{h}, o_h) &= \frac{1}{t}\cdot\sum_{\seq{a}{h+1}{h+k}\in\cA^k} \biggl(\sum_{\traj^{h+k+1}_{h-\ell}\in\cD^t(\seq{a}{h-\ell}{h+k})} \eid{\traj^{h+k+1}_{h+1}}\eid{\seq{o}{h-\ell}{h-1}}^\top\biggr),\label{eq::tabular_hatY}
\#
In the sequel, we summarize the estimations of initial trajectory density and Bellman operators in the $t$-th iterate by the parameter $\param^t$. Accordingly, we estimate the Bellman operator in the $t$-th iterate by minimizing the following objective,
\$
\hat L^t_h = \sup_{\seq{a}{h-\ell}{h}\in\cA^{\ell+1}}\| \ObsB^\param_h(a_h, o_h)\hat \pone^t_h(\seq{a}{h-\ell}{h-1}) - \ptwo^t_h(\seq{a}{h-\ell}{h}, o_h) \|_1
\$
We define the following confidence set of the parameter $\param$ in the $t$-th iteration.
\#\label{eq::def_CI}
{\cC^t} = \Bigl\{\param\in\Param: &\max_{h\in[H]}\bigl\{\|\vecb^\param_1 - \hat\vecb^t_1\|_1, 
\hat L^t_h\bigr\} \leq \beta_t\cdot \sqrt{1/t},\forall h\in[H]\Bigr\}.
\#
where we set
\$
\beta_t =  (1+\nu)\cdot (k+\ell)\cdot \sqrt{A^{5k+1}\cdot O^{k+\ell}\cdot\log(O\cdot A\cdot T\cdot H/\delta)/t}.
\$
Note that the initial density $\vecb^\param_1$ and Bellman operators $\{\ObsB^\param_h\}_{h\in[H]}$ are sufficient for policy evaluation since they recovers the visitation density of an arbitrary deterministic policy (Lemma \ref{lem::traj_density_lemma}). We conduct optimistic planning in the $t$-th iteration as follows,
\$
\pi^t = \argmax_{\pi\in\Pi, \param^t\in{\cC^t}} V^\pi(\param^t),
\$
where $V^\pi(\param^t)$ is the policy evaluation of $\pi$ with parameter $\param^t$ and $\Pi$ is the set of all deterministic policies.

\begin{algorithm}[htpb]
\caption{Embed to Control for Tabular POMDP}
\label{alg::POMDP}
\begin{algorithmic}[1]
\REQUIRE Number of iterates $T\in\NN$. A set of tuning parameters $\{\beta_t\}_{t\in[T]}$.
\STATE{\bf Initialization:} Set $\pi_0$ as a deterministic policy. Set the dataset $\cD^0_{h}(\seq{a}{h-\ell}{h+k})$ as an empty set for all $(h, \seq{a}{h-\ell}{h+k})\in[H]\times\cA^{k+\ell+1}$.
\FOR{$t\in[T]$}
\FOR{$(h, \seq{a}{h-\ell}{h+k})\in[H]\times\cA^{k+\ell+1}$}
\STATE \label{line::samp_st}Start a new episode from the $(1-\ell)$-th step.
\STATE Execute policy $\pi^{t-1}$ until the $(h-\ell)$-th step and receive the observations $\sp{t}{}{\seq{o}{1-\ell}{h-\ell}}$.
\STATE \label{line::samp_ed} Execute the action sequence $\seq{a}{h-\ell}{h+k}$ regardless of the observations and receive the observations $\sp{t}{}{\seq{o}{h-\ell+1}{h+k+1}}$.
\STATE Update the dataset $
\cD^t_h(\seq{a}{h-\ell}{h+k}) \leftarrow \cD^{t-1}_h(\seq{a}{h-\ell}{h+k}) \cup \bigl\{\sp{t}{}{\seq{o}{h-\ell+1}{h+k+1}}\bigr\}.
$
\ENDFOR
\STATE Update the density mappings $\hat\pone^t_h$ and $\hat\ptwo^t_h$ by \eqref{eq::tabular_hatX} and \eqref{eq::tabular_hatY}, respectively.
\STATE Update the initial density estimation $\hat\vecb^t_1(\traj^H_1) \leftarrow$ by \eqref{eq::tabular_hatb}.
\STATE Update the confidence set $\cC^t$ by \eqref{eq::def_CI}.
\STATE Update the policy $
\pi^t \leftarrow \argmax_{\pi\in\Pi}\max_{\param\in\cC^t} V^\pi(\param)
$. 
\ENDFOR
\STATE {\bf Output:} policy set $\{\pi^t\}_{t\in[T]}$.
\end{algorithmic}
\end{algorithm}

\subsection{Theory}
In the sequel, we present the sample efficiency analysis of \algosp for the tabular POMDPs.

\smallsec{Calculating the Performance Difference.} Similar to the analysis under the low-rank POMDPs, we first calculate the performance difference of a policy between two different POMDPs defined by the parameter $\param$ and $\param'$, respectively. The following lemma is adopted from \cite{jin2020sample}.
\begin{lemma}[Trajectory Density \citep{jin2020sample}]
\label{lem::traj_density_lemma}
It holds that 
\$
\PP^{\param, \pi}(\seq{o}{1}{H-1}) = \PP^\param\bigl(\seq{o}{1}{H-1}\given \seq{(a^\pi)}{1}{H}\bigr),
\$
where $\seq{(a^\pi)}{1}{H} = (a^\pi_1, \ldots, a^\pi_H)$ and $a^\pi_h = \pi(\seq{a}{1}{h-1},\seq{o}{1}{h})$ is the action taken by $\pi$ in the $h$-th step for all $h\in[H]$.
\end{lemma}
\begin{proof}
See \cite{jin2020provably} for a detailed proof.
\end{proof}
We now calculate the performance difference in the following lemma.
\begin{lemma}[Performance Difference]
\label{lem::perform_diff}
It holds for any policy $\pi$ that
\$
&|V^\pi(\param) - V^\pi(\param')|\notag\\
&\qquad\leq\nu\cdot\sqrt{S}\cdot H\cdot \sum_{h=2}^{H-1}\sum_{\seq{o}{1}{h}\in\cO^h} \bigl\|\bigl(\ObsB^\param_{h}(a^\pi_{h}, o_{h}) - \ObsB^{\param'}_{h}(a^\pi_{h}, o_{h})\bigr)\ObsB^{\param'}_{h-1}(a^\pi_{h-1}, o_{h-1})  \ldots  \ObsB^{\param'}_1(a^\pi_1, o_1)\vecb^{\param'}_1\bigr\|_1\notag\\
&\qquad\qquad + \nu\cdot\sqrt{S}\cdot H\cdot\sum_{a_1\in\cA}\sum_{o_1\in\cO}\bigl\|\bigl(\ObsB^\param_{1}(a^\pi_{1}, o_{1}) - \ObsB^{\param'}_{1}(a^\pi_{1}, o_{1})\bigr)\vecb^{\param'}_1\bigr\|_1 +\nu\cdot\sqrt{S}\cdot H\cdot\|\vecb^\param_1 - \vecb^{\param'}_1\|_1,
\$
where $\ObsB^\param_h$ is the Bellman operator corresponding to the parameter $\param$ for all $h\in[H]$, and $\vecb^\param_1 = \UU^\param_{1,k} \mu_1$ is the initial trajectory distribution corresponding to the parameter $\param$. Here the action $a^\pi_h = \pi(\seq{(a^\pi)}{1}{H-1}, \seq{o}{1}{H-1})$ is the action taken by $\pi$ in the $h$-th step for all $h\in[H]$.
\end{lemma}
\begin{proof}
See \S\ref{sec::pf_lem_perform_diff} for a detailed proof.
\end{proof}

We define the following state density array,
\#\label{eq::def_state_arr_pf_mu}
\mu^{\param}_{h-1}(\seq{a}{h-\ell}{h-1},\seq{o}{1}{h-1}; \pi) &= \underbrace{\diagO^{\param}_{h-1}(o_{h-1})\Tran^{\param}_{h-2}(a_{h-2})\cdots \Tran^{\param}_{h-\ell}(a_{h-\ell}) \diagO^{\param}_{h-\ell}(o_{h-\ell})}_{\displaystyle\rm (i)}\notag\\
&\qquad\underbrace{\cdot\Tran^{\param}_{h-\ell-1}(a^\pi_{h-\ell-1})\cdots \Tran^{\param}_{1}(a^\pi_{1}) \diagO^{\param}_{1}(o_{1})\mu_1}_{\displaystyle\rm (ii)}\notag\\
&=\bigl[\PP^{\param, \pi}(s_h, \seq{o}{1}{h-1} \given \seq{a}{h-\ell}{h-1})\bigr]_{s_h\in\cS} \in\RR^S.
\#
Here the actions $a^\pi_{h-\ell-1}, \ldots, a^\pi_{1}$ in (ii) of \eqref{eq::def_state_arr_pf_mu} is determined by the observations array $\seq{o}{1}{h-\ell-2}$ and the policy $\pi$. Meanwhile, the action array $\seq{a}{h-\ell}{h-1}$ is the fixed action array that defines the state density array $\mu^{\param'}_{h-1}(\pi, \seq{a}{h-\ell}{h-1},\seq{o}{1}{h-1})$. In addition, we denote by $\PP^{\param}$ the probability that corresponds to the transition dynamics defined by the operators $\{\diagO^\param_h, \Tran^\param_h\}_{h\in[H]}$. Based on \eqref{eq::def_state_arr_pf_mu}, we further define the following marginal state density array,
\#\label{eq::def_state_arr_pf_marginal}
\tilde\mu^{\param}_{h-1}( \seq{a}{h-\ell}{h-1};\pi) &= \sum_{\seq{o}{1}{h-1}\in\cO^{h-1}}\mu^{\param}_{h-1}(\pi, \seq{a}{h-\ell}{h-1},\seq{o}{1}{h-1}) \notag\\
&= \bigl[\PP^{\param, \pi}(s_{h-1} \given \seq{a}{h-\ell}{h-1})\bigr]_{s_h\in\cS} \in\RR^S.
\#
The marginal state density array $\tilde\mu^{\param}_{h-1}(\seq{a}{h-\ell}{h-1};\pi)$ captures the state distribution of $s_{h-1}$ given the following interaction protocol: (i) starting with the initial observation, interacting with the environment based on policy $\pi$ until the $(h-\ell)$-th step and observing $o_{h-\ell}$, and (ii) interacting with the environment with a fixed action sequence $\seq{a}{h-\ell}{h-1}$ regardless of the observations until the $(h-1)$-th step and observing $o_{h-1}$. We remark that such interaction protocol is identical to the sampling process in Line \ref{line::samp_st}--\ref{line::samp_ed} of Algorithm \ref{alg::POMDP}. The following lemma upper bounds the performance difference calculated in Lemma \ref{lem::perform_diff}.


\begin{lemma}[Upper Bound of Performance Difference]
\label{lem::upper_bound_rhs}
It holds for all $\pi$ and  $h>1$ that
\$
&\sum_{\seq{o}{1}{h}\in\cO^h} \bigl\|\bigl(\ObsB^\param_{h}(a^\pi_{h}, o_{h}) - \ObsB^{\param'}_{h}(a^\pi_{h}, o_{h})\bigr)\ObsB^{\param'}_{h-1}(a^\pi_{h-1}, o_{h-1})  \ldots  \ObsB^{\param'}_1(a^\pi_1, o_1)\vecb^{\param'}_1\bigr\|_1 \notag\\
&\qquad \leq \sum_{\seq{a}{h-\ell}{h}\in\cA^{\ell+1}}\sum_{o_h\in\cO}\sum_{s_{h-1}\in\cS} \bigl\|\bigl(\ObsB^\param_{h}(a_{h}, o_{h}) - \ObsB^{\param'}_{h}(a_{h}, o_{h})\bigr)\UU^{\param'}_{h,k} \Tran^{\param'}_{h-1}(s_{h-1}, a_{h-1})\bigr\|_1\cdot \PP^\pi(s_{h-1} \given \seq{a}{h-\ell}{h-1}),\\
\$
where the action $a^\pi_h = \pi(\seq{(a^\pi)}{1}{h-1}, \seq{o}{1}{h})$ is the action taken by $\pi$ in the $h$-th step for all $h\in[H]$.Here $\Tran^{\param'}_{h-1}(s_{h-1}, a_{h-1}) \in\RR^S$ is the state distribution array $[\Tran^{\param'}_{h-1}(s_{h}\given s_{h-1}, a_{h-1})]_{s_h} \in\RR^S$ for all $h>1$.
\end{lemma}
\begin{proof}
See \S\ref{sec::pf_lem_upper} for a detailed proof.
\end{proof}

\smallsec{Confidence Set Analysis.} We now analyze the confidence set utilized for optimistic planning. We define the following visitation measure of mix policy in the $t$-th iteration for all $t>0$,
\$
\PP^t(s_h) = \frac{1}{t}\cdot\sum^{t-1}_{\omega = 0}\PP^{\pi^\omega}(s_h),
\$
where $\{\pi^\omega\}_{\omega \in [t]}$ is the set of policy returned by Algorithm \ref{alg::POMDP}. Meanwhile, recall that we define the empirical density estimators,
\$
\hat \vecb^t_1 &= \frac{1}{t}\cdot\sum_{\seq{a}{1}{k}\in\cA^k} \biggl(\sum_{\traj^{k+1}_1 \in\cD^t(\seq{a}{1}{k})} \eid{\traj^{k+1}_1}\biggr),\notag\\
\hat \pone^t_h(\seq{a}{h-\ell}{h-1}) &=  \frac{1}{t}\cdot\sum_{\seq{a}{h}{h+k-1}\in\cA^k}\biggl(\sum_{\traj^{h+k}_{h-\ell}\in\cD^t(\seq{a}{h-\ell}{h+k-1})} \eid{\traj^{h+k}_h}\eid{\seq{o}{h-\ell}{h-1}}^\top\biggr),\notag\\
\hat\ptwo^t_h(\seq{a}{h-\ell}{h}, o_h) &= \frac{1}{t}\cdot\sum_{\seq{a}{h+1}{h+k}\in\cA^k} \biggl(\sum_{\traj^{h+k+1}_{h-\ell}\in\cD^t(\seq{a}{h-\ell}{h+k})} \eid{\traj^{h+k+1}_{h+1}}\eid{\seq{o}{h-\ell}{h-1}}^\top\biggr),
\$
where we denote by $\eid{x}$ the indicator vector that takes value one at the index $x$. Recall that we define the confidence set as follows,
\$
{\cC^t} = \Bigl\{\param\in\Param: &\max\bigl\{\|\vecb^\param_1 - \hat\vecb^t_1\|_1, 
\hat L^t_h\bigr\} \leq \beta_t\cdot \sqrt{1/t},\forall h\in[H]\Bigr\}.
\$
where we set
\$
\beta_t =  A^k\cdot (k+\ell)\cdot\sqrt{\log(O\cdot A\cdot T\cdot H/\delta)}.
\$
Recall that in the $t$-th iteration of Algorithm \ref{alg::POMDP}, we update the policy $\pi^t$ as follows,
\$
\pi^t = \argmax_{\pi\in\Pi, \param\in{\cC^t}} V^\pi(\param).
\$
The following lemma shows that the empirical estimations aligns closely to the true density corresponding to the exploration.

\begin{lemma}[Concentration Bound of Density Estimation]
\label{lem::concen}
It holds for all $t\in[T]$ with probability at least $1 - \delta$ that
\$
&\max\bigl\{\|\pone^t_h(\seq{a}{h-\ell}{h-1}) - \hat \pone^t_h(\seq{a}{h-\ell}{h-1})\|_F, \|\vecb_1 - \hat \vecb^t_1\|_2, \|\ptwo^t_h(\seq{a}{h-\ell}{h}, o_h) -  \hat\ptwo^t_h(\seq{a}{h-\ell}{h}, o_h)\|_F \bigr\} \notag\\
&\qquad = \cO\bigl(A^k\cdot (k+\ell)\cdot \sqrt{\log(O\cdot A\cdot T\cdot H/\delta)/t}\bigr)
\$
with probability at least $1 - \delta$.
\end{lemma}
\begin{proof}
See \S\ref{sec::pf_lem_concen} for a detailed proof.
\end{proof}

In what follows, we define the reverse emission operators for the tabular POMDPs.
\begin{definition}[Revserse Emission]
\label{def::rev_op}
For all $1<h\leq H$ and $\seq{a}{h-\ell}{h-2} \in\cA^{\ell-1}$, we define
\$
A^\pi_{h-1, \ell}(\seq{a}{h-\ell}{h-2}) &= \diagO_{h-1}(o_{h-1} = \cdot)\Tran_{h-2}(a_{h-2})\ldots \cdot\diagO_{h-\ell}(o_{h-\ell} = \cdot)\diag\bigl(\PP^\pi(s_{h-\ell} = \cdot)\bigr)\notag\\
&=\diagO_{h-1}(o_{h-1} = \cdot) \bigl( \Pi^{h-\ell}_{i = h-2}\Tran_{i}(a_{i})\diagO_{i}(o_{i} = \cdot) \bigr)\diagO_{h-\ell}(o_{h-\ell} = \cdot)\diag\bigl(\PP^\pi(s_{h-\ell} = \cdot)\bigr)\notag\\
&\in\RR^{S\times O^\ell}.
\$
\end{definition}
By Definition \ref{def::rev_op} and the identity in Lemma \ref{lem::P1_P2}, we have the following identity,
\#\label{eq::rev_op_id}
\pone^t_h(\seq{a}{h-\ell}{h-1}) = \UU_{h}\Tran_{h-1}(a_{h-1})A^{\overline \pi^t}_{h-1, \ell}(\seq{a}{h-\ell}{h-2}), \qquad \forall \seq{a}{h-\ell}{h-1} \in\cA^\ell.
\#
Here we denote by $\overline \pi^t$ the mixed policy induced by the policies $\{\pi^\omega\}_{\omega\in[t]}$ obtained until the $t$-th iteration of Algorithm \ref{alg::POMDP}.

\begin{assumption}[Past Sufficiency]
\label{asu::inv_rev_op}
We define the following matrix for all policy $\pi$, $\seq{a}{h-\ell}{h-2}\in\cA^{\ell-1}$, and $0<h\leq H$,
\$
C^\pi_{h-1, \ell}(\seq{a}{h-\ell}{h-1}) = \diag(\PP^\pi(s_{h-1}\given \seq{a}{h-\ell}{h-1}))^{-1} A^\pi_{h-1, \ell} \in \RR^{S\times (A^{\ell-1}\cdot O^\ell)}.
\$
Here recall that $\diag(v)$ is the diagonal matrix where the diagonal entries align with the vector $v$. We assume that $C^\pi_{h-1, \ell}$ has full row rank for all $\pi$, $\seq{a}{h-\ell}{h-1}\in\cA^{\ell-1}$, and $0<h\leq H$. We denote by $C^{\pi, \dagger}_{h-1, \ell}(\seq{a}{h-\ell}{h-1})$ the right inverse of $C^\pi_{h-1, \ell}(\seq{a}{h-\ell}{h-1})$. We assume further that 
\$
\|C^{\pi, \dagger}_{h-1, \ell}(\seq{a}{h-\ell}{h-1})^\top\|_{1\mapsto1} \leq \gamma
\$ 
for an absolute constant $\gamma >0$ for all $\pi$, $\seq{a}{h-\ell}{h-1}\in\cA^{\ell-1}$, and $0<h\leq H$.
\end{assumption}

\begin{lemma}[Good Event Probability]
\label{lem::good_event}
Under Assumptions \ref{asu::inv_U} and \ref{asu::inv_rev_op}, it holds with probability at least $1 - \delta$ that $\param^* \in \cC_t$. Moreover, it holds for all $t\in[T]$ with probability at least $1-\delta$ that
\#
&\|\vecb_1 - \hat \vecb^{\param^t}_1\|_1 = \cO\Bigl(\nu\cdot (k+\ell)\cdot \sqrt{A^{5k+1}\cdot O^{k+\ell}\cdot\log(O\cdot A\cdot T\cdot H/\delta)/t}\Bigr),\label{eq::lem_ge_eq1}\\
&\bigl\|\bigl(\ObsB^{\param^t}_{1}(a_{1}, o_{1}) - \ObsB^{\param^*}_{1}(a_{1}, o_{1})\bigr) \vecb_1\bigr\|_1 = \cO\Bigl(\nu\cdot (k+\ell)\cdot \sqrt{A^{5k+1}\cdot O^{k+\ell}\cdot\log(O\cdot A\cdot T\cdot H/\delta)/t}\Bigr)\label{eq::lem_ge_eq2},
\#
Meanwhile, it holds for all $1<h\leq H$ and $t\in[T]$ with probability at least $1-\delta$ that
\#
&\sum_{s_{h-1}\in\cS}\bigl\|\bigl(\ObsB^{\param^t}_{h}(a_{h}, o_{h}) - \ObsB^{\param^*}_{h}(a_{h}, o_{h})\bigr)\UU^{\param^*}_{h,k} \Tran^{\param^*}_{h-1}(s_{h-1}, a_{h-1})\bigr\|_1\cdot\PP^t(s_{h-1} \given \seq{a}{h-\ell}{h-1})\notag\\
&\qquad =  \cO\Bigl(\gamma\cdot \nu\cdot (k+\ell)\cdot \sqrt{ A^{5k+\ell}\cdot O^{k+1}\cdot \log(O\cdot A\cdot T\cdot H/\delta)/t}\Bigr)\label{eq::lem_ge_eq3}
\#
for all $1<h\leq H$. Here $\{\ObsB^{\param^t}_h\}_{h\in[H]}$ and $\vecb^t_1$ are the updated Bellman operators and initial trajectory distribution, respectively, in the $t$-th iterate of Algorithm \ref{alg::POMDP}.
\end{lemma}
\begin{proof}
See \S\ref{sec::pf_lem_good_event} for a detailed proof.
\end{proof}

\smallsec{Sample Complexity Analysis.} We are now ready to present the sample efficiency analysis of \algosp under the tabular POMDPs.
\begin{theorem}
\label{thm::sample_complexity}
Let
\$
T = \cO\Bigl(\text{poly}(S, A, O, H) \cdot\gamma^2\cdot \nu^4\cdot (k+\ell)^2\cdot  \log(O\cdot A\cdot  H/\delta)/\epsilon^2\Bigr).
\$
Under Assumptions \ref{asu::inv_U} and \ref{asu::inv_rev_op}, it holds with probability at least $1-\delta$ that $\overline\pi^T$ is $\epsilon$-optimal. Here $\text{poly}(S, A, O, H)$ is a polynomial that takes the following order,
\$
\text{poly}(S, A, O, H) = \cO(S^2\cdot A^{10k+2\ell}\cdot O^{2k+2\ell}\cdot H^2).
\$
\end{theorem}
\begin{proof}
It holds that
\#\label{eq::pf_thm_eq0}
V^{*}(\param^*) - V^{\overline\pi^t}(\param^*) = \frac{1}{T}\cdot\sum^T_{t = 1} V^*(\param^*) - V^{\pi^t}(\param^*).
\#
It suffices to upper bound the performance difference
\$
V^*(\param^*) - V^{\pi^t}(\param^*)
\$
for all $t\in[T]$. By Lemma \ref{lem::good_event}, it holds with probability at least $1-\delta$ that $\param^*\in{\cC^t}$ for all $t\in[T]$. Thus, by the update of $\pi^t$ in Algorithm \ref{alg::POMDP}, it holds with probability at least $1-\delta$ that
\#\label{eq::pf_thm_eq1}
V^*(\param^*) - V^{\pi^t}(\param^*) \leq V^{\pi^t}(\param^t) - V^{\pi^t}(\param^*).
\#
It now suffices to upper bound the performance difference on the right-hand side of \eqref{eq::pf_thm_eq1}. By Lemma \ref{lem::perform_diff}, we obtain
\#\label{eq::pf_thm_eq2}
&|V^{\pi^t}(\param^t) - V^\pi(\param^*)|\\
&\quad\leq\nu\cdot\sqrt{S}\cdot H\cdot \underbrace{\sum_{h=2}^{H-1}\sum_{\seq{o}{1}{H-1}\in\cO^h} \bigl\|\bigl(\ObsB^{\param^t}_{h}(a_{h}, o_{h}) - \ObsB^{\param^*}_{h}(a_{h}, o_{h})\bigr)\ObsB^{\param^*}_{h-1}(a_{h-1}, o_{h-1})  \ldots  \ObsB^{\param^*}_1(a_1, o_1)\vecb^{\param^*}_1\bigr\|_1}_{\displaystyle\rm (i)}\notag\\
&\quad\quad + \underbrace{\nu\cdot\sqrt{S}\cdot H\cdot\sum_{a_1\in\cA}\sum_{o_1\in\cO}\bigl\|\bigl(\ObsB^{\param^t}_{1}(a_{1}, o_{1}) - \ObsB^{\param^*}_{1}(a_{1}, o_{1})\bigr)\vecb^{\param^*}_1\bigr\|_1}_{\displaystyle \rm (ii)} +\underbrace{\nu\cdot\sqrt{S}\cdot H\cdot\|\vecb^{\param^t}_1 - \vecb^{\param^*}_1\|_1}_{\displaystyle \rm (iii)}.\notag
\#
In the sequel, we upper bound terms (i), (ii), and (iii) on the right-hand side of \eqref{eq::pf_thm_eq2}. By Lemma \ref{lem::good_event}, it holds for all $t\in[T]$ with probability at least $1-\delta$ that
\#\label{eq::pf_thm_ii_iii}
{\rm (ii)} &= \bigl\|\bigl(\ObsB^{\param^t}_{1}(a_{1}, o_{1}) - \ObsB^{\param^*}_{1}(a_{1}, o_{1})\bigr) \vecb_1\bigr\|_1 = \cO\Bigl(\nu\cdot (k+\ell)\cdot \sqrt{A^{5k+1}\cdot O^{k+\ell}\cdot\log(O\cdot A\cdot T\cdot H/\delta)/t}\Bigr),\notag\\
{\rm (iii)} &= \|\vecb_1 - \hat \vecb^{\param^t}_1\|_1 = \cO\Bigl(\nu\cdot (k+\ell)\cdot \sqrt{A^{5k+1}\cdot O^{k+\ell}\cdot\log(O\cdot A\cdot T\cdot H/\delta)/t}\Bigr).
\#
It remains to upper bound term (i) on the right-hand side of \eqref{eq::pf_thm_eq2}. By Lemma \ref{lem::upper_bound_rhs}, we obtain that
\#\label{eq::pf_thm_i}
&{\rm (i)} = \sum_{\seq{o}{1}{H-1}\in\cO^h} \bigl\|\bigl(\ObsB^{\param^t}_{h}(a_{h}, o_{h}) - \ObsB^{\param^*}_{h}(a_{h}, o_{h})\bigr)\ObsB^{\param^*}_{h-1}(a_{h-1}, o_{h-1})  \ldots  \ObsB^{\param^*}_1(a_1, o_1)\vecb^{\param'}_1\bigr\|_1 \\
&\qquad \leq \sum_{\seq{a}{h-\ell}{h}\in\cA^{\ell+1}}\sum_{o_h\in\cO}\sum_{s_{h-1}\in\cS} \bigl\|\bigl(\ObsB^{\param^t}_{h}(a_{h}, o_{h}) - \ObsB^{\param^*}_{h}(a_{h}, o_{h})\bigr)\UU^{\param^*}_{h,k} \Tran^{\param^*}_{h-1}(s_{h-1}, a_{h-1})\bigr\|_1\cdot \PP^{\pi^t}(s_{h-1} \given \seq{a}{h-\ell}{h-1})\notag
\#
Meanwhile, by Lemma \ref{lem::good_event}, it holds for all $h\in[T]$ and $t\in[T]$ with probability at least $1-\delta$ that
\#\label{eq::pf_thm_eq4}
&\sum_{s_{h-1}\in\cS}\bigl\|\bigl(\ObsB^{\param^t}_{h}(a_{h}, o_{h}) - \ObsB^{\param^*}_{h}(a_{h}, o_{h})\bigr)\UU^{\param^*}_{h,k} \Tran^{\param^*}_{h-1}(s_{h-1}, a_{h-1})\bigr\|_1\cdot\PP^t(s_{h-1} \given \seq{a}{h-\ell}{h-1})\\
&\qquad =  \cO\Bigl(\gamma\cdot \nu\cdot (k+\ell)\cdot \sqrt{ A^{5k+\ell}\cdot O^{k+1}\cdot \log(O\cdot A\cdot T\cdot H/\delta)/t}\Bigr),\notag
\#
where we define
\#\label{eq::pf_thm_eq5}
\PP^t = \frac{1}{t}\cdot \sum^{t-1}_{\omega = 0}\PP^{\pi^\omega}
\#
for $t>1$. We remark that the upper bound in \eqref{eq::pf_thm_eq4} does not match the right-hand side of \eqref{eq::pf_thm_eq2}. The only difference is the probaility density of $s_{h-1}$, which is $\PP^{\pi^t}$ on the right-hand side of \eqref{eq::pf_thm_eq2} but $\PP^{t}$ defined in \eqref{eq::pf_thm_eq5} on the left-hand side of \eqref{eq::pf_thm_eq4}, respectively. To this end, we utilize the same calculation trick as \S\ref{sec::pf_thm_sample_complexity_linear} and adopt Lemma \ref{lem::sum_trick}. By the upper bound in \eqref{eq::pf_thm_eq4} and Lemma \ref{lem::sum_trick} with
\$
z_t &= \bigl\|\bigl(\ObsB^{\param^t}_{h}(a_{h}, o_{h}) - \ObsB^{\param^*}_{h}(a_{h}, o_{h})\bigr)\UU^{\param^*}_{h,k} \Tran^{\param^*}_{h-1}(s_{h-1}, a_{h-1})\bigr\|_1, \notag\\
w_t &=\PP^{\pi^t}(s_{h-1} \given \seq{a}{h-\ell}{h-1}),
\$
we obtain for all $s_{h-1}\in\cS$ that
\#\label{eq::pf_thm_eq6}
&\frac{1}{T}\cdot \sum^T_{t = 1}\bigl\|\bigl(\ObsB^{\param^t}_{h}(a_{h}, o_{h}) - \ObsB^{\param^*}_{h}(a_{h}, o_{h})\bigr)\UU^{\param^*}_{h,k} \Tran^{\param^*}_{h-1}(s_{h-1}, a_{h-1})\bigr\|_1\cdot \PP^{\pi^t}(s_{h-1} \given \seq{a}{h-\ell}{h-1})\notag\\
&\qquad = \cO\Bigl(\gamma\cdot \nu^2\cdot (k+\ell)\cdot \sqrt{ A^{5k+\ell}\cdot O^{k+1}\cdot \log T\cdot \log(O\cdot A\cdot T\cdot H/\delta)/T}\Bigr).
\#
Thus,combining \eqref{eq::pf_thm_eq1}, \eqref{eq::pf_thm_ii_iii}, \eqref{eq::pf_thm_i}, and \eqref{eq::pf_thm_eq6}, we obtain
\$
&\frac{1}{T} \cdot\sum^T_{t = 1}|V^{\pi^t}(\param^t) - V^\pi(\param^*)| \notag\\
&\qquad= \cO\Bigl(\text{poly}(S, A, O, H) \cdot\gamma\cdot \nu^2\cdot (k+\ell)\cdot \sqrt{\log T\cdot \log(O\cdot A\cdot T\cdot H/\delta)/T}\Bigr),
\$
with probability at least $1 - \delta$, where we define
\$
\text{poly}(S, A, O, H) = H\cdot \sqrt{S^3\cdot A^{5k+3\ell}\cdot O^{k+3}}.
\$
By \eqref{eq::pf_thm_eq0}, it further holds with probability at least $1-\delta$ that
\$
V^*(\param^*) - V^{\overline\pi^T}(\param^*) = \cO\Bigl(\text{poly}(S, A, O, H) \cdot\gamma\cdot \nu^2\cdot (k+\ell)\cdot \sqrt{\log T\cdot \log(O\cdot A\cdot T\cdot H/\delta)/T}\Bigr).
\$
Hence, by setting 
\$
T = \cO\Bigl(\text{poly}(S, A, O, H) \cdot\gamma^2\cdot \nu^4\cdot (k+\ell)^2\cdot  \log(O\cdot A\cdot  H/\delta)/\epsilon^2\Bigr),
\$
it holds with probability at least $1-\delta$ that $V^*(\param^*) -  V^{\overline\pi^T}(\param^*) \leq \epsilon$, which completes the proof of Theorem \ref{thm::sample_complexity}.
\end{proof}

\section{Proof of Tabular POMDP}
In this section, we present the proof of the auxiliary results in \S\ref{sec::tabular_POMDP}.

\subsection{Proof of Lemma \ref{lem::FOO}}
\label{pf::lem_FOO}
\begin{proof}
Recall that we define for all $h\in[H]$ the following operators,
\$
\diagO_h(o_h) &= \diag\bigl(\OO_h(o_h\given \cdot)\bigr) = \diag\Bigl(\bigl[\OO(o_h\given s_h)\bigr]_{s_h}\Bigr) \in \RR^{S\times S},\\
\Tran_h(a_h) &= \PP_h (\cdot\given \cdot, a_h) = \bigl[\PP_h(s_{h+1}\given s_h, a_h)\bigr]_{s_{h} , s_{h+1}} \in \RR^{S\times S},\\
\ObsO_h &= \OO_h(\cdot \given \cdot) = \bigl[\OO_h(o_{h}\given s_h)\bigr]_{ o_{h}, s_{h}} \in \RR^{O\times S},
\$
where we denote by $\diag(v)\in\RR^{S\times S}$ the diagonal matrix where the diagonal entries aligns with the vector $v\in\RR^S$. Thus, it holds that 
\$
\diagO_{h}(o_{h+1}) \eid{s_{h}} = \OO_h(o_{h}\given s_{h})\cdot \eid{s_{h}} \in\RR^S.
\$
By further calculation, we have
\$
\Tran_{h}(a_{h})\diagO_{h}(o_{h}) \eid{s_{h}} &= \bigl[\PP_h(s_{h+1}\given a_{h}, s_{h})\cdot \OO_h(o_{h}\given s_{h})\bigr]_{s_{h+1}}\notag\\
&=\bigl[\PP(s_{h+1}, o_{h}\given a_{h}, s_{h})\bigr]_{s_{h+1}} \in \RR^S,
\$
where the second equality holds since we have $O_{h} \indp S_{h+1} \given s_{h}, a_{h}$. It then holds that
\#\label{eq::lem_FOO_eq1}
\diagO_{h+1}(o_{h+1})\Tran_{h}(a_{h})\diagO_{h}(o_{h}) \eid{s_{h}} &= \bigl[ \OO_{h+1}(o_{h+1}\given s_{h+1})\cdot\PP(s_{h+1}, o_{h}\given a_{h}, s_{h})\bigr]_{s_{h+1}}\notag\\
&= \bigl[ \PP(s_{h+1}, o_{h+1}, o_{h}\given a_{h}, s_{h}) \bigr]_{s_{h+1}} \in \RR^S,
\#
where the second equality holds since the observation $O_{h+1}$ is independent of all the other random variables in \eqref{eq::lem_FOO_eq1} given the state $S_{h+1} = s_{h+1}$. By further multiplying the right-hand side of \eqref{eq::lem_FOO_eq1} by $\Tran_{h+1}(a_{h+1})$, we obtain that
\#\label{eq::lem_FOO_eq2}
&\Tran_{h+1}(a_{h+1})\diagO_{h+1}(o_{h+1})\Tran_{h}(a_{h})\diagO_{h}(o_{h}) \eid{s_{h}}\notag\\
&\qquad = \biggl[\sum_{s_{h+1}\in\cS}\PP_{h+1}(s_{h+2}\given s_{h+1}, a_{h+1})\cdot \PP(s_{h+1}, o_{h+1}, o_{h}\given a_{h}, s_{h}) \biggr]_{s_{h+2}}\notag\\
&\qquad=\biggl[\sum_{s_{h+1}\in\cS}\PP(s_{h+2}, s_{h+1}, o_{h+1}, o_{h}\given a_{h+1}, a_{h}, s_{h}) \biggr]_{s_{h+2}}\notag\\
&\qquad=\bigl[\PP(s_{h+2}, o_{h+1}, o_{h}\given a_{h+1}, a_{h}, s_{h})\bigr]_{s_{h+2}} \in \RR^S,
\#
where the second equality holds since the state $S_{h+2}$ is independent of all the other random variables in \eqref{eq::lem_FOO_eq2} given the previous state $S_{h+1} = s_{h+1}$ and action $A_{h+1} = a_{h+1}$. By an iterative calculation similar to \eqref{eq::lem_FOO_eq1} and \eqref{eq::lem_FOO_eq2}, we obtain that
\$
&\Tran_{h+k-1}(a_{h+k-1})\diagO_{h+k-1}(o_{h+k-1})\cdots\Tran_{h}(a_{h})\diagO_{h}(o_{h}) \eid{s_{h}}\notag\\
&\qquad= \bigl[\PP(s_{h+k}, o_{h+k-1}, \ldots, o_{h} \given a_{h+k-1}, \ldots, a_{h}, s_{h})\bigr]_{s_{h+k}}\in\RR^S.
\$
By further calculation, we obtain that
\#\label{eq::lem_FOO_eq3}
&\ObsO_{h+k}\Tran_{h+k-1}(a_{h+k-1})\diagO_{h+k-1}(o_{h+k-1})\cdots\Tran_{h}(a_{h})\diagO_{h}(o_{h}) \eid{s_{h}}\notag\\
&\qquad= \bigl[\PP(o_{h+k}, o_{h+k-1}, \ldots, o_{h} \given a_{h+k-1}, \ldots, a_{h}, s_{h})\bigr]_{o_{h+k}}\in\RR^O.
\#
Finally, by multiplying the right-hand side of \eqref{eq::lem_FOO_eq3} with the indicator vector $\eid{o_{h+k}}$, we conclude that
\$
\UU_{h}&=\bigl[\eid{o_{h+k}}^\top\ObsO_{h+k}\Tran_{h+k-1}(a_{h+k-1})\diagO_{h+k-1}(o_{h+k-1})\notag\\
&\qquad\cdots\Tran_{h}(a_{h})\diagO_{h}(o_{h}) \eid{s_{h}} \bigr]_{(\seq{o}{h}{h+k}, \seq{a}{h}{h+k-1}), s_{h}}\notag\\
&=\bigl[\PP(o_{h+k}, o_{h+k-1}, \ldots, o_{h} \given a_{h+k-1}, \ldots, a_{h}, s_{h})\bigr]_{(\seq{o}{h}{h+k}, \seq{a}{h}{h+k-1}), s_{h}}\notag\\
&=\bigl[\PP(\traj^{h+k}_h \given s_{h})\bigr]_{\traj^{h+k}_h, s_{h}} \in \RR^{(A^k\cdot O^{k+1})\times S},
\$
which completes the proof of Lemma \ref{lem::FOO}.
\end{proof}
\subsection{Proof of Lemma \ref{lem::trans_est}}
\label{pf::lem_trans_est}
\begin{proof}
The proof is similar to that of Lemma \ref{lem::FOO}. By the definition of Bellman operators in Definition \ref{def::obs_op}, it holds that
\#\label{eq::pf_lem_trans_est_eq1}
\eid{o_{H}}^\top\ObsB_{H-1}(a_{H-1}, o_{H-1}) \ldots \ObsB_1(a_1, o_1)\UU_{1} &= \Bigl(\Pi_{h = 1}^{H-1}\UU_{h+1}\Tran_h(a_h)\diagO_h(o_h) \UU_{h}^{\dagger}\Bigr)\UU_{1} \notag\\
&= \eid{o_{H}}^\top\UU_{H-1}\Pi_{h = 1}^{H-1}\Tran_h(a_h)\diagO_h(o_h),
\#
where $\mu_1$ is the probability array of initial state distribution. Following the same computation as the proof of Lemma \ref{lem::FOO} in \S\ref{pf::lem_FOO}, we obtain that
\#\label{eq::pf_lem_trans_est_eq2}
\UU_{H-1}\Pi_{h = 1}^{H-1}\Tran_h(a_h)\diagO_h(o_h) \mu_1 = \bigl[\PP(\seq{o}{1}{H+k}\given \seq{a}{1}{H+k-1})\bigr]_{(\seq{o}{1}{H+k}, \seq{a}{1}{H+k-1})}.
\#
Thus, multiplying the right-hand side of \eqref{eq::pf_lem_trans_est_eq2} by the indicator $\eid{o_{H}}$ that takes value $1$ for all indices that contain $o_H$ and a fixed action sequence $\seq{a}{h}{h+k-1}$, we obtain that
\#\label{eq::pf_lem_trans_est_eq3}
\eid{o_{H}}^\top \UU_{H-1}\Pi_{h = 1}^{H-1}\Tran_h(a_h)\diagO_h(o_h) \mu_1 &= \eid{o_{H}}^\top \bigl[\PP(\seq{o}{1}{H+k}\given \seq{a}{1}{H+k-1})\bigr]_{(\seq{o}{1}{H+k}, \seq{a}{1}{H+k-1})}\notag\\
&=\PP(\seq{o}{1}{H} \given \seq{a}{1}{H-1}),
\#
which is as desired. Thus, combining \eqref{eq::pf_lem_trans_est_eq1} and \eqref{eq::pf_lem_trans_est_eq3}, we complete the proof of Lemma \ref{lem::trans_est}.
\end{proof}

\subsection{Proof of Lemma \ref{lem::P1_P2}}
\label{pf::lem_p1_p2}
\begin{proof}
The proof is similar to that of Lemma \ref{lem::FOO}. Recall that we define for all $h\in[H]$ the following operators,
\$
\diagO_h(o_h) &= \diag\bigl(\OO_h(o_h\given \cdot)\bigr) = \diag\Bigl(\bigl[\OO(o_h\given s_h)\bigr]_{s_h}\Bigr) \in \RR^{S\times S},\\
\Tran_h(a_h) &= \PP_h(s_{h+1} = \cdot\given  \cdot, a_h) = \bigl[\PP_h(s_{h+1}\given s_h, a_h)\bigr]_{s_{h} , s_{h+1}} \in \RR^{S\times S},\\
\ObsO_h &= \OO_h( \cdot \given \cdot) = \bigl[\OO_h(o_{h}\given s_h)\bigr]_{ o_{h}, s_{h}} \in \RR^{O\times S}.
\$
Recall that we define
\$
\pone_h(\seq{a}{h-\ell}{h-1}) &= \UU_{h}\Tran_{h-1}(a_{h-1})\diagO_{h-1}(\cdot)\ldots \Tran_{h-\ell}(a_{h-\ell})\diagO_{h-\ell}(\cdot)\mu_{h-\ell},\notag\\
\ptwo_h(\seq{a}{h-\ell}{h}, o_h) &= \UU_{h+1}\Tran_h(a_h)\diagO_h(o_h)\Tran_{h-1}(a_{h-1})\diagO_{h-1}(\cdot)\ldots \Tran_{h-\ell}(a_{h-1})\diagO_{h-\ell}(\cdot)\mu_{h-\ell}.
\$
We first show that the following equation holds,
\#\label{eq::pf_lem_p1p2_eq0}
&\Tran_{h-1}(a_{h-1})\diagO_{h-1}(\cdot)\ldots \Tran_{h-\ell}(a_{h-\ell})\diagO_{h-\ell}(\cdot)\mu_{h-\ell}\notag\\
&\qquad= \PP(s_{h} = \cdot, o_{h-1} = \cdot, \ldots, o_{h-\ell} = \cdot \given a_{h-1}, \ldots, a_{h-\ell}).
\#
To see such a fact, note that for all $o_{h-\ell}\in\cO$, we have
\#\label{eq::pf_lem_p1p2_eq1}
\diagO_{h-\ell}(o_{h-\ell})\mu_{h-\ell} &= \bigl[\OO_{h-\ell}(o_{h-\ell}\given s_{h-\ell})\cdot\PP(s_{h-\ell})\bigr]_{s_{h-\ell}} \in \RR^S.
\#
It thus holds that
\#\label{eq::pf_lem_p1p2_eq2}
\Tran_{h-\ell}(a_{h-\ell})\diagO_{h-\ell}(o_{h-\ell})\mu_{h-\ell} &= \biggl[\sum_{s_{h-\ell}\in\cS}\PP(s_{h-\ell+1}\given s_{h-\ell}, a_{h-\ell})\cdot\OO_{h-\ell}(o_{h-\ell}\given s_{h-\ell})\cdot\PP(s_{h-\ell})\biggr]_{s_{h-\ell+1}}\notag\\
&=\bigl[\PP(s_{h-\ell+1}, o_{h-\ell}\given a_{h-\ell})\bigr]_{s_{h-\ell+1}},
\#
where the second equality follows from the fact that $O_{h-\ell}\indp S_{h-\ell+1}\given s_{h-\ell}$. Thus, by recursive computation similar to \eqref{eq::pf_lem_p1p2_eq1} and \eqref{eq::pf_lem_p1p2_eq2}, we obtain \eqref{eq::pf_lem_p1p2_eq0}, namely,
\#\label{eq::pf_lem_p1p2_eq3}
&\Tran_{h-1}(a_{h-1})\diagO_{h-1}(\cdot)\ldots \Tran_{h-\ell}(a_{h-\ell})\diagO_{h-\ell}(\cdot)\mu_{h-\ell}\notag\\
&\qquad= \PP(s_{h} = \cdot, o_{h-1} = \cdot, \ldots, o_{h-\ell} = \cdot \given a_{h-1}, \ldots, a_{h-\ell})\notag\\
&\qquad= \bigl[\PP(\utraj^{h-1}_{h-\ell}\given s_{h}) \bigr]_{s_{h}, \seq{o}{h-\ell}{h-1}}\in\RR^{S\times O^{\ell}}
\#
Meanwhile, by Lemma \ref{lem::FOO}, we have
\#\label{eq::pf_lem_p1p2_eq4}
\UU_{h} = \bigl[\PP(\traj^{h+k}_h\given s_h )\bigr]_{\traj^{h+k}_h, s_{h}} \in \RR^{(O^{k+1}\cdot A^{k})\times S}.
\#
Thus, it holds for all $h\in[H]$ that
\$
\pone_h(\seq{a}{h-\ell}{h-1}) &= \UU_{h}\Tran_{h-1}(a_{h-1})\diagO_{h-1}(\cdot)\ldots \Tran_{h-\ell}(a_{h-\ell})\diagO_{h-\ell}(\cdot)\mu_{h-\ell}\notag\\
&=\biggl[\sum_{s_{h}\in\cS} \PP( \traj^{h+k}_h \given s_h)\cdot\PP(s_{h}, \seq{o}{h-\ell}{h-1} \given \seq{a}{h-\ell}{h-1})  \biggr]_{(\seq{o}{h}{h+k}, \seq{a}{h}{h+k-1}), \seq{o}{h-\ell}{h-1}}\notag\\
&=\bigl[\PP(\traj^{h+k}_{h-\ell})\bigr]_{\traj^{h+k}_h, \seq{o}{h-\ell}{h-1}} \in \RR^{(O^{k+1}\cdot A^{k})\times O^\ell},
\$
where the second equality follows from the fact that $\seq{o}{h}{h+k} \indp \seq{o}{h-\ell}{h-1} \given s_h$. The computation of $\pone_h(\seq{a}{h-\ell}{h-1})$ is identical to that of $\pone_h(\seq{a}{h-\ell}{h-1})$. In conclusion, we have
\$
\ptwo_h(\seq{a}{h-\ell}{h}, o_h) &= \UU_{h+1}\Tran_h(a_h)\diagO_h(o_h)\Tran_{h-1}(a_{h-1})\diagO_{h-1}(\cdot)\ldots \Tran_{h-\ell}(a_{h-1})\diagO_{h-\ell}(\cdot)\mu_{h-\ell}.\\
&= \bigl[\PP(\traj^{h+k+1}_h)\bigr]_{\traj^{h+k+1}_{h+1}, \seq{o}{h-\ell}{h-1}}\in \RR^{(O^{k+1}\cdot A^{k})\times O^\ell},
\$
which completes the proof of Lemma \ref{lem::P1_P2}.

\end{proof}

\subsection{Proof of Lemma \ref{lem::perform_diff}}
\label{sec::pf_lem_perform_diff}
\begin{proof}
By Lemma \ref{lem::trans_est} and \eqref{eq::lem_policy_traj}, we have
\#\label{eq::pf_lem_perform_diff_1}
V^\pi(\param) &= \sum_{\seq{o}{1}{H-1}\in\cO^H} r(\seq{o}{1}{H-1})\cdot\PP^\pi(\seq{o}{1}{H-1})\notag\\
&= \sum_{\seq{o}{1}{H-1}\in\cO^H} r(\seq{o}{1}{H-1})\cdot \eid{o_{H}}^\top \ObsB^\param_{H-1}(a^\pi_{H-1}, o_{H-1}) \ldots \ObsB^\param_1(a^\pi_1, o_1) \vecb^\param_1.
\#
Here the actions $a^\pi_h$ are taken based on the policy $\pi$ and the past observations and actions taken $(o_1, a^\pi_1, \ldots, a^\pi_{h-1}, o_{h})$. In addition, recall that we define $\vecb^\param_1 = \UU^\param_{1} \mu_1$, where $\mu_1$ is the probability array of initial state distribution. It thus follows from \eqref{eq::pf_lem_perform_diff_1} that
\#\label{eq::pf_lem_perform_diff_2}
V^\pi(\param) - V^\pi(\param') &= \sum_{\seq{o}{1}{H-1}\in\cO^H}\sum_{h=1}^{H-1} r(\seq{o}{1}{H-1})\cdot \eid{o_{H}}^\top \ObsB^\param_{H-1}(a^\pi_{H-1}, o_{H-1}) \cdots\notag\\
&\qquad \qquad \cdots \bigl(\ObsB^\param_{h}(a^\pi_{h}, o_{h}) - \ObsB^{\param'}_{h}(a^\pi_{h}, o_{h})\bigr)\ObsB^{\param'}_{h-1}(a^\pi_{h-1}, o_{h-1}) \cdots  \ObsB^{\param'}_1(a^\pi_1, o_1) \vecb^{\param'}_1\notag\\
&\qquad+ \sum_{\seq{o}{1}{H-1}\in\cO^H} r(\seq{o}{1}{H-1})\cdot \eid{o_{H}}^\top \ObsB^\param_{H-1}(a^\pi_{H-1}, o_{H-1}) \ldots \ObsB^\param_1(a^\pi_1, o_1) (\vecb^\param_1 - \vecb^{\param'}_1).
\#
In the sequel, we upper bound the absolute value of the right-hand side of \eqref{eq::pf_lem_perform_diff_2}. We define the following vectors for all $h\in[H-1]$ for notational simplicity,
\$
v_h &= \bigl(\ObsB^\param_{h}(a^\pi_{h}, o_{h}) - \ObsB^{\param'}_{h}(a^\pi_{h}, o_{h})\bigr)\ObsB^{\param'}_{h-1}(a^\pi_{h-1}, o_{h-1})  \ldots  \ObsB^{\param'}_1(a^\pi_1, o_1)\vecb^{\param'}_1,\notag\\
v_0 &= \vecb^\param_1 - \vecb^{\param'}_1.
\$
Since $0\leq r(\seq{o}{1}{H-1}) \leq H$ for all observation sequences $\seq{o}{1}{H-1}\in\cO^H$, we obtain that
\$
&\biggl|\sum_{\seq{o}{1}{H-1}\in\cO^H}\sum_{h=1}^{H-1} r(\seq{o}{1}{H-1})\cdot \eid{o_{H}}^\top \ObsB^\param_{H-1}(a^\pi_{H-1}, o_{H-1})\ldots \ObsB^\param_{h+1}(a^\pi_{h+1}, o_{h+1}) v_h\biggr|\notag\\
&\qquad\leq H \cdot \biggl|\sum_{\seq{o}{1}{H-1}\in\cO^H}\sum_{h=1}^{H-1}  \eid{o_{H}}^\top \ObsB^\param_{H-1}(a^\pi_{H-1}, o_{H-1})\ldots \ObsB^\param_{h+1}(a^\pi_{h+1}, o_{h+1}) v_h\biggr|.
\$
Moreover, since the vector $\eid{o_{H}}$ takes value in $\{0, 1\}$ for all the indices, we further obtain that
\#\label{eq::pf_lem_perform_diff_3}
&H \cdot \biggl|\sum_{\seq{o}{1}{H-1}\in\cO^H}\sum_{h=1}^{H-1}  \eid{o_{H}}^\top \ObsB^\param_{H-1}(a_{H-1}, o_{H-1})\ldots \ObsB^\param_{h+1}(a^\pi_{h+1}, o_{h+1}) v_h\biggr|\notag\\
&\qquad\leq H\cdot \sum_{\seq{o}{1}{H-1}\in\cO^H}\sum_{h=1}^{H-1} \|\ObsB^\param_{H-1}(a^\pi_{H-1}, o_{H-1})\ldots \ObsB^\param_{h+1}(a^\pi_{h+1}, o_{h+1}) v_h\|_1
\#
It now suffices to upper bound the right-hand side of \eqref{eq::pf_lem_perform_diff_3}. 
By the definition of Bellman operators in Definition \ref{def::obs_op}, we obtain for all $h\in\{0, \ldots, H-1\}$ that
\#\label{eq::pf_lem_perform_diff_4}
&\|\ObsB^\param_{H-1}(a^\pi_{H-1}, o_{H-1})\ldots \ObsB^\param_{h+1}(a^\pi_{h+1}, o_{h+1}) v_h\|_1\notag\\
&\qquad = \| \UU_{h} \Tran_{H-1}(a^\pi_{H-1})\diagO_{H-1}(o_{H-1})\ldots \Tran_{h}(a^\pi_{h})\diagO_{h}(o_{h})\UU_{h}^\dagger v_h  \|_1
\#
The following lemma upper bound the right-hand side of \eqref{eq::pf_lem_perform_diff_4}.
\begin{lemma}
\label{lem::norm_bound}
It holds for all $\seq{a}{h}{H-1}\in\cA^{H-h}$, $h\in[H]$, and $u \in\RR^S$ that
\$
\sum_{\seq{o}{h}{H-1}\in\cO^{H-h}}\|\UU_{h} \Tran_{H-1}(a_{H-1})\diagO_{H-1}(o_{H-1})\ldots \Tran_{h}(a_{h})\diagO_{h}(o_{h}) u\|_1 \leq \|u\|_1.
\$
\end{lemma}
\begin{proof}
See \S\ref{sec::pf_lem_norm_bound} for a detailed proof.
\end{proof}
Meanwhile, by Assumption \ref{asu::inv_U} and the fact that $\|A\|_{1 \mapsto 1} \leq \sqrt{S} \|A\|_2$ for any matrix $A\in\RR^{S\times N}$, we obtain that
\#\label{eq::pf_lem_perform_diff_6}
\|\UU_{h}^\dagger v_h \|_1 \leq \|\UU_{h}^\dagger\|_{1 \mapsto 1} \cdot \| v_h \|_1 \leq \sqrt{S} \cdot \|\UU_{h}^\dagger\|_2\cdot \|v_h\|_1 \leq \nu\cdot\sqrt{S}\cdot \|v_h\|_1
\#
for all $h\in[H]$. Combining Lemma \ref{lem::norm_bound}, \eqref{eq::pf_lem_perform_diff_3}, \eqref{eq::pf_lem_perform_diff_4}, and \eqref{eq::pf_lem_perform_diff_6}, we obtain that
\$
|V^\pi(\param) - V^\pi(\param')| &\leq \nu\cdot\sqrt{S}\cdot H\cdot \sum_{h=0}^{H-1}\sum_{\seq{o}{1}{h}\in\cO^h} \|v_h\|_1\notag\\
&=\nu\cdot\sqrt{S}\cdot H\cdot \sum_{h=1}^{H-1}\sum_{\seq{o}{1}{h}\in\cO^h} \bigl\|\bigl(\ObsB^\param_{h}(a_{h}, o_{h}) - \ObsB^{\param'}_{h}(a_{h}, o_{h})\bigr)\ObsB^{\param'}_{h-1}(a_{h-1}, o_{h-1})  \ldots  \ObsB^{\param'}_1(a_1, o_1)\vecb^{\param'}_1\bigr\|_1\notag\\
&\qquad +\nu\cdot\sqrt{S}\cdot H\cdot\bigl\|\bigl(\ObsB^\param_{1}(a_{1}, o_{1}) - \ObsB^{\param'}_{1}(a_{1}, o_{1})\bigr)\vecb^{\param'}_1\bigr\|_1 +\nu\cdot\sqrt{S}\cdot H\cdot\|\vecb^\param_1 - \vecb^{\param'}_1\|_1,
\$
which completes the proof of Lemma \ref{lem::perform_diff}.
\end{proof}

\subsection{Proof of Lemma \ref{lem::upper_bound_rhs}}
\label{sec::pf_lem_upper}
\begin{proof}
It holds for all policy $\pi$ and the corresponding action sequence $a^\pi_{1:h}$ generated by $\pi$ that
\#\label{eq::pf_lem_upper_eq1}
&\sum_{\seq{o}{1}{h}\in\cO^{h}} \bigl\|\bigl(\ObsB^\param_{h}(a^\pi_{h}, o_{h}) - \ObsB^{\param'}_{h}(a^\pi_{h}, o_{h})\bigr)\ObsB^{\param'}_{h-1}(a^\pi_{h-1}, o_{h-1})  \ldots  \ObsB^{\param'}_1(a^\pi_1, o_1)\vecb^{\param'}_1\bigr\|_1 \\
&\quad \leq \sum_{\seq{a}{h-\ell}{h}\in\cA^{\ell+1}}\sum_{\seq{o}{1}{h}\in\cO^{h}}\bigl\|\bigl(\ObsB^\param_{h}(a_{h}, o_{h}) - \ObsB^{\param'}_{h}(a_{h}, o_{h})\bigr)\ObsB^{\param'}_{h-1}(a_{h-1}, o_{h-1})  \ldots  \ObsB^{\param'}_1(a_1, o_1)\vecb^{\param'}_1\bigr\|_1\notag\\
&\quad = \sum_{\seq{a}{h-\ell}{h}\in\cA^{\ell+1}}\sum_{\seq{o}{1}{h}\in\cO^{h}} \bigl\|\bigl(\ObsB^\param_{h}(a_{h}, o_{h}) - \ObsB^{\param'}_{h}(a_{h}, o_{h})\bigr)\UU^{\param'}_{h,k} \Tran^{\param'}_{h-1}(a_{h-1})\mu^{\param'}_{h-1}(\pi, \seq{a}{h-\ell}{h-1}, \seq{o}{1}{h-1}) \bigr\|_1,\notag
\#
where $\mu^{\param'}_{h-1}(\pi, \seq{a}{h-\ell}{h-1}, \seq{o}{1}{h-1})$ is the state distribution array defined in \eqref{eq::def_state_arr_pf_mu}. Note that on the right-hand side of \eqref{eq::pf_lem_upper_eq1}, the state distribution array $\mu^{\param'}_{h-1}(\pi, \seq{a}{h-\ell}{h-1}, \seq{o}{1}{h-1})$ is the only term that is related to policy $\pi$. In what follows, we define
\#\label{eq::pf_lem_upper_def_M}
M_h(\param, \param', a_h, a_{h-1}, o_h) = \bigl(\ObsB^\param_{h}(a_{h}, o_{h}) - \ObsB^{\param'}_{h}(a_{h}, o_{h})\bigr)\UU^{\param'}_{h,k} \Tran^{\param'}_{h-1}(a_{h-1}) \in \RR^{(A^{k+1}\cdot O^k)\times S}
\#
for notational simplicity. It holds that 
\#\label{eq::pf_lem_upper_eq2}
&\sum_{\seq{o}{1}{h}\in\cO^{h}} \|M_h(\param, \param', a_h, a_{h-1}, o_h)\mu^{\param'}_{h-1}(\pi, \seq{a}{h-\ell}{h-1}, \seq{o}{1}{h-1})\|_1\notag\\
&\qquad \leq  \sum_{\seq{o}{1}{h}\in\cO^{h}} \sum_{s_{h-1} \in \cS} \bigl\| \bigl[M_h(\param, \param', a_h, a_{h-1}, o_h)\bigr]_{s_{h-1}}\|_1 \cdot \bigl[\mu^{\param'}_{h-1}(\pi, \seq{a}{h-\ell}{h-1}, \seq{o}{1}{h-1})\bigr]_{s_{h-1}},
\#
where we denote by $[M_h(\param, \param', a_h, a_{h-1}, o_h)]_{s_{h-1}}$ and $[\mu^{\param'}_{h-1}(\pi, \seq{a}{h-\ell}{h-1}, \seq{o}{1}{h-1})]_{s_{h-1}}$ the $s_{h-1}$-th column of $M_h(\param, \param', a_h, a_{h-1}, o_h)$ and the $s_{h-1}$-th entry of $\mu^{\param'}_{h-1}(\pi, \seq{a}{h-\ell}{h-1}, \seq{o}{1}{h-1})$, respectively. Recall that we have
\$
\bigl[\mu^{\param'}_{h-1}(\pi, \seq{a}{h-\ell}{h-1}, \seq{o}{1}{h-1})\bigr]_{s_{h-1}} = \PP^{\param}(s_h, \seq{o}{1}{h-1} \given \seq{a}{h-\ell}{h-1}, \pi).
\$
Thus, by marginalizing over the observation sequence $\seq{o}{1}{h-1}$, it holds that
\#\label{eq::pf_lem_upper_eq3}
\sum_{\seq{o}{1}{h-1}\in\cO^{h-1}}\bigl[\mu^{\param'}_{h-1}(\pi, \seq{a}{h-\ell}{h-1}, \seq{o}{1}{h-1})\bigr]_{s_{h-1}} = \PP^{\param'}(s_h \given \seq{a}{h-\ell}{h-1}, \pi).
\#
Plugging \eqref{eq::pf_lem_upper_eq3} into \eqref{eq::pf_lem_upper_eq2}, we obtain that
\$
&\sum_{\seq{o}{1}{h}\in\cO^{h}}\|M_h(\param, \param', a_h, a_{h-1}, o_h)\mu^{\param'}_{h-1}(\pi, \seq{a}{h-\ell}{h-1}, \seq{o}{1}{h-1})\|_1\notag\\
&\qquad \leq  \sum_{o_h\in\cO} \sum_{s_{h-1} \in \cS} \bigl\| \bigl[M_h(\param, \param', a_h, a_{h-1}, o_h)\bigr]_{s_{h-1}}\|_1 \cdot \PP^{\param'}(s_h \given \seq{a}{h-\ell}{h-1}, \pi)\notag\\
&\qquad = \sum_{\seq{a}{h-\ell}{h}\in\cA^{\ell+1}}\sum_{o_h\in\cO} \bigl\|\bigl(\ObsB^\param_{h}(a_{h}, o_{h}) - \ObsB^{\param'}_{h}(a_{h}, o_{h})\bigr)\UU^{\param'}_{h,k} \Tran^{\param'}_{h-1}(a_{h-1})\diag\bigl(\tilde\mu^{\param'}_{h-1}(\pi, \seq{a}{h-\ell}{h-1})\bigr) \bigr\|_1.
\$
Thus, we complete the proof of Lemma \ref{lem::upper_bound_rhs}.
\end{proof}

\subsection{Proof of Lemma \ref{lem::concen}}
\label{sec::pf_lem_concen}
\begin{proof}
Recall that we aim to recover the following density matrices in the estimation of Bellman operators in \S\ref{sec::bellman_tabular},
\#
&\vecb_1 = \bigl[\PP(\traj^{k}_{1})\bigr]_{\traj^{k}_{1}}\in\RR^{\cA^k\cdot\cO^{k+1}},\notag\\
&\pone^t_h(\seq{a}{h-\ell}{h-1}) = \bigl[\PP^t(\traj^{h+k}_{h-\ell}) \bigr]_{\traj^{h+k}_{h}, \seq{o}{h-\ell}{h-1}}  \in\RR^{(\cA^{k}\cdot \cO^{k+1})\times\cO^\ell},\label{eq::pf_p1t}\\
&\ptwo^t_h(\seq{a}{h-\ell}{h}, o_h)= \bigl[\PP^t(\traj^{h+k+1}_{h-\ell})\bigr]_{\traj^{h+k+1}_{h+1}, \seq{o}{h-\ell}{h-1}}\in\RR^{(\cA^{k}\cdot \cO^{k+1})\times\cO^\ell}.\label{eq::pf_p2t}
\#
Here we denote by $\PP^t$ the visitation measure of the mixed policy $\{\pi^\omega\}_{\omega\in[t]}$ generated by Algorithm \ref{alg::POMDP}. Alternatively, we can write the densities in \eqref{eq::pf_p1t} in the following vector product form,
\$
\pone^t_h(\seq{a}{h-\ell}{h-1}) = \sum_{\traj^{h+k}_{h}\in\cA^k\times\cO^{k+1}} \biggl(\sum_{\seq{o}{h-\ell}{h-1}\in\cO^\ell} \eid{\traj^{h+k}_{h}}\eid{\seq{o}{h-\ell}{h-1}}^\top\cdot \PP^t(\traj^{h+k}_{h-\ell})\biggr).
\$
Recall that we adopt the following estimator of the probability density defined in \eqref{eq::pf_p1t},
\#\label{eq::pf_lem_concen_eq1}
\hat \pone^t_h(\seq{a}{h-\ell}{h-1}) &=  \frac{1}{t}\cdot\sum_{\seq{a}{h}{h+k-1}\in\cA^k}\biggl(\sum_{\traj^{h+k}_{h-\ell}\in\cD^t(\seq{a}{h-\ell}{h+k-1})} \eid{\traj^{h+k}_{h-\ell}}\eid{\seq{o}{h-\ell}{h-1}}^\top\biggr).
\#
By the martingale concentration inequality (see e.g., \cite{jin2019short}) and the fact that
\$
\|\eid{\traj^{h+k}_{h-\ell}}\eid{\seq{o}{h-\ell}{h-1}}^\top\|_F \leq 1
\$
for all trajectory $\traj^{h+k}_{h-\ell}\in\cA^{k+\ell}\times\cO^{k+\ell+1}$ and observation sequence $\seq{o}{h-\ell}{h-1}\in\cO^{\ell}$, we obtain for all $h\in[H]$ and $\seq{a}{h-\ell}{h+k-1}\in\cA^{k+\ell}$ that
\$
&\biggl\|\sum_{\seq{o}{h-\ell}{h+k}\in\cO^{k+\ell+1}}\eid{\traj^{h+k}_h}\eid{\seq{o}{h-\ell}{h-1}}^\top\cdot \PP^t(\traj^{h+k}_{h-\ell})\notag\\
&\qquad -  \frac{1}{t-1}\sum_{\traj^{h+k}_{h-\ell}\in\cD^t(\seq{a}{h-\ell}{h+k-1})} \eid{\traj^{h+k}_{h-\ell}}\eid{\seq{o}{h-\ell}{h-1}}^\top \biggr\|_F \leq C\cdot(k+\ell)\cdot\sqrt{\log(O\cdot A\cdot T/\delta)/t}
\$
with probability at least $1 - \delta/A^{k}$, where $C$ is a positive absolute constant. It thus holds for all $h\in[H]$, $t\in[T]$, and $\seq{a}{h-\ell}{h-1}\in\cA^{\ell}$ that
\$
\|\pone^t_h(\seq{a}{h-\ell}{h-1}) - \hat \pone^t_h(\seq{a}{h-\ell}{h-1})\|_F &\leq  \sum_{\seq{a}{h}{h+k-1}\in\cA^k}\biggl\|\sum_{\seq{o}{h-\ell}{h+k}\in\cO^{k+\ell+1}}\eid{\traj^{h+k}_h}\eid{\seq{o}{h-\ell}{h-1}}^\top\cdot \PP^t(\traj^{h+k}_{h-\ell})\notag\\
&\qquad\qquad\qquad\quad- \frac{1}{t-1}\sum_{\traj^{h+k}_{h-\ell}\in\cD^t(\seq{a}{h-\ell}{h+k-1})} \eid{\traj^{h+k}_{h-\ell}}\eid{\seq{o}{h-\ell}{h-1}}^\top\biggr\|_F\notag\\\
&= \cO\bigl(A^k\cdot(k+\ell)\cdot \sqrt{\log(O\cdot A\cdot T\cdot H/\delta)/t}\bigr)
\$
with probability at least $1 - \delta/2$. Meanwhile, recall that we estimate $\ptwo^t_h$ in \eqref{eq::pf_p2t} by the following estimator,
\#\label{eq::pf_lem_concen_eq2}
\hat\ptwo^t_h(\seq{a}{h-\ell}{h}, o_h) &= \frac{1}{t-1}\cdot\sum_{\seq{a}{h+1}{h+k}\in\cA^k} \sum_{\traj^{h+k+1}_{h-\ell}\in\cD^t(\seq{a}{h-\ell}{h+k})} \eid{\traj^{h+k+1}_{h+1}}\eid{\seq{o}{h-\ell}{h-1}}^\top,
\#
Following a similar computation, it holds for all $h\in[H]$, $t\in[T]$, $\seq{a}{h-\ell}{h-1}\in\cA^{\ell}$, $a_h\in\cA$, and $o_h\in\cO$ that
\$
\|\ptwo^t_h(\seq{a}{h-\ell}{h}, o_h) -  \hat\ptwo^t_h(\seq{a}{h-\ell}{h}, o_h)\|_F = \cO\bigl(A^k\cdot(k+\ell)\cdot \sqrt{\log(O\cdot A\cdot T/\delta)/t}\bigr)
\$
with probability at least $1 - \delta/2$. Similarly, it also holds for all $h\in[H]$ and $t\in[T]$ that
\$
\|\vecb_1 - \hat \vecb^t_1\|_2 = \cO(A^k\cdot k\cdot\sqrt{\log(O\cdot A\cdot T/\delta)/t})
\$
with probability at least $1 - \delta/2$. Thus, we complete the proof of Lemma \ref{lem::concen}.
\end{proof}

\subsection{Proof of Lemma \ref{lem::good_event}}
\label{sec::pf_lem_good_event}
\begin{proof}
In what follows, we prove \eqref{eq::lem_ge_eq1}--\eqref{eq::lem_ge_eq3} separately.
\smallsec{Part I: Proof of Upper Bound in \eqref{eq::lem_ge_eq1}.} By Lemma \ref{lem::concen}, it holds with probability at least $1 - \delta$ that
\#\label{eq::pf_lem_ge_eq4}
\|\vecb_1 - \hat \vecb^t_1\|_1 \leq \sqrt{A^k\cdot O^{k+1}}\cdot \|\vecb_1 - \hat \vecb^t_1\|_2 = \cO\Bigl((k+\ell)\cdot\sqrt{A^{3k}\cdot O^{k+1}\cdot\log(O\cdot A\cdot T\cdot H/\delta)/t}\Bigr).
\#
Meanwhile, by the update of $\vecb^{\param^t}_1$ in Algorithm \ref{alg::POMDP}, it holds with probability at least $1 - \delta$ that $\vecb^{\param^t}_1 \in {\cC^t}$, namely, 
\#\label{eq::pf_lem_ge_eq5}
\|\vecb^{\param^t}_1 - \hat\vecb^t_1\|_2 \leq C\cdot\nu\cdot (k+\ell)\cdot \sqrt{A^{5k+1}\cdot O^{k+\ell}\cdot\log(O\cdot A\cdot T\cdot H/\delta)/t}.
\#
Combining \eqref{eq::pf_lem_ge_eq4} and \eqref{eq::pf_lem_ge_eq5}, it holds with probability at least $1 - \delta$ that
\$
\|\vecb_1 - \vecb^{\param^t}_1\|_1 &\leq \|\vecb_1 - \hat \vecb^t_1\|_1 + \|\vecb^{\param^t}_1 - \hat\vecb^t_1\|_1 \leq \sqrt{A^k\cdot O^{k+1}}\cdot\bigl(\|\vecb_1 - \hat \vecb^t_1\|_2 + \|\vecb^{\param^t}_1 - \hat\vecb^t_1\|_2\bigr)\notag\\
& = \cO\Bigl(\nu\cdot (k+\ell)\cdot \sqrt{A^{5k+1}\cdot O^{k+\ell}\cdot\log(O\cdot A\cdot T\cdot H/\delta)/t}\Bigr),
\$
which completes the proof of \eqref{eq::lem_ge_eq1}.
\smallsec{Part II: Proof of Upper Bound in \eqref{eq::lem_ge_eq2}.}
Recall that we define
\$
&\vecb_1 = \UU_{1} \mu_1 = \bigl[\PP(\traj^{k+1}_1)\bigr]_{\traj^{k+1}_1} \in \RR^{O^{k+1}\cdot A^k}, \notag\\
&\pone^t_1(\seq{a}{1-\ell}{0}) = \bigl[\PP^t(\traj^{k+1}_{1-\ell}) \bigr]_{\traj^{k+1}_1, \seq{o}{1-\ell}{0}} \in\RR^{(O^{k+1}\cdot A^k)\times O^\ell}.
\$
Thus, it holds for all action array $\seq{a}{1-\ell}{0}$ and $t\in[T]$ that
\#\label{eq::pf_lem_ge_II_eq0.5}
[\vecb_1]_{\traj^{k+1}_1} &= \PP(\traj^{k+1}_1) = \sum_{\seq{o}{1-\ell}{0} \in\cO^\ell} \PP^t(\utraj^0_{1-\ell})\cdot \PP^t(\traj^{k+1}_{1-\ell}).
\#
It thus holds for all $\seq{a}{1-\ell}{0}\in\cA^{\ell}$ that
\#\label{eq::pf_lem_ge_II_eq1}
&\bigl\|\bigl(\ObsB^{\param^t}_{1}(a_{1}, o_{1}) - \ObsB^{\param^*}_{1}(a_{1}, o_{1})\bigr) \vecb_1\|_1  \notag\\
&\qquad= \biggl\|\bigl(\ObsB^{\param^t}_{1}(a_{1}, o_{1}) - \ObsB^{\param^*}_{1}(a_{1}, o_{1})\bigr)\sum_{\seq{o}{1-\ell}{0} \in\cO^\ell} \PP^t(\utraj^0_{1-\ell})\cdot \bigl[\pone^t_1(\seq{a}{1-\ell}{0})\bigr]_{\seq{o}{1-\ell}{0}}  \biggr\|_1\notag\\
&\qquad\leq\biggl\|\sum_{\seq{o}{1-\ell}{0} \in\cO^\ell} \bigl(\ObsB^{\param^t}_{1}(a_{1}, o_{1}) - \ObsB^{\param^*}_{1}(a_{1}, o_{1})\bigr) \bigl[\pone^t_1(\seq{a}{1-\ell}{0})\bigr]_{\seq{o}{1-\ell}{0}}  \biggr\|_1\notag\\
&\qquad =\bigl\|\bigl(\ObsB^{\param^t}_{1}(a_{1}, o_{1}) - \ObsB^{\param^*}_{1}(a_{1}, o_{1})\bigr)\pone^t_1(\seq{a}{1-\ell}{0})  \bigr\|_1,
\#
Here with a slight abuse of notation, we denote by $\bigl[\pone^t_1(\seq{a}{1-\ell}{0})\bigr]_{\seq{o}{1-\ell}{0}}$ the $\seq{o}{1-\ell}{0}$-th column of the matrix $\pone^t_1(\seq{a}{1-\ell}{0})\in\RR^{(O^{k+1}\cdot A^k)\times O^\ell}$. Meanwhile, the inequality follows from the fact that $0 \leq \PP^t(\utraj^{0}_{1-\ell}) \leq 1$ for all $t\in[T]$ and $\utraj^{0}_{1-\ell}\in\cA^\ell\times\cO^\ell$. It remains to establish high confidence bound for the right-hand side of \eqref{eq::pf_lem_ge_II_eq1}. To this end, we first obtain by triangle inequality that
\#\label{eq::pf_lem_ge_II_eq2}
&\bigl\|\bigl(\ObsB^{\param^t}_{1}(a_{1}, o_{1}) - \ObsB^{\param^*}_{1}(a_{1}, o_{1})\bigr)\pone^t_1(\seq{a}{1-\ell}{0})  \bigr\|_1 \leq  {\rm (i)} + {\rm (ii)} + {\rm (iii)} + {\rm (iv)},\notag\\
\#
where we define
\$
{\rm (i)} &= \bigl\| \ObsB^{\param^t}_{1}(a_{1}, o_{1})\bigl(\hat\pone^t_1(\seq{a}{1-\ell}{0}) - \pone^t_1(\seq{a}{1-\ell}{0})\bigr) \bigr\|_1,\notag\\
{\rm (ii)}&= \|\ObsB^{\param^t}_{1}(a_{1}, o_{1})\hat\pone^t_1(\seq{a}{1-\ell}{0}) - \hat\ptwo^t_1(\seq{a}{1-\ell}{1}, o_1)\|_1,\notag\\
{\rm (iii)}&=\|\hat\ptwo^t_1(\seq{a}{1-\ell}{1}, o_1) - \ptwo^t_1(\seq{a}{1-\ell}{1}, o_1)\|_1,\notag\\
{\rm (iv)}&=\|\ObsB^{\param^*}_{1}(a_{1}, o_{1})\pone^t_1(\seq{a}{1-\ell}{0}) - \ptwo^t_1(\seq{a}{1-\ell}{1}, o_1)\|_1.
\$
In what follows, we upper bound terms (i)--(iv) on the right-hand side of \eqref{eq::pf_lem_ge_II_eq2}. By the definition of parameter space and the concentration inequality in Lemma \ref{lem::concen}, we obtain that
\#\label{eq::pf_lem_ge_II_i}
{\rm (i)} &= \bigl\| \ObsB^{\param^t}_{1}(a_{1}, o_{1})\bigl(\hat\pone^t_1(\seq{a}{1-\ell}{0}) - \pone^t_1(\seq{a}{1-\ell}{0})\bigr) \bigr\|_1 \notag\\
&\leq \|\ObsB^{\param^t}_{1}(a_{1}, o_{1})\|_{1\mapsto1}\cdot \|\hat\pone^t_1(\seq{a}{1-\ell}{0}) - \pone^t_1(\seq{a}{1-\ell}{0})\|_1\notag\\
& = \cO\bigl(\nu\cdot A^{2k}\cdot\sqrt{A^{k+1}\cdot O^{k+\ell}}\cdot (k+\ell)\cdot \sqrt{\log(O\cdot A\cdot T\cdot H/\delta)/T}\bigr),
\#
where the third equality follows from the fact that $\|\ObsB^{\param^t}_{1}(a_{1}, o_{1})\|_{1\mapsto1} \leq \nu\cdot A^k$ and the fact that $\|x\|_1 \leq \sqrt{mn}\cdot \|x\|_2$ for $x\in\RR^{m\times n}$. Meanwhile, by the definition of confidence set ${\cC^t}$ in \eqref{eq::def_CI} and the fact that $\param^t \in{\cC^t}$, we obtain that
\#\label{eq::pf_lem_ge_II_ii}
{\rm (ii)} &= \|\ObsB^{\param^t}_{1}(a_{1}, o_{1})\hat\pone^t_1(\seq{a}{1-\ell}{0}) - \hat\ptwo^t_h(\seq{a}{h-\ell}{h}, o_h)\|_1\notag\\
&=\cO\bigl(\nu\cdot A^{2k}\cdot\sqrt{A^{k+1}\cdot O^{k+\ell}}\cdot (k+\ell)\cdot \sqrt{\log(O\cdot A\cdot T\cdot H/\delta)/T}\bigr).
\#
By the concentration inequality in Lemma \ref{lem::concen}, we further obtain that
\#\label{eq::pf_lem_ge_II_iii}
{\rm (iii)} &= \|\hat\ptwo^t_h(\seq{a}{h-\ell}{h}, o_h) - \ptwo^t_h(\seq{a}{h-\ell}{h}, o_h)\|_1 \notag\\
&= \cO\bigl(\sqrt{A^{k+1}\cdot O^{k+\ell}}\cdot A^k\cdot (k+\ell)\cdot \sqrt{\log(O\cdot A\cdot T\cdot H/\delta)/T}\bigr).
\#
Finally, by the identity of Bellman operators in \eqref{eq::id_OO}, we have
\#\label{eq::pf_lem_ge_II_iv}
{\rm (iv)}=\|\ObsB^{\param^*}_{1}(a_{1}, o_{1})\pone^t_1(\seq{a}{1-\ell}{0}) - \ptwo^t_h(\seq{a}{h-\ell}{h}, o_h)\|_F  = 0.
\#
Plugging \eqref{eq::pf_lem_ge_II_i}, \eqref{eq::pf_lem_ge_II_ii}, \eqref{eq::pf_lem_ge_II_iii}, and \eqref{eq::pf_lem_ge_II_iv} into \eqref{eq::pf_lem_ge_II_eq2}, it holds for all $a_1 \in\cA$, $o_1 \in\cO$, and $\seq{a}{1-\ell}{0}\in\cA^\ell$ with probability at least $1 - \delta$ that
\$
&\bigl\|\bigl(\ObsB^{\param^t}_{1}(a_{1}, o_{1}) - \ObsB^{\param^*}_{1}(a_{1}, o_{1})\bigr)\pone^t_1(\seq{a}{1-\ell}{0})  \bigr\|_1  \notag\\
&\qquad= \cO\bigl(\nu\cdot (k+\ell)\cdot \sqrt{A^{5k+1}\cdot O^{k+\ell}\cdot\log(O\cdot A\cdot T\cdot H/\delta)/t}\bigr),
\$
which completes the proof of \eqref{eq::lem_ge_eq2}.

\smallsec{Part III: Proof of Upper Bound in \eqref{eq::lem_ge_eq3}.} Under Assumption \ref{asu::inv_rev_op}, it holds for all $1<h\leq H$ that
\#\label{eq::pf_lem_ge_III_eq1}
&\sum_{s_{h-1}\in\cS}\bigl\|\bigl(\ObsB^{\param^t}_{h}(a_{h}, o_{h}) - \ObsB^{\param^*}_{h}(a_{h}, o_{h})\bigr)\UU^{\param^*}_{h,k} \Tran^{\param^*}_{h-1}(s_{h-1}, a_{h-1})\bigr\|_1\cdot\PP^t(s_{h-1} \given \seq{a}{h-\ell}{h-1})\notag\\
&\qquad = \bigl\|\bigl(\ObsB^\param_{h}(a_{h}, o_{h}) - \ObsB^{\param'}_{h}(a_{h}, o_{h})\bigr)\UU^{\param'}_{h,k} \Tran^{\param'}_{h-1}(a_{h-1})\diag\bigl(\tilde\mu^{\param'}_{h-1}(\overline\pi^t, \seq{a}{h-\ell}{h-1})\bigr) \bigr\|_1\notag\\
&\qquad = \bigl\|\bigl(\ObsB^\param_{h}(a_{h}, o_{h}) - \ObsB^{\param'}_{h}(a_{h}, o_{h})\bigr)\UU^{\param'}_{h,k} \Tran^{\param'}_{h-1}(a_{h-1})A^{\overline\pi^t}_{h-1, \ell}(\seq{a}{h-\ell}{h-2})C^{\overline\pi^t, \dagger}_{h-1, \ell}(\seq{a}{h-\ell}{h-1})\bigr\|_1,
\#
where $A^{\overline\pi^t}_{h-1, \ell}$ is the reverse emission operator defined in \eqref{def::rev_op} and $C^{\overline\pi^t, \dagger}_{h-1, \ell}$ is the right inverse of $C^{\overline\pi^t}_{h-1, \ell}$ in Assumption \ref{asu::inv_rev_op}. Meanwhile, by Assumption \ref{asu::inv_rev_op}, it holds that 
\#\label{eq::pf_lem_ge_III_eq2}
&\bigl\|\bigl(\ObsB^\param_{h}(a_{h}, o_{h}) - \ObsB^{\param'}_{h}(a_{h}, o_{h})\bigr)\UU^{\param'}_{h,k} \Tran^{\param'}_{h-1}(a_{h-1})A^{\overline\pi^t}_{h-1, \ell}(\seq{a}{h-\ell}{h-2})C^{\overline\pi^t, \dagger}_{h-1, \ell}(\seq{a}{h-\ell}{h-1})\bigr\|_1 \notag\\
&\qquad \leq \gamma\cdot \|\bigl(\ObsB^\param_{h}(a_{h}, o_{h}) - \ObsB^{\param'}_{h}(a_{h}, o_{h})\bigr)\UU^{\param'}_{h,k} \Tran^{\param'}_{h-1}(a_{h-1})A^{\overline\pi^t}_{h-1, \ell}(\seq{a}{h-\ell}{h-1})\bigr\|_1.
\#
By the identity of $\pone^t_h(\seq{a}{h-\ell}{h-1})$ in \eqref{eq::rev_op_id}, we further obtain that
\#\label{eq::pf_lem_ge_III_eq3}
&\|\bigl(\ObsB^\param_{h}(a_{h}, o_{h}) - \ObsB^{\param'}_{h}(a_{h}, o_{h})\bigr)\UU^{\param'}_{h,k} \Tran^{\param'}_{h-1}(a_{h-1})A^{\overline\pi^t}_{h-1, \ell}\bigr\|_1 \notag\\
&\qquad = \|\bigl(\ObsB^\param_{h}(a_{h}, o_{h}) - \ObsB^{\param'}_{h}(a_{h}, o_{h})\bigr)\pone^t_h(\seq{a}{h-\ell}{h-1})\bigr\|_1.
\#
We now upper bound the right-hand side of \eqref{eq::pf_lem_ge_III_eq3}. The calculation is similar to that in Part II of the proof. By triangle inequality, we obtain
\#\label{eq::pf_lem_ge_III_eq4}
\|\bigl(\ObsB^\param_{h}(a_{h}, o_{h}) - \ObsB^{\param'}_{h}(a_{h}, o_{h})\bigr) \pone^t_h(\seq{a}{h-\ell}{h-1})\bigr\|_1 \leq {\rm (v)} + {\rm (vi)} + {\rm (vii)} + {\rm (viii)},
\#
where we define
\$
{\rm (v)} &= \bigl\| \ObsB^{\param^t}_{h}(a_{h}, o_{h})\bigl(\hat\pone^t_h(\seq{a}{h-\ell}{h-1}) - \pone^t_h(\seq{a}{h-\ell}{h-1})\bigr) \bigr\|_1,\notag\\
{\rm (vi)}&= \|\ObsB^{\param^t}_{h}(a_{h}, o_{h})\hat\pone^t_h(\seq{a}{h-\ell}{h-1}) - \hat\ptwo^t_h(\seq{a}{h-\ell}{h}, o_h)\|_1,\notag\\
{\rm (vii)}&=\|\hat\ptwo^t_h(\seq{a}{h-\ell}{h}, o_h) - \ptwo^t_h(\seq{a}{h-\ell}{h}, o_h)\|_1,\notag\\
{\rm (viii)}&=\|\ObsB^{\param^*}_{h}(a_{h}, o_{h})\pone^t_h(\seq{a}{h-\ell}{h-1}) - \ptwo^t_h(\seq{a}{h-\ell}{h}, o_h)\|_1.
\$
In what follows, we upper bound terms (v)--(viii) on the right-hand side of \eqref{eq::pf_lem_ge_III_eq4}. By the definition of parameter space and the concentration inequality in Lemma \ref{lem::concen}, we obtain that
\#\label{eq::pf_lem_ge_III_i}
{\rm (v)} &= \bigl\| \ObsB^{\param^t}_{h}(a_{h}, o_{h})\bigl(\hat\pone^t_1(\seq{a}{h-\ell}{h-1}) - \pone^t_h(\seq{a}{h-\ell}{h-1})\bigr) \bigr\|_1 \notag\\
&\leq \|\ObsB^{\param^t}_{h}(a_{h}, o_{h})\|_{1\mapsto1}\cdot \|\hat\pone^t_h(\seq{a}{h-\ell}{h-1}) - \pone^t_h(\seq{a}{h-\ell}{h-1})\|_1\notag\\
& = \cO\bigl(\nu\cdot \sqrt{A^{5k+1}\cdot O^{k+\ell}}\cdot (k+\ell)\cdot \sqrt{\log(O\cdot A\cdot T\cdot H/\delta)/T}\bigr).
\#
Meanwhile, by the definition of confidence set ${\cC^t}$ in \eqref{eq::def_CI} and the fact that $\param^t \in{\cC^t}$, we obtain that
\#\label{eq::pf_lem_ge_III_ii}
{\rm (vi)} &= \|\ObsB^{\param^t}_{h}(a_{h}, o_{h})\hat\pone^t_h(\seq{a}{h-\ell}{h-1}) - \hat\ptwo^t_h(\seq{a}{h-\ell}{h}, o_h)\|_1\notag\\
&=\cO\bigl(\nu\cdot \sqrt{A^{5k+1}\cdot O^{k+\ell}}\cdot (k+\ell)\cdot \sqrt{\log(O\cdot A\cdot T\cdot H/\delta)/T}\bigr).
\#
By the concentration inequality in Lemma \ref{lem::concen}, we further obtain that
\#\label{eq::pf_lem_ge_III_iii}
{\rm (vii)} &= \|\hat\ptwo^t_h(\seq{a}{h-\ell}{h}, o_h) - \ptwo^t_h(\seq{a}{h-\ell}{h}, o_h)\|_1 \notag\\
&= \cO\bigl(\sqrt{A^{3k+1}\cdot O^{k+\ell}}\cdot (k+\ell)\cdot \sqrt{\log(O\cdot A\cdot T\cdot H/\delta)/T}\bigr).
\#
Finally, by the identity of Bellman operators in \eqref{eq::id_OO}, we have
\#\label{eq::pf_lem_ge_III_iv}
{\rm (viii)}=\|\ObsB^{\param^*}_{h}(a_{h}, o_{h})\pone^t_h(\seq{a}{h-\ell}{h-1}) - \ptwo^t_h(\seq{a}{h-\ell}{h}, o_h)\|_1  = 0.
\#
Plugging \eqref{eq::pf_lem_ge_III_i}, \eqref{eq::pf_lem_ge_III_ii}, \eqref{eq::pf_lem_ge_III_iii}, and \eqref{eq::pf_lem_ge_III_iv} into \eqref{eq::pf_lem_ge_III_eq4}, it holds for all $a_h \in\cA$, $o_h \in\cO$, and $\seq{a}{h-\ell}{h-1}\in\cA^\ell$ with probability at least $1 - \delta$ that
\#\label{eq::pf_lem_ge_III_eq5}
&\|\bigl(\ObsB^\param_{h}(a_{h}, o_{h}) - \ObsB^{\param'}_{h}(a_{h}, o_{h})\bigr)\pone^t_h(\seq{a}{h-\ell}{h-1})\bigr\|_1 \notag\\
&\qquad= \cO\bigl(\nu\cdot \sqrt{A^{5k+1}\cdot O^{k+\ell}}\cdot (k+\ell)\cdot \sqrt{\log(O\cdot A\cdot T\cdot H/\delta)/T}\bigr).
\#
Combining \eqref{eq::pf_lem_ge_III_eq1}, \eqref{eq::pf_lem_ge_III_eq2}, \eqref{eq::pf_lem_ge_III_eq3}, and \eqref{eq::pf_lem_ge_III_eq5}, we conclude that
\$
&\sum_{s_{h-1}\in\cS}\bigl\|\bigl(\ObsB^{\param^t}_{h}(a_{h}, o_{h}) - \ObsB^{\param^*}_{h}(a_{h}, o_{h})\bigr)\UU^{\param^*}_{h,k} \Tran^{\param^*}_{h-1}(s_{h-1}, a_{h-1})\bigr\|_1\cdot\PP^t(s_{h-1} \given \seq{a}{h-\ell}{h-1})\notag\\
&\qquad = \bigl\|\bigl(\ObsB^\param_{h}(a_{h}, o_{h}) - \ObsB^{\param'}_{h}(a_{h}, o_{h})\bigr)\UU^{\param'}_{h,k} \Tran^{\param'}_{h-1}(a_{h-1})A^\pi_{h-1, \ell}(\seq{a}{h-\ell}{h-2})C^{\pi, \dagger}_{h-1, \ell}(\seq{a}{h-\ell}{h-2})\bigr\|_1\notag\\
&\qquad\leq \gamma\cdot \|\bigl(\ObsB^\param_{h}(a_{h}, o_{h}) - \ObsB^{\param'}_{h}(a_{h}, o_{h})\bigr)\pone^t_h(\seq{a}{h-\ell}{h-1})\bigr\|_1\notag\\
&\qquad=\cO\Bigl(\gamma\cdot \nu\cdot (k+\ell)\cdot \sqrt{ A^{5k+\ell}\cdot O^{k+1}\cdot \log(O\cdot A\cdot T\cdot H/\delta)/t}\Bigr),
\$
where the first inequality follows from \eqref{eq::pf_lem_ge_III_eq2} and \eqref{eq::pf_lem_ge_III_eq3}. Thus, we complete the proof of \eqref{eq::lem_ge_eq2}.

\end{proof}

\subsection{Proof of Lemma \ref{lem::norm_bound}}
\begin{proof}
\label{sec::pf_lem_norm_bound}
Recall that we define 
\$
\diagO_h(o_h) &= \diag\bigl(\OO_h(o_h\given \cdot)\bigr) = \diag\Bigl(\bigl[\OO(o_h\given s_h)\bigr]_{s_h}\Bigr) \in \RR^{S\times S},
\$
where we denote by $\diag(v)\in\RR^{S\times S}$ the diagonal matrix where the diagonal entries aligns with the vector $v\in\RR^S$. Thus, it holds for all $h\in[H]$ that
\$
\sum_{o_h\in\cO} \|\Tran_h(a_h)\diagO_h(o_h) u\|_1 &\leq \sum_{o_h\in\cO} \|\diagO_h(o_h) u\|_1 = \sum_{o_h\in\cO}\sum_{s_h\in\cS} \OO_h(o_h\given s_h)\cdot |u(s_h)| \notag\\
&= \sum_{s_h\in\cS}|u(s_h)| = \|u\|_1,
\$
where the first inequality follows from the fact that $\Tran_h(a_h)$ is a transition matrix. Here we denote by $u(s_h)$ the $s_h$-th entry of $u\in\RR^S$. Inductively, we obtain that
\#\label{eq::pf_lem_norm_bound_1}
\sum_{\seq{o}{h}{H-1}\in\cO^{H-h}}\|\Tran_{H-1}(a_{H-1})\diagO_{H-1}(o_{H-1})\ldots \Tran_{h}(a_{h})\diagO_{h}(o_{h}) u\|_1 \leq \|u\|_1.
\#
Meanwhile, note that $\UU_{h}$ is a transition matrix. Thus, it holds for all $h\in[H]$ that
\#\label{eq::pf_lem_norm_bound_2}
&\|\UU_{h} \Tran_{H-1}(a_{H-1})\diagO_{H-1}(o_{H-1})\ldots \Tran_{h}(a_{h})\diagO_{h}(o_{h}) u\|_1 \notag\\
&\qquad\leq \|\Tran_{H-1}(a_{H-1})\diagO_{H-1}(o_{H-1})\ldots \Tran_{h}(a_{h})\diagO_{h}(o_{h}) u\|_1
\#
Combining \eqref{eq::pf_lem_norm_bound_1} and \eqref{eq::pf_lem_norm_bound_2}, we conclude
\$
\sum_{\seq{o}{h}{H-1}\in\cO^{H-h}}\|\UU_{h} \Tran_{H-1}(a_{H-1})\diagO_{H-1}(o_{H-1})\ldots \Tran_{h}(a_{h})\diagO_{h}(o_{h}) u\|_1 \leq \|u\|_1,
\$
which completes the proof of Lemma \ref{lem::norm_bound}.
\end{proof}

\end{document}